\documentclass[12pt]{article}
\usepackage{amsmath,amsfonts,amsthm,graphicx,bm,enumerate,fancyhdr,mathtools, multicol, amssymb, caption, adjustbox, multirow, mathabx, hyperref, xr-hyper}
\usepackage[nodisplayskipstretch]{setspace}
\usepackage{ mathrsfs,color,soul, JASA_manu}
\usepackage{natbib}
\usepackage{longtable}
\usepackage{sectsty,subfigure}
\sectionfont{\fontsize{18}{15}\selectfont}
\subsectionfont{\fontsize{14}{15}\selectfont}
\subsubsectionfont{\fontsize{12}{15}\selectfont}
\paragraphfont{\large}
\usepackage{rotating,booktabs}
\usepackage[margin=1.0in]{geometry}
\definecolor{darkred}{RGB}{150,50,50}
\definecolor{brown}{RGB}{250,100,100}
\definecolor{green}{RGB}{000,150,100}
\definecolor{purple}{RGB}{250,000,180}

\def\red{\color{black}}

\def\Ninv{N^{-1}}
\def\sumiN{\sum_{i=1}^N}
\def\Psc{\bar{\Psc}}

\def\nhalf{n^{1/2}}

\usepackage{xr}
\usepackage{hyperref}
\usepackage{verbatim}

\DeclareMathOperator*{\argmin}{arg\,min}

\DeclareMathAlphabet{\mathpzc}{OT1}{pzc}{m}{it}
\hypersetup{
    bookmarks=true,         
    unicode=false,          
    pdftoolbar=true,        
    pdfmenubar=true,        
    pdffitwindow=false,     
    pdfstartview={FitH},    
    pdftitle={All the way down},    
    pdfauthor={Jim Hefferon, Saint Michael's College},     
    pdfsubject={important things, and some smaller things},   
    pdfcreator={Jim Hefferon},   
    pdfproducer={Jim Hefferon}, 
    pdfkeywords={turtles, creation} {computation} {turing machines}, 
    pdfnewwindow=true,      
    colorlinks=true,       
    linkcolor=blue,          
    citecolor=black,        
    filecolor=red,      
    urlcolor=blue           
}

\def\Usc{\mathcal{U}}

\def\bzero{{\bf 0}}
\def\bone{{\bf 1}}
\def\ba{{\mbox{\boldmath$a$}}}
\def\bb{{\bf b}}

\def\be{{\bf e}}

\def\br{{\bf r}}

\def\bu{{\bf u}}
\def\bv{{\bf v}}
\def\bw{{\bf w}}
\def\bx{{\bf x}}

\def\bz{{\bf z}}
\def\bA{{\bf A}}
\def\bB{{\bf B}}
\def\bC{{\bf C}}
\def\bD{{\bf D}}

\def\bF{{\bf F}}
\def\bG{{\bf G}}
\def\bH{{\bf H}}
\def\bI{{\bf I}}
\def\bJ{{\bf J}}
\def\bK{{\bf K}}

\def\bP{{\bf P}}
\def\bQ{{\bf Q}}
\def\bR{{\bf R}}

\def\bU{{\bf U}}

\def\bW{{\bf W}}
\def\bX{{\bf X}}

\def\bZ{{\bf Z}}

\def\thick#1{\hbox{\rlap{$#1$}\kern0.25pt\rlap{$#1$}\kern0.25pt$#1$}}
\def\balpha{\boldsymbol{\alpha}}
\def\bbeta{\boldsymbol{\beta}}
\def\bgamma{\boldsymbol{\gamma}}

\def\btheta{\boldsymbol{\theta}}

\def\bmu{\boldsymbol{\mu}}
\def\bnu{\boldsymbol{\nu}}
\def\bxi{\boldsymbol{\xi}}

\def\bDelta{\boldsymbol{\Delta}}

\def\bTheta{\boldsymbol{\Theta}}

\def\bXi{\boldsymbol{\Xi}}

\def\bSigma{\boldsymbol{\Sigma}}

\def\bPhi{\boldsymbol{\Phi}}

\def\bPsi{\boldsymbol{\Psi}}

\def\smbalpha{\boldsymbol{{\scriptstyle{\alpha}}}}

\def\what{{\widehat w}}

\def\Dhat{{\widehat D}}

\def\What{{\widehat W}}

\def\mtilde{{\widetilde m}}

\def\bAhat{{\widehat \bA}}

\def\bUhat{{\widehat \bU}}

\def\bWhat{{\widehat \bW}}

\def\bQtilde{{\widetilde \bQ}}

\def\thetahat{{\widehat\theta}}

\def\pihat{{\widehat\pi}}
\def\rhohat{{\widehat\rho}}

\def\thetatilde{{\widetilde\theta}}

\def\bthetahat{{\widehat\btheta}}

\def\bSigmahat{{\widehat\bSigma}}

\def\smbalpha{\widehat{\smbalpha}}

\def\balphatilde{{\widetilde\balpha}}

\def\bgammatilde{{\widetilde\bgamma}}

\def\bthetatilde{{\widetilde\btheta}}

\def\bnutilde{{\widetilde\bnu}}

\def\hbar{\bar{ h}}

\def\Dbar{\bar{ D}}

\def\Bsc{{\cal B}}

\def\Dsc{{\cal D}}

\def\Gsc{{\cal G}}

\def\Isc{{\cal I}}

\def\Msc{{\cal M}}

\def\Psc{{\cal P}}

\def\Ssc{{\cal S}}
\def\Tsc{{\cal T}}
\def\Usc{{\cal U}}
\def\Vsc{{\cal V}}
\def\Wsc{{\cal W}}

\def\Ysc{{\cal Y}}
\def\Zsc{{\cal Z}}

\def\Dschat{\widehat{{\cal D}}}

\def\Tschat{\widehat{{\cal T}}}

\def\Wschat{\widehat{{\cal W}}}

\def\Isctilde{\widetilde{{\cal I}}}

\def\Msctilde{\widetilde{{\cal M}}}

\def\Tsctilde{\widetilde{{\cal T}}}

\def\transpose{{\sf \scriptscriptstyle{T}}}

\def\half{\frac{1}{2}}
\def\ninv{n^{-1}}
\def\nhalf{n^{\half}}

\def\E{\mbox{E}}

\def\var{\mbox{var}}

\def\trans{^{\transpose}}

\def\argmindum{\mathop{\mbox{argmin}}}
\def\argmin#1{\argmindum_{#1}}

\def\cov{\mbox{cov}}

\def\var{\mbox{var}}

\def\mybox#1{\vskip1mm \begin{center}
        \hspace{.0\textwidth}\vbox{\hrule\hbox{\vrule\kern6pt
\parbox{.9\textwidth}{\kern6pt#1\vskip6pt}\kern6pt\vrule}\hrule}
        \end{center} \vskip-5mm}
\def\lboxit#1{\vbox{\hrule\hbox{\vrule\kern6pt
      \vbox{\kern6pt#1\vskip6pt}\kern6pt\vrule}\hrule}}

\def\thickboxit#1{\vbox{{\hrule height 1mm}\hbox{{\vrule width 1mm}\kern6pt
          \vbox{\kern6pt#1\kern6pt}\kern6pt{\vrule width 1mm}}
               {\hrule height 1mm}}}

\def\Cbb{\mathbb{C}}

\def\Sbb{\mathbb{S}}

\def\Ubb{\mathbb{U}}

\def\fat#1{\hbox{\rlap{$#1$}\kern0.25pt\rlap{$#1$}\kern0.25pt$#1$}}

\def\sumiN{\sum_{i=1}^{N}}
\def\Ninv{N^{-1}}
\def\nsinv{n_s^{-1}}

\def\nsnhalf{n_s^{-\half}}
\def\sumsS{\sum_{s=1}^S}

\def\Dscr{\mathscr{D}}
\def\Fscr{\mathscr{F}}
\def\Lscr{\mathscr{L}}
\def\Uscr{\mathscr{U}}
\def\yscr{ {\mathcal Y}}
\def\yscrl{ {\mathcal Y}}

\def\supone{^{(1)}}
\def\suptwo{^{(2)}}
\def\supthree{^{(3)}}
\def\suponetrans{^{(1) \transpose}}
\def\suptwotrans{^{(2) \transpose}}
\def\var{\text{var}}
\def\bthetatilde{\widetilde{\btheta}}

\def\supcv{^{\mbox{\tiny cv}}}

\def\bthetabar{\overline{\btheta}}
\def\bgammabar{\overline{\bgamma}}
\def\bnubar{\overline{\bnu}}

\newtheorem{theorem}{Theorem}

\def\sI{\mbox{\tiny I}}
\def\sII{\mbox{\tiny II}}

\newtheorem{lemma}{Lemma}

\newtheorem{prop}{Proposition}

\newtheorem{cond}{Condition}

\newtheorem{remark}{Remark}

\usepackage{graphicx}
\graphicspath{{./Figures/}}

\makeatletter
\newcommand*{\addFileDependency}[1]{
  \typeout{(#1)}
  \@addtofilelist{#1}
  \IfFileExists{#1}{}{\typeout{No file #1.}}
}
\makeatother

\newcommand*{\myexternaldocument}[1]{%
    \externaldocument{#1}%
    \addFileDependency{#1.tex}%
    \addFileDependency{#1.aux}%
}
\myexternaldocument{Revised_Supplementary_Materials}

\def\SL{\mbox{\tiny SL}}
\def\SSL{\mbox{\tiny SSL}}
\def\DR{\mbox{\tiny DR}}
\def\intri{\mbox{\tiny intri}}

\def\subSL{_{\SL}}
\def\subSSL{_{\SSL}}
\def\subDR{_{\DR}}
\def\subintri{_{\intri}}
\def\subplain{_{\SSL}}
\def\Yscbar{\bar{\Ysc}}
\def\Ysc{\mathcal{Y}}
\def\Tsctilde{\widetilde{\Tsc}}
\def\e{\boldsymbol{e}}
\def\bXi{\boldsymbol{\Xi}}
\def\bDelta{\boldsymbol{\Delta}}
\def\r{\boldsymbol{r}}
\def\balpha{\boldsymbol{\alpha}}
\def\balphatilde{\widetilde{\boldsymbol{\alpha}}}
\def\proj{{\rm Proj}}
\def\bb{\boldsymbol{b}}
\def\ba{\boldsymbol{a}}
\def\supvert{^{\perp}}
\def\bxi{\boldsymbol{\xi}}
\def\bz{\mathbf{z}}
\def\bvarphi{\boldsymbol{\varphi}}
\def\bbf{\mathbf{f}}
\def\P{{\rm P}}
\begin{document}

\footnotetext[1]{Jessica Gronsbell is an Assistant Professor in the Department of Statistical Sciences, University of Toronto, Toronto, ON  M5S 3G3, CA (Email: j.gronsbell@utoronto.ca). Molei Liu is a Ph.D. student in the Department of Biostatistics, Harvard University, Boston, MA 02115, US. Lu Tian is an Associate Professor, Department of Biomedical Data Science, Stanford University, Palo Alto, California 94305, US. Tianxi Cai is a Professor, Department of Biostatistics, Harvard University, Boston, MA 02115, CA.  This research was supported by grants F31-GM119263, T32-NS048005, and R01HL089778  from the National Institutes of Health and RGPIN-2021-03734 from the Natural Sciences and Engineering Research Council.} 
\footnotetext[2]{The first two authors are equal contributors to this work.}

    \title{Efficient Evaluation of Prediction Rules in Semi-Supervised Settings under Stratified Sampling}
    \date{}
     \author{Jessica Gronsbell, Molei Liu, Lu Tian and Tianxi Cai$^{1,2}$ }
    \maketitle
    \thispagestyle{empty}

\begin{abstract}
\vspace{0.2in}
\noindent
In many contemporary applications, large amounts of unlabeled data are readily available while labeled examples are limited. There has been substantial interest in semi-supervised learning (SSL) which aims to leverage unlabeled data to improve estimation or prediction. However, current SSL literature focuses primarily on settings where labeled data is selected {\red{uniformly at random}} from the population of interest. {\red{Stratified sampling}}, while posing additional analytical challenges, is highly applicable to many real world problems. Moreover, no SSL methods currently exist for estimating the prediction performance of a fitted model when the labeled data is not selected {\red{uniformly at random}}. In this paper, we propose a two-step SSL procedure for evaluating a prediction rule derived from a working binary regression model based on the Brier score and overall misclassification rate under stratified sampling. In step I, we impute the missing labels via weighted regression with nonlinear basis functions to account for {\red{stratified sampling}} and to improve efficiency. In step II, we augment the initial imputations to ensure the consistency of the resulting estimators regardless of the specification of the prediction model or the imputation model. The final estimator is then obtained with the augmented imputations. We provide asymptotic theory and numerical studies illustrating that our proposals outperform their supervised counterparts in terms of efficiency gain. Our methods are motivated by electronic health record (EHR) research and validated with a real data analysis of an EHR-based study of diabetic neuropathy.

\vspace{0.1in}
\end{abstract}
\noindent


\newpage
\doublespace
\allowdisplaybreaks

\section{Introduction}
Semi-supervised learning (SSL) has emerged as a powerful learning paradigm to address big data problems where the outcome is cumbersome to obtain and the predictors are readily available  \citep{chapelle2009semi}. Formally, the SSL problem is characterized by two sources of data:  (i) a relatively small  sized labeled dataset $\Lscr$ with $n$ observations on the  outcome $y$ and the predictors $\bx$ and (ii) a much larger unlabeled dataset $\Uscr$ with $N\gg n$ observations on only $\bx$.  A promising application of SSL and the motivation for this work is in electronic health record (EHR) research. EHRs have immense potential to serve as a major data source for biomedical research as they have generated extensive information repositories on representative patient populations \citep{murphy2009instrumenting, kohane2011using, wilke2011emerging}.  Nonetheless, a primary bottleneck in recycling EHR data for secondary use is to accurately and efficiently extract patient level disease phenotype information  \citep{sinnott2014improving, liao2015development}. Frequently, true phenotype status is not well characterized by disease-specific billing codes.  For example, at Partner's Healthcare, only 56\% of patients with at least 3 International Classification of Diseases, Ninth Revision (ICD9) codes for rheumatoid arthritis (RA) have confirmed RA after manual chart review \citep{liao2010electronic}.  More accurate EHR phenotyping has been achieved by training a prediction model based on a number of features including billing codes, lab results and mentions of clinical terms in narrative notes extracted via natural language processing (NLP) \citep[e.g.]{liao2013autoantibodies,  xia2013modeling, ananthakrishnan2013improving}.  The model is traditionally trained and evaluated using a small amount of labeled data obtained from manual medical chart review by domain experts.  SSL methods are particularly attractive for developing such models as they leverage unlabeled data to achieve higher estimation efficiency than their supervised counterparts.  In practice, this increase in efficiency can be directly translated into requiring fewer chart reviews without a loss in estimation precision.



In the EHR phenotyping setting, it is often infeasible to select the labeled examples {\red{uniformly at random}}, either due to practical constraints or due to the nature of the application.  For example, it may be necessary to oversample individuals for labeling with a particular billing code or diagnostic procedure for a rare disease to ensure that an adequate number of cases are available for model estimation.  In some settings, training data may consist of a subset of individuals selected {\red{uniformly at random}} together with a set of registry patients whose phenotype status is confirmed through routine collection.  Stratified sampling is also an effective strategy when interest lies in simultaneously characterizing multiple phenotypes. One may oversample patients with at least one billing code for several rare phenotypes and then perform chart review on the selected patients for all phenotypes of interest.  Though it is critical to account for the aforementioned sampling mechanisms to make valid statistical inference, it is non-trivial in the context of SSL since the fraction of subjects being sampled for labeling is near zero. The challenge is further amplified when the fitted prediction models are potentially misspecified. 

Existing SSL literature primarily concerns the setting in which $\Lscr$ is a {\red{uniform random sample}} from the underlying pool of data and thus the missing labels in $\Uscr$ are missing completely at random (MCAR) \citep{wasserman2008statistical}.  In this setting, a variety of methods for classification have been proposed including generative modeling \citep{castelli1996relative, jaakkola1999exploiting}, manifold regularization \citep{belkin2006manifold, niyogi2008manifold} and graph-based regularization  \citep{belkin2004semi}.  While making use of both $\Uscr$  and $\Lscr$  can improve estimation, in many cases SSL is outperformed by supervised learning (SL) using only $\Lscr$ when the assumed models are incorrectly specified \citep{castelli1996relative, jaakkola1999exploiting, corduneanu2001stable, cozman2002unlabeled, cozman2003semi}.  As model misspecification is nearly inevitable in practice, recent work has called for `safe' SSL methods that are always at least as efficient as the SL counterparts.  For example, several authors have considered safe SSL methods for discriminative models based on density ratio weighted maximum likelihood \citep{sokolovska2008asymptotics,kawakita2013semi,kawakita2014safe}.  Though the true density ratio is 1 in the MCAR setting, the efficiency gain is achieved through estimation of density ratio weight, a statistical paradox previously observed in the missing data literature \citep{robins1992estimating, robins1994estimation}.  More recently, \cite{krijthe2016projected} introduced a SSL method for least squares classification that is guaranteed to outperform SL. \cite{chakrabortty2018efficient} proposed an adaptive imputation-based SSL approach for linear regression that also outperforms supervised least squares estimation. It is unclear, however, whether these methods can be extended to accommodate additional loss functions.  Moreover, none of the aforementioned methods are applicable to settings where the labeled data is not a {\red{uniform random sample}} from the underlying data such as the stratified sampling design.  Additionally, the focus of existing work has been on the estimation of prediction models, rather than the estimation of model performance metrics.  \cite{gronsSSL2017} recently proposed a semi-supervised procedure for estimating the receiver operating characteristic parameters, but this method is similarly limited to the standard MCAR setting.    


This paper addresses these limitations through the development of an efficient SS estimation method for model performance metrics that is robust to model misspecification in the presence of stratified sampling. Specifically, we develop an imputation based procedure to evaluate the prediction performance of a potentially misspecified binary regression model.  To the best of our knowledge, the proposed method is the first SSL procedure that provides efficient and robust estimation of prediction performance measures under stratified sampling. We focus on two commonly used error measurements, the overall misclassification rate (OMR) and the Brier score. The proposed method involves two steps of estimation.  In step I, the missing labels are imputed with a weighted regression with nonlinear basis functions to account for stratified sampling and to improve efficiency. In step II, the initial imputations are augmented to ensure the consistency of the resulting estimators regardless of the specification of the prediction model or the imputation model.  Through theoretical results and numerical studies, we demonstrate that the SS estimators of prediction performance are (i) robust to the misspecification of the prediction or imputation model and (ii) substantially more efficient than their SL counterparts.  We also develop an ensemble cross-validation (CV) procedure to adjust for overfitting and a perturbation resampling procedure for variance estimation.  

The remainder of this paper is organized as follows. In Section \ref{sec-prelim}, we specify the data structure and problem set-up. We then develop the estimation and bias correction procedure for the accuracy measures in Sections \ref{sec:est-acu} and \ref{sec-cv}. Section \ref{sec:asym} outlines the asymptotic properties of the estimators and section \ref{sec-resample} introduces the perturbation resampling procedure for making inference.  Our proposals are then validated through a simulation study in Section \ref{sec-sim} and a real data analysis of an EHR-based study of diabetic neuropathy is presented in Section \ref{sec-data}. We conclude with additional discussions in Section \ref{sec-conc}.  


\section{Preliminaries}\label{sec-prelim}
\subsection{Data Structure}\label{sec-data-struc}
Our interest lies in evaluating a prediction model for a binary phenotype $y$ based on a predictor vector  $\bx =(1, x_1, \dots, x_p) \trans$ for some fixed $p$. The underlying full data consists of $N=\sumsS N_s$ independent and identically distributed random vectors 
$$\Fscr = \{ \bF_i = (y_i, \bu_i\trans)\trans \}_{i = 1}^{N}$$
where $\bu_i  = (\bx_i \trans, \Ssc_i)\trans$, $\Ssc_i \in \{1,2, \dots, S\}$ is a discrete stratification variable that defines a fixed number of strata $S$ for sampling, and $N_s=\sumiN I(\Ssc_i=s)$ is the sample size of stratum $s$.  Throughout, we let $\bF_0 = (y_0, {\bx_0}\trans, \Ssc_0)\trans$ be a future realization of $\bF$.  

Due to the difficulty in ascertaining $y$, {\red{a small uniform random sample is obtained from each stratum and labeled with outcome information}}. The observable data therefore consists of 
$$\Dscr = \{ \bD_i =  (y_iV_i, \bu_i \trans, V_i )\trans\}_{i = 1}^{N}$$ 
where $V_i \in \{0,1\}$ indicates whether $y_i$ is ascertained.  We let 
$$P(V_i = 1 \mid \Fscr) = \pihat_{\Ssc_i}, \quad \pihat_s = n_{s}/N_{s} \quad \mbox{and} \quad n_s = \sum_{j= {1}}^{N}I(\Ssc_j = s)V_j.$$ 
Without loss of generality, we suppose that the first $n= \sumsS n_s$ subjects are labeled and $\{n_s, s=1,..., S\}$ are specified by design. We assume that $$\rhohat_{1s} = n_s /n \overset{p}\to \rho_{1s} \in (0, 1)  \quad \mbox{and} \quad \rhohat_s = N_s/N \overset{p}\to \rho_s \in (0,1)$$ 
as $n$ and $N\to \infty$ respectively.  This ensures that 
$$\hat\pi_s/\hat \pi_t \overset{p}\to  (\rho_{1s}\rho_t)/(\rho_s\rho_{1t}) \in (0, 1)$$ 
as $n\rightarrow \infty$ for any pair of $s$ and $t$ {\red{\citep{mirakhmedov2014edgeworth}}}. As in the standard SS setting, we further assume  $\max_s \pihat_s\overset{p}\to 0$ as $n \to \infty$ {\red{\citep{chakrabortty2018efficient, zhang2019semi}}}. {\red{This assumption distinguishes the current setting from (i) the familiar missing data setting where $\pihat_s$ is bounded above 0 and (ii) standard SSL under uniform random sampling as $y$ is MCAR conditional on $\Ssc$ (i.e.\ $V \perp (y, \bx) | \Ssc$)  under stratified sampling}}.

\subsection{Problem Set-Up}
To predict $y_0$ based on $\bx_0$, we fit a {\em working} regression model 
\begin{equation}
{\rm P} (y  = 1 \mid \bx) = g(\btheta\trans\bx)  \label{model}
\end{equation}
where $\btheta = (\theta_0, \theta_1, \dots, \theta_{p})\trans$ is an unknown vector of regression parameters and $g(\cdot): (-\infty,\infty) \to (0,1)$ is a specified, smooth monotone function such as the expit function. The target model parameter, $\bthetabar$, is the solution to the estimating equation
\[
\bU( \btheta) = {\rm E}\left[ \bx \{ y -g(\btheta \trans\bx) \} \right] = \bzero .
\]
We let the predicted value for $y_0$ be $\Ysc(\bthetabar \trans \bx_0)$ for some function $\Ysc$.  In this paper, we aim to obtain SS estimators of the prediction performance of $\Ysc(\bthetabar \trans \bx_0)$ quantified by the Brier score
$$
\Dbar_1 = {\rm E} [\{y_0 - \yscr_1(\bthetabar  \trans \bx_0) \} ^2] \quad \mbox{with} \quad  \yscr_1(x) =   g(x), $$
and the overall misclassification rate (OMR) 
 $$\Dbar_2={\rm E} [\{y_0 - \yscr_2(\bthetabar\trans \bx_0) \}^2]  \quad \mbox{with}\quad \yscr_2(x) =  I\{ g(x) > c\} \quad \mbox{for some specified constant $c$.}$$
{\red{We focus on these two metrics as the convey distinct information about the performance of the prediction model.  The OMR summarizes the overall discrimination capacity of the model while the Brier score summarizes the calibration of the model. More complete discussions regarding assessment of model performance can be found in \cite{hand1997construction, hand2001measuring, gneiting2007strictly} and \cite{gerds2008performance}.}}

To simplify presentation, we generically write 
$$\Dbar = D(\bthetabar) \quad \mbox{with} \quad D(\btheta) =  {\rm E} [ d\{y_0, \yscr(\btheta  \trans \bx_0) \} ] ,$$
where $d(y,z) = (y - z)^2$, $\yscr(\cdot) = \yscr_1(\cdot)$ for Brier score, and $\yscr(\cdot) = \Ysc_2(\cdot)$ for the OMR.  We will construct a SS estimator of $D(\bthetabar)$ to improve the statistical efficiency of its SL counterpart, $\Dhat\subSL(\bthetahat\subSL)$, where 
\begin{align*}
\mbox{$\bthetahat\subSL$ solves \ } \bU_n( \btheta)= \frac{1}{N}\sum_{i = 1}^N \what_i \bx_i \{ y_i -g(\btheta \trans\bx_i) \} = \bzero ,\quad\Dhat\subSL( \btheta) = \frac{1}{N}\sum_{i = 1}^N \what_i d\{y_i, \yscr(\btheta\trans \bx_i)\} ,
\end{align*}
and the weights $\what_i = V_i/\pihat_{\Ssc_i}$ account for the stratified sampling with $\sumiN \what_i = N$.  {\red{Since $\what_i \overset{p}\to \infty$ for those with $V_i=1$, standard M-estimation theory cannot be directly applied to establish the asymptotic behavior of the SL estimators.}} We show in Appendix \ref{sec: asym-thetaSL} that $\bthetahat\subSL$ is a root-$n$ consistent estimator for $\bthetabar$ and derive the asymptotic properties of $\Dhat\subSL( \bthetahat\subSL)$ in Appendix \ref{sec: asym-DSL}.  {\red{We also note that throughout the article we use the subscripts $n$ or $N$ to index estimating equations to clarify if they are computed with the labeled or full data, respectively.}}


\section{Estimation Procedure}\label{sec:est-acu}
Our approach to obtaining a SS estimator of $D(\bthetabar)$ proceeds in two steps.  First, the missing outcomes are imputed with a flexible model to improve statistical efficiency.  Next, the imputations are augmented so that the resulting estimators are consistent for $D(\bthetabar)$ regardless of the specification of the prediction model or the imputation model.  The final estimator of the accuracy measure is then estimated using the full data and the augmented imputations.  This estimation procedure is detailed in the subsequent sections.  

We comment here that the initial imputation step allows for construction of a simple and efficient SS estimator of $\bthetabar$.  As efficient estimation of $\bthetabar$ may not be of practical utility in the prediction setting, we keep our focus on accuracy parameter estimation and defer some of the technical details of model parameter estimation to the Appendix.  However, we do note that the estimator of $D(\bthetabar)$ inherently has two sources of estimation variability.  The dominating source of variation is from estimating the accuracy measure itself while the second source is from the estimation of the regression parameter.  Therefore, by leveraging a SS estimator of $\bthetabar$ in estimating $D(\bthetabar)$, we may further improve the efficiency of our SS estimator of the accuracy measure.  These statements are elucidated by the influence function expansions of the SS and SL estimators of $D(\bthetabar)$ presented in Section \ref{sec:asym}.    

\subsection{Step 1: Flexible imputation}\label{sec:flex:imp}
We propose to impute the missing $y$ with an estimate of $m(\bu) = P(y = 1|\bu)$.  The purpose of the imputation step is to make use of $\Uscr$ as it essentially characterizes the covariate distribution due to its size.  The accuracy metrics provide measures of agreement between the true and predicted outcomes and therefore depend on the covariate distribution.  We thus expect to decrease estimation precision by incorporating $\Uscr$ into estimation.  In taking an imputation-based approach,  we rely on our estimate of $m(\bu)$ to capture the dependency of $y$ on $\bu$ in order to glean information from $\Usc$.  While a fully nonparametric method such as kernel smoothing allows for complete flexibility in estimating $m(\bu)$, smoothing generally does not perform well with moderate $p$ due to the curse of dimensionality  \citep{kpotufe2010curse}.  To overcome this challenge and allow for a rich model for $m(\bu)$, we incorporate some parametric structure into the imputation step via basis function regression.


Let $\bPhi(\bu)$ be a finite set of basis functions {with fixed dimension} that includes $\bx$. We fit a {\em working model}
\begin{equation}\label{impmod}
{\rm P}(y = 1 \mid \bu) = g\{\bgamma \trans \bPhi(\bu) \}
\end{equation}
to $\Lscr$ and impute $y$ as $g\{\bgammatilde \trans \bPhi(\bu) \}$ where $\bgammatilde$ is the solution to 
 \begin{equation}\label{impee}
\bQtilde_n(\bgamma)=  \frac{1}{N}\sum_{i = 1}^N \what_i \bPhi_i [y_i - g\{\bgamma \trans \bPhi_i\} ]  - \lambda_n \bgamma = \bzero
 \end{equation}
$\bPhi_i = \bPhi(\bu_i)$ and $\lambda_n = o(n^{-\frac{1}{2}})$ is a tuning parameter to ensure stable fitting. Under Conditions 1-3 given in Section \ref{sec:asym}, we argue in Appendix \ref{sec: asym-DSL} that $\bgammatilde$ is a regular root-$n$ consistent estimator for the unique solution, $\bgammabar$, to $\bQ(\bgamma)= {\rm E} \left\{ \bPhi(\bu) [y - g\{\bgamma \trans \bPhi(\bu)\} ] \right\} = \bzero.$ We take our initial imputations as $\mtilde_{\sI}(\bu) = g\{\bgammatilde \trans \bPhi(\bu) \}$.

With $y$ imputed as $\mtilde_{\sI}(\bu)$, we may also obtain a simple SS estimator for $\bthetabar$, $\widecheck \btheta \subSSL$, as the solution to  
 \[
 \bUhat_{N}(\btheta)= \frac{1}{N}\sum_{i = 1}^{N} \bx_i \{ g(\bgammatilde \trans \bPhi_i) - g(\btheta \trans \bx_i)  \} =\bzero.
 \]
 The asymptotic behaviour of $\widecheck\btheta\subSSL$ is presented and compared with $\bthetahat\subSL$ in the Appendix.  When the working regression model (\ref{model}) is correctly specified, it is shown that $\bthetahat\subSL$ is fully efficient and $\widecheck\btheta\subSSL$ is asymptotically equivalent to $\bthetahat\subSL$. When the outcome model in (\ref{model}) is not correctly specified, but the imputation model in (\ref{impmod}) is correctly specified, we show in Appendix \ref{sec: asym-thetaSSL} that $\widecheck\btheta\subSSL$ is more efficient than $\bthetahat\subSL$. When the imputation model is also misspecified, $\widecheck\btheta\subSSL$ tends to be more efficient than $\bthetahat\subSL$, but the efficiency gain is not theoretically guaranteed.  We therefore obtain the final SS estimator, denoted as $\bthetahat\subSSL=(\thetahat_{\SSL,0},  \dots,\thetahat_{\SSL,p})\trans$, as a linear combination of $\bthetahat\subSL$ and $\widecheck \btheta\subSSL$ to minimize the asymptotic variance.  Details are provided in Appendices \ref{sec:ss-mod-param} and \ref{sec: inf-thetaSSL}.  

\subsection{Step 2: Robustness augmentation}\label{sec:aug}
To obtain an efficient SS estimator for $\Dbar = D(\bthetabar)$, we note that 
\begin{align}\label{True}
d(y, \Ysc) = y(1- 2\Ysc) + \Ysc^2
\end{align} 
is linear in $y$ when $y\in\{0,1\}$. With a given estimate of $m(\cdot)$, denoted by $\mtilde(\cdot)$, a SS estimate of $\Dbar$ can be obtained as
\[
 \frac{1}{N}\sum_{i=1}^{N} d\{\mtilde(\bu_i), \yscrl(\btheta \trans \bx_i) \} .
\]
However, $\mtilde(\bu)$ needs to be carefully constructed to ensure that the resulting estimator is consistent for $\Dbar$
 under possible misspecification of the imputation model.  Using the expression in (\ref{True}), a sufficient condition to guarantee consistency for $\Dbar$ is that
\begin{equation}
{\rm E} \left[\{y_0 -  \mtilde(\bu_0)\}\{1-2\yscrl(\bthetabar \trans \bx_0) \} \big| \mtilde(\cdot) \right] \overset{p}\to  0 \quad \mbox{as $n \to \infty$.} \label{condition}  
\end{equation}
This condition implies that ${\rm E}[d\{y_0, \yscrl(\bthetabar \trans \bx_0)\}-d\{\mtilde(\bu_0), \yscrl(\bthetabar \trans \bx_0)\}]\overset{p}\to 0$.
Unfortunately,  $\mtilde_{\sI}(\bu) = g\{\bgammatilde \trans \bPhi(\bu) \}$ does not satisfy (\ref{condition}) when (\ref{impmod}) is misspecified. 
To ensure that (\ref{condition}) holds regardless of the adequacy of the imputation model used for estimating the regression parameters, we augment the initial imputation $\mtilde_{\sI}(\bu)$ as
$$\mtilde_{\sII}(\bu; \btheta) = g\{\bgammatilde \trans \bPhi(\bu) + \bnutilde_{\btheta} \trans \bz_{\btheta} \}$$
where $\bz_{ \btheta}=[1,\yscrl(\btheta \trans \bx)]\trans$  and  $\bnutilde_{\btheta}$ is the solution to the IPW estimating equation
\begin{equation}
\bP_n(\bnu, \btheta) = \frac{1}{N} \sum_{i = 1}^N \what_i \{y_i -   g(\bgammatilde \trans \bPhi_i + \bnu \trans \bz_{i \btheta})  \}  \bz_{i \btheta} = \bzero  \quad \mbox{for any given $\btheta$.}
\label{equ:mom:aug}
\end{equation}
We let $\bnubar_{\btheta}$ be the limiting value of $\bnutilde_{ \btheta}$ which solves the limiting estimation equation
\begin{equation}
 \bR(\bnu \mid \btheta)=
{\rm E} \left( \bz_{\btheta}[y - g\{\bgammabar \trans \bPhi(\bu)+\bnu\trans\bz_{\btheta}\} ] \right) = \bzero.
\label{eq:limesteq}
\end{equation}
This estimating equation is monotone in $\bz_{\btheta}$ for any $\btheta$ and thus $\bnubar_{\btheta}$ exists under mild regularity conditions. {\red{It also follows from (\ref{eq:limesteq}) that
(i) ${\rm E} \left([y - g\{\bar\bgamma \trans \bPhi(\bu)+\bnubar_{\btheta} \trans\bz_{\btheta}\} ] \right) = 0$ and (ii) ${\rm E} \left( \Ysc(\btheta \trans \bx)[y - g\{\bar\bgamma \trans \bPhi(\bu)+\bnubar_{\btheta} \trans\bz_{\btheta}\} ] \right) = 0 $
which ensure that the sufficiency condition in (\ref{condition}) is satisfied.}}  We thus construct a SS estimator of $D(\btheta)$ as
 \begin{align*}
 \Dhat\subSSL( \btheta) =  \frac{1}{N} \sum_{i =1}^{N} d\left\{\mtilde_{\sII}(\bu_i; \btheta),  \yscrl(\btheta\trans\bx_i) \right\}.
\end{align*} 

In Section \ref{sec:asym-D}, we present the asymptotic properties of $\Dhat\subSSL(\bthetahat\subSSL)$ and $\Dhat\subSSL(\widecheck\btheta\subSSL)$ and compare  $\Dhat\subSSL(\widecheck\btheta\subSSL)$ with its supervised counterpart, $\Dhat\subSL(\bthetahat\subSL)$. Similar to the SS estimation of $\bthetabar$, it is shown in Appendix \ref{sec: asym-DSSL} that the unlabeled data helps to reduce the asymptotic variance of $\Dhat\subSSL(\widecheck\btheta\subSSL)$. Specifically, $\Dhat\subSSL(\widecheck\btheta\subSSL)$ is shown to be asymptotically more efficient than $\Dhat\subSL(\bthetahat\subSL)$ when the imputation model is correct.  In practice, however, we may want to use $\Dhat\subSSL(\bthetahat\subSSL)$ instead of $\Dhat\subSSL(\widecheck\btheta\subSSL)$ to achieve improved finite sample performance. 

\section{Bias Correction via Ensemble Cross-Validation}\label{sec-cv}
Similar to the supervised estimators of the prediction performance measures, the proposed plug-in estimator uses the labeled data for both constructing and evaluating the prediction model and is therefore prone to overfitting bias \citep{efron1986biased}.  $K$-fold cross-validation (CV) is a commonly used method to correct for such bias.  However, it has been observed that CV tends to result in overly pessimistic estimates of accuracy measures, particularly when $n$ is not very large relative to $p$ \citep{jiang2007comparison}.  Bias correction  methods such as the 0.632 bootstrap have been proposed to address this behavior \citep{efron1983estimating,efron1997improvements, fu2005estimating, molinaro2005prediction}.  Here, we propose an alternative ensemble CV procedure that takes a weighted sum of the apparent and $K$-fold CV estimators.

We first construct a $K$-fold CV estimator by randomly partitioning $\Lscr$ into $K$ disjoint folds of roughly equal size, denoted by $\{\Lscr_k, k = 1, ...,K\}$. Since $N$ is assumed to be sufficiently large, no CV is necessary for projecting to the full data.  For a given $k$, we use $\Lscr/ \Lscr_k$ to estimate $\bgammabar$ and $\bthetabar$, denoted as $\bgammatilde_{(\text{-}k)}$ and $\bthetahat_{(\text{-}k)}$, respectively.  The {\red{$n_k$}} observations in $\Lscr_k$ are used in the augmentation step to obtain $\bnutilde_{\bthetahat_{(\text{-}k)}, k}$, the solution to $\bP_{n_k}(\bnu, \bthetahat_{(\text{-}k)})=0$.  For the $k^{th}$ fold, we estimate the accuracy measure as 
\begin{equation*}
\Dhat_k(\bthetahat_{(\text{-}k)} ) =  N^{-1} \sum_{i = 1}^{N}   d\left\{\mtilde_{\sII, k}(\bu_i), \yscrl(\bthetahat_{(\text{-}k)}\trans\bx_i) \right\},
\end{equation*}
where $\mtilde_{\sII, k}(\bu_i) = g( \bgammatilde_{(\text{-}k)}\trans \bPhi_i + \bnutilde_{\bthetahat_{(\text{-}k)}, k}\trans \bz_{i \bthetahat_{(\text{-}k)} } )$, and take the final CV estimator as $\Dhat\subSSL\supcv= K^{-1}   \sum_{k = 1}^K\Dhat_k(\bthetahat_{(\text{-}k)} ).$  In practice, we suggest averaging over several replications of CV to remove the variation due to the CV partition.   We then obtain the weighted CV estimator with 
\[
\Dhat\subSSL^{\omega} = \omega \Dhat\subSSL + (1- \omega)  \Dhat\subSSL\supcv, \quad\mbox{where $\omega = K/(2K-1)$.}
\]
We may similarly obtain a CV-based supervised estimator, denoted by $\Dhat\subSL\supcv$, as well as the corresponding weighted estimator, $\Dhat\subSL^{\omega}$. Note that the fraction of observations from stratum $s$ in the $k$th fold, $\rhohat_{1s,k}$, deviates from $\rhohat_{1s}$ in the order of $O(\sqrt{n_s}/n)$. Although this deviation is asymptotically negligibile, it may be desirable to perform the $K$-fold partition within each strata to ensure that $\rhohat_{1s,k} = \rhohat_{1s}$ when $n_s$ is small or moderate.  

Using similar arguments as those given in \cite{tian2007model}, it is not difficult to show that $\nhalf(\Dhat\subSSL-\Dbar)$ and $\nhalf(\Dhat\subSSL\supcv-\Dbar)$ are first-order asymptotically equivalent. Thus, the ensemble CV estimator $\Dhat\subSSL^{\omega}$ reduces the higher order bias of $\Dhat\subSSL$ and $\Dhat\subSSL\supcv$, but has the same asymptotic distribution. Although the empirical performance is promising, it is difficult to rigorously study the bias  properties of $\Dhat\subSSL^{\omega}$ as the regression parameter doesn't necessarily minimize the loss function, $\Dhat(\btheta)$. {We provide a heuristic justification of the ensemble CV method in Appendix \ref{sec: asym-cvW} which assumes $\bthetahat$ minimizes $\Dhat(\btheta)$.}



\section{Asymptotic Analysis}\label{sec:asym}
We next present the asymptotic properties of our proposed SS estimator of $\Dbar$.  To facilitate our presentation, we first discuss the properties of $\widecheck\btheta\subSSL$ as the accuracy parameter estimates inherently depend on the variability in estimating $\bthetabar$.  We then present our main result highlighting the efficiency gain of our proposed SS approach for accuracy parameter estimation. {\red{We conclude our theoretical analysis with two practical discussions of (i) intrinsic efficient estimation in the SS setting and (ii) optimal allocation in stratified sampling}}.  

For our asymptotic analysis, we let $\bSigma_1\succ\bSigma_2$ if $\bSigma_1-\bSigma_2$ is positive definite and $\bSigma_1\succeq\bSigma_2$ if  $\bSigma_1-\bSigma_2$ is positive semi-definite for any two symmetric matrices $\bSigma_1$ and $\bSigma_2$.  For any matrix $\mathbf{M}$ and vectors $\bv_1$ and $\bv_2$, $\mathbf{M}_{j.}$ represents the $j^{th}$ row vector, $\bv_1^{\otimes 2} = \bv_1 \bv_1 \trans$, and $\{\bv_1,\bv_2\}=(\bv_1\trans,\bv_2\trans)\trans$ is the vector concatenating $\bv_1$ and $\bv_2$. To establish our theoretical results, we {\red recall that $\rhohat_{1s}$ and $\rhohat_s$ converge to some fixed values $\rho_{1s}$ and $\rho_s$ in probability, as assumed in Section \ref{sec-data-struc}, and introduce} the following three conditions.

\begin{cond} 
The basis $\bPhi(\bu)$ contains $\bx$, has compact support, and is of fixed dimension.  The density function for $\bx$, denoted by $p(\bx)$, and ${\rm P}(y=1\mid\bu)$ are continuously differentiable in the continuous components of $\bx$ and $\bu$, respectively. There is at least one continuous component of $\bx$ with corresponding non-zero component in $\bthetabar$.  
\label{cond:1}
\end{cond}

\begin{cond}
The link function $g(\cdot)$ is continuously differentiable with derivative $\dot g(\cdot).$
\label{cond:3}
\end{cond}

\begin{cond}
(A) There is no vector $\bgamma$ such that  $P (\bgamma \trans \bPhi_1 > \bgamma \trans \bPhi_2 \mid y_1>y_2)=1$ and ${\rm E}\left[\bPhi^{\otimes 2}\dot{g}\{\bgammabar\trans \bPhi\}\right]\succ \bzero$. (B) There is a small neighborhood of $\bthetabar,$ $\bTheta=\{\btheta:\|\btheta-\bthetabar\|_2<\delta\}$ for some $\delta>0$,  such that for any $\btheta\in \bTheta,$  there is no vector $\br$ such that  $P (\br \trans \{\bPhi_1,\bz_{1\btheta}\} > \br \trans \{\bPhi_2,\bz_{2\btheta}\} \mid y_1>y_2)=1$ and $E \left[\bz_{\btheta}^{\otimes 2}\dot{g}\{\bar{\bgamma}\trans \bPhi+\bnubar_{\btheta}\trans\bz_{\btheta}\}\right] \succ\bzero$. (C) ${\rm E}[\bx^{\otimes 2}\dot{g}(\bthetabar\trans\bx)]\succ \bzero$.
\label{cond:2}
\end{cond}

\begin{remark}
 Conditions 1-3 are commonly used regularity conditions in M-estimation theory and are satisfied in broad applications. Similar conditions can be found in \cite{tian2007model} and Section 5.3 of \cite{van2000asymptotic}. Condition 3(A) and 3(B) assume that there is no $\bgamma$ and $\bnu$ such that $\bgamma\trans\bPhi+\bnu\trans\bz_{\btheta}$ can perfectly separate the samples based on $y$. In our application of EHR data analysis, these conditions are typically satisfied as the outcomes of interest (i.e. disease status) do not perfectly depend on covariates such as billing codes, lab values, procedure codes, and other features extracted from free-text.  Similar to \cite{tian2007model}, Condition 3(A) ensures the existence and uniqueness of the limiting parameters $\bthetabar$ and $\bgammabar$ and Condition 3(B) ensures the existence and uniqueness of $\bnubar_{\btheta}$. 
\end{remark}

\subsection{Asymptotic Properties of $\widecheck\btheta\subSSL$}\label{sec:asm_theta}
The asymptotic properties of $\widecheck\btheta\subSSL$ are summarized in Theorem \ref {thm:1} and the justification is provided in Appendix \ref{sec: asym-thetaSSL}. 
\begin{theorem}
\label{thm:1}
Under Conditions \ref{cond:1}-\ref{cond:2}, $\widecheck \btheta\subSSL \overset{p}{\to} \bthetabar$, and 
\[
\Wschat\subSSL  =  \nhalf( \widecheck \btheta \subSSL - \bthetabar) = \nhalf\sum_{s=1}^S \rho_s\left( n_s^{-1}  \sum_{i=1}^N  V_i I(\Ssc_i = s) \be_{\SSL i} \right) + o_p(1) 
\]
which weakly converges to $N(\bzero, \bSigma\subSSL)$ where
\[
\bSigma\subSSL =  \sum_{s=1}^S \rho_s^2 \rho_{1s}^{-1} E\left\{\be_{\SSL i}^{\otimes 2} \mid \Ssc_i = s\right\} , 
\  \be_{\SSL i} = \bA^{-1}\bx_i \{y_i - g(\bgammabar \trans \bPhi_i) \},  \mbox{ and } \bA = {\rm E} \{ \bx_i^{\otimes 2} \dot{g}(\bthetabar \trans  \bx_i) \}.
\]
\end{theorem}

\begin{remark}
To contrast with the supervised estimator $\bthetahat\subSL$,  we show in Appendix \ref{sec: asym-thetaSL} that
\[
\Wschat\subSL  = \nhalf(\bthetahat\subSL - \bthetabar) 
=  \nhalf\sum_{s=1}^S \rho_s \left( n_s^{-1} \sum_{i=1}^N V_iI(\Ssc_i = s) \be_{\SL i} \right)+ o_p(1),
\]
which weakly converges to $N(\bzero,\bSigma\subSL)$ where 
\[
\bSigma\subSL =  \sum_{s=1}^S \rho_s^2 \rho_{1s}^{-1} {\rm E}\left\{\be_{\SL i}^{\otimes 2} \mid \Ssc_i = s\right\},\quad \mbox{and}\quad
 \be_{\SL i}= \bA^{-1}\bx_i \{y_i - g(\bthetabar \trans \bx_i)\} .
\]
It follows that when the imputation model $P(y = 1 \mid \bu) = g\{\bgammabar \trans \bPhi(\bu)\}$ is correctly specified, $\bSigma\subSL \succeq \bSigma\subSSL$.  When $P(\bgammabar \trans \bPhi(\bu) \ne \bthetabar\trans\bx) > 0,$ we have that $\bSigma\subSL \succ \bSigma\subSSL$.
\end{remark}

\subsection{Asymptotic Properties of $\Dhat\subSSL(\widecheck\btheta\subSSL)$ and $\Dhat\subSSL(\bthetahat\subSSL)$}\label{sec:asym-D}
The asymptotic properties of $\Dhat\subSSL( \widecheck\btheta \subSSL)$ are summarized in Theorem \ref {thm:2} and the justification is provided in Appendix \ref{sec: asym-DSSL}. 
\begin{theorem}\label{thm:2}
Under Conditions \ref{cond:1}-\ref{cond:2}, $\Dhat\subSSL( \widecheck\btheta \subSSL)\overset{p}{\to} \Dbar$, and $\widecheck\Tsc\subSSL =  \nhalf\{\Dhat\subSSL( \widecheck\btheta \subSSL) - D( \bthetabar)\} $ is asymptotically Gaussian with mean zero and variance $\sigma^2\subSSL$ given in Appendix \ref{sec: asym-DSSL}. Also, $\widecheck\Tsc\subSSL$ is asymptotically equivalent to 
\[
\nhalf \sum_{s= 1}^S \rho_s \left( \ninv_s \sum_{i = 1}^N V_i I(\Ssc_i = s) \Big[ \{d(y_i, \Yscbar_i)  - d(m_{\sII,i} , \Yscbar_i) \}+ 
\dot \bD(\bthetabar)\trans\be_{\SSL i} \Big] \right) ,
\]
where $\Yscbar_i = \Ysc(\bthetabar\trans\bx_i)$, $m_{\sII, i} = g(\bgammabar \trans \bPhi_i  + \bnubar_{\bthetabar}\trans\bz_{i \bthetabar })$ is the imputation model based 
approximation to ${\rm P}(y=1 \mid \bu)$ and $\dot{\bD}(\btheta) = \partial D(\btheta)/\partial\btheta$.
\end{theorem}
\noindent
\begin{remark}
We also show that $\Tschat\subSSL =  \nhalf\{\Dhat\subSSL( \bthetahat \subSSL) - D( \bthetabar)\}$ is asymptotically equivalent to 
\[
\nhalf \sum_{s= 1}^S \rho_s \left( \ninv_s \sum_{i = 1}^N  V_iI(\Ssc_i = s) \Big[ \{d(y_i, \Yscbar_i)  - 
d(m_{\sII,i} , \Yscbar_i) \}+ 
\dot \bD(\bthetabar)\trans\{\bW    \be_{\SSL i} +  (\bI - \bW)  \be_{\SL i} \} \Big] \right) ,
\]
which is also asymptotically Gaussian with mean zero where $\bW$ is a diagonal matrix defined in Appendix \ref{sec: inf-thetaSSL}. 
\end{remark}

\begin{remark}
As shown in Appendix \ref{sec: asym-DSL},
$\Tschat\subSL =  \nhalf\{\Dhat\subSL( \bthetahat \subSL) - D( \bthetabar)\}$ is asymptotically Gaussian with mean zero and variance $\sigma^2\subSL$ defined in Appendix \ref{sec: asym-DSL}. It is equivalent to
\begin{align*}
\nhalf  \sum_{s= 1}^S \rho_s \left(\ninv_s \sum_{i = 1}^N V_i I(\Ssc_i = s) \left[ \{d(y_i, \Yscbar_i) - D(\bthetabar)\} + \dot{\bD}(\bthetabar)\trans \be_{\SL i}  \right]\right).
\end{align*}
\end{remark}

We verify in Appendix \ref{sec: asym-DSSL} that when the imputation model is correctly specified,  the asymptotic variance of $\Dhat\subSSL(\widecheck\btheta\subSSL)$ is smaller than that of $\Dhat\subSL( \bthetahat \subSL)$ regardless of the specification of the working regression model in (\ref{model}). This is because the accuracy measures always depend on the marginal distribution of $\bx$ and the proposed SS approach leverages $\Usc$. Therefore $\Dhat\subSSL(\widecheck\btheta\subSSL)$ is asymptotically more efficient than $\Dhat\subSL(\bthetahat \subSL)$ even when model (\ref{model}) is correctly specified and $\bthetahat\subSL$ is fully efficient.

{\red{While we cannot theoretically guarantee that the SS estimator is more efficient than the supervised estimator under misspecification of the imputation model, the first and dominating term in the influence function expansion corresponds to the variability from estimating the accuracy measure.  Even when the imputation model is misspecified, it may still provide a close approximation to $P(y = 1\mid \bu)$ and therefore result in reduced variability relative to the supervised approach.  The second term of the influence function corresponds to the variability from estimation of the regression parameter.  As the SS estimator of the regression parameter is more efficient than its supervised counterpart under model misspecification, we also expect this term to have smaller variation than its supervised counterpart.  In our simulation studies, we evaluate the performance of our proposals under various model misspecifications to assess whether this heuristic justification holds up empirically.  We also further study this limitation from a theoretical viewpoint in the next section where we introduce a SS estimator with the intrinsic efficiency property from the semiparametric inference literate for comparison.}}



\subsection{Intrinsic Efficient SS Estimation}\label{sec:asym-intrinsic}
{\red 
For simplicity, we begin our discussion of intrinsic efficient estimation focusing on estimation of the regression parameter.  Recall that the idea in Section \ref{sec:flex:imp} is to (i) solve $N^{-1}\sum_{i = 1}^N \what_i \bPhi_i [y_i - g\{\bgamma \trans \bPhi_i\} ]  - \lambda_n \bgamma = \bzero$ to obtain estimated coefficients $\widetilde\bgamma$ for imputation and then (ii) solve $N^{-1}\sum_{i = 1}^{N} \bx_i \{ g(\bgammatilde \trans \bPhi_i) - g(\btheta \trans \bx_i)  \} =\bzero$ to obtain the SS estimator, $\widehat\btheta\subplain$. By Theorem \ref{thm:1}, for any $\e\in\mathbb{R}^{p+1}\setminus\{\bzero\}$, the asymptotic variance of $n^\half(\e\trans\bthetahat\subplain-\e\trans\bthetabar)$ can be expressed as:
\begin{equation}
 \frac{1}{n}\sum_{i=1}^n \zeta_i(\e\trans\bA^{-1}\bx_i)^2\{y_i - g(\bgammabar \trans \bPhi_i) \}^2 ,  \label{equ:var:ctheta}
\end{equation}
where $\zeta_i=\sum_{s=1}^S\rho_s^2 \rho_{1s}^{-2}V_i I(\Ssc_i = s)$ for each $i\in\{1,2,...,N\}$.  When the imputation model ${\rm P}(y = 1 \mid \bu) = g(\bgamma \trans\bPhi)$ is misspecified, an alternative estimating equation for $\bgamma$ may be used to directly reduce the asymptotic variance of the resulting SS estimator.  Specifically, for a fixed $\bPhi$, we may find the estimating equation for $\bgamma$ that leads to the lowest asymptotic variance of the estimator for $\e\trans\bar\btheta$, a property referred to as ``intrinsic efficiency" in the semiparametric inference literature \citep{tan2010bounded}. We briefly propose estimation procedures for an estimator achieving this property with potential to improve upon our original proposal under potential misspecification of the imputation model.

To directly minimize the asymptotic variance of the SS estimator of $\e\trans\bar\btheta$ given by (\ref{equ:var:ctheta}), we obtain the estimated coefficients for the imputation model with
\begin{equation}
\begin{split}
\widetilde\bgamma\supone=\argmin\bgamma \frac{1}{2n}&\sum_{i=1}^n\widehat\zeta_i(\e\trans\widehat\bA^{-1}\bx_i)^2\{y_i - g(\bgamma \trans \bPhi_i) \}^2+\lambda_n\supone\|\bgamma\|_2^2,\\  
\mbox{ s.t. }\frac{1}{N}&\sum_{i=1}^N\what_i\bx_i\{y_i-g(  \bgamma\trans\bPhi_i)\}=\bzero,
\end{split}
\label{equ:min:theta}
\end{equation}
where $\widehat\zeta_i=\sum_{s=1}^S\widehat\rho_s^2 \widehat\rho_{1s}^{-2}V_i I(\Ssc_i = s)$, $\widehat{\bA}= N^{-1}\sum_{i=1}^N\bx_i^{\otimes 2} \dot{g}(\widehat\btheta\subSSL \trans  \bx_i)$ are the empirical estimates of $\zeta_i$ and $\bA$, respectively, and $\lambda_n\supone = o(n^{-\frac{1}{2}})$ is again a tuning parameter for stable fitting. We then solve $\Ninv\sum_{i = 1}^{N} \bx_i \{ g(  \bgammatilde\suponetrans\bPsi_i) - g(\btheta \trans \bx_i)  \} =\bzero$ to obtain $\widehat\btheta\subintri$, and return $\e\trans\widehat\btheta\subintri$ as the intrinsic efficient estimator for $\e\trans\bar\btheta$. The moment condition in (\ref{equ:min:theta}) is used for calibrating the potential bias from a misspecified imputation model and ensuring the consistency of $\bthetahat\subintri$. This condition is explicitly imposed when constructing our original proposal.

To study the asymptotic properties of $\widehat\btheta\subintri$ and compare it with our original proposal, $\widehat\btheta\subplain$, we let 
\[
\bar\bgamma\supone=\argmin\bgamma {\rm E}[R(\e\trans\bA^{-1}\bx)^2\{y - g(\bgamma \trans \bPhi) \}^2],\quad\mbox{s.t.}\quad{\rm E}[\bx\{y - g(\bgamma \trans \bPhi) \}]=\bzero,
\]
be the limit of $\bgammatilde\supone$, where $R=\sum_{s=1}^SI(\Ssc=s){\rho_s}/{\rho_{1s}}$.  The proof of Theorem \ref{thm:3} is provided in Appendix \ref{sec:app:intri}.

\begin{theorem}
Under condition \ref{cond:1}, and conditions \ref{cond:a1} and \ref{cond:a2} introduced in Appendix \ref{sec:app:intri}, $n^{\frac{1}{2}}(\widehat\btheta\subintri-\bthetabar)$ converges weakly to a mean zero normal distribution, and is asymptotically equivalent to $\widehat\Wsc(\bgammabar\supone)$ where
\[
\widehat\Wsc(\bgamma)=\nhalf\sum_{s=1}^S \rho_s \left[ n_s^{-1} \sum_{i=1}^N V_iI(\Ssc_i = s)  \bA^{-1}\bx_i \{y_i - g(\bgamma \trans \bPhi_i) \} \right].
\]
In addition: (i) when the imputation model ${\rm P}(y = 1 \mid \bu) = g(\bgamma \trans \bPhi)$ is correctly specified, $\widehat\btheta\subintri$ is asymptotically equivalent to $\widehat\btheta\subplain$ and (ii) the asymptotic variance of $n^{\frac{1}{2}}(\e\trans\widehat\btheta\subintri-\e\trans\bthetabar)$ is minimized among estimators with $\{\e\trans\widehat\Wsc(\bgamma):{\rm E}[\bx\{y - g(\bgamma \trans\bPhi) \}]=\bzero\}$. Consequently, the variance of the intrinsic efficient estimator is always less than or equal to the asymptotic variance of both $n^{\frac{1}{2}}(\e\trans\widehat\btheta\subSL-\e\trans\bthetabar)$ and $n^{\frac{1}{2}}(\e\trans\widehat\btheta\subplain-\e\trans\bthetabar)$.

\label{thm:3}
\end{theorem}

The details and theoretical analysis of the intrinsic efficient estimation procedure of the accuracy measure $\Dbar$ is presented Appendices \ref{sec:app:intri:cons} and \ref{sec:app:intri}.  Similar to Theorem \ref{thm:3}, we show that $\Dhat\subintri$ is asymptotically equivalent with our proposal, $\Dhat\subplain$, when the imputation model is correctly specified and has smaller asymptotic variance than $\Dhat\subplain$ when the imputation model is misspecified. However, it is important to note that estimation based  
on intrinsic efficiency is a non-convex problem and one may encounter numerical optimization issues which may limit its use in practice.  We provide simulation studies comparing the intrinsic efficient estimator to the proposed approach in Section S4 Supplement.

}

{\red
\subsection{Optimal Allocation in Stratified Sampling}\label{sec:thm:opt:all}
Another important practical issue is how to select the strata and the corresponding selection probabilities.  Here we provide here a detailed assessment of the optimal (or Neyman) allocation of the labeled data across the strata.  Specifically, the general form of the influence function for our estimators is 
\begin{align*}
&n^{1/2} \sum_{s=1}^S \rho_{s}  \left\{\ninv_s\sum_{i = 1}^N V_i I(\Ssc_i = s) f(\bF_i) \right\} + o_p(1)  \nonumber
\end{align*}
for a function $f$ with $\sigma^2_s = \E\{f^2(\bF_i) \mid \Ssc_i =  s\}$ and $\E\{f(\bF_i)\} = 0$.  The asymptotic variance can then be expressed as 
\begin{align*}
&\sum_{s=1}^S \rho_{s}^{2} \left\{n_s^{-2} \sum_{i = 1}^N V_i I(\Ssc_i = s) \E\{f^2(\bF_i) \mid \Ssc_i =  s\} \right\} =  \sum_{s=1}^S  \rho_s^2  \frac{\sigma^2_s}{n_s} \\
&=n^{-1}\sum_{s=1}^Sn_s \sum_{s=1}^S \frac{(\rho_s\sigma_s)^2}{n_s}\ge n^{-1}\left(\sum_{s=1}^S \rho_s\sigma_s\right)^2, 
\end{align*}
by the Cauchy-Schwarz inequality, and equality holds if and only if 
\begin{equation}
 n_s = n \frac{\rho_s \sigma_s}{\sum_{s=1}^S \rho_s \sigma_s} \quad \mbox{for $s = 1, \dots, S$}. \label{neq}
\end{equation}
The optimal sampling probabilities are therefore proportional to (i) the relative stratum size and (ii) the variability within the stratum, with greater weight placed on large stratum with high variability. 
Consequently, stratified sampling leads to a more efficient estimator than uniform random sampling when the allocation in (\ref{neq}) is used. 

\begin{remark}
There is a rich body of survey sampling literature concerning model-assisted approaches that address the practically important question of how to select the strata and the corresponding selection probabilities \citep{neyman1934two, sarndal2003model, nedyalkova2008optimal}. The optimal allocation given by (\ref{neq}) is in similar spirit with the sampling schemes used in \cite{cai2012evaluating,liu2012evaluating}. It is particularly useful for EHR-based phenotyping studies such as the diabetic neuropathy example in Section \ref{sec-data} as it is often straightforward for domain experts to define a filter variable(s) that yields relatively large stratum with increased prevalence of $y$ (e.g.\ patients with notes containing terms related to the disease, relevant lab values, or specialist visits) and thus increased variability. 
\label{rem:5.4.1}
\end{remark}

We provide additional numerical studies to illustrate Remark \ref{rem:5.4.1} in Section \ref{sec:strata} of the Supplement.

}



\section{Perturbation Resampling Procedure for Inference}\label{sec-resample}
We next propose a perturbation resampling procedure to construct standard error (SE) and confidence interval (CI) estimates in finite samples. Resampling procedures are particularly attractive for making inference about $\Dbar$ when $\Ysc = \Ysc_2$ since $\Dhat\subSSL(\btheta)$ is not differentiable in $\btheta$. To this end, we generate a set of independent and identically distributed ({\red{i.i.d}}) non-negative random variables, $\Gsc = (G_1, \dots, G_n)$, independent of $\Dsc$, from a known distribution with mean one and unit variance. 

For each set of $\Gsc$, we first obtain a perturbed version of $\bthetahat \subSSL$ as 
$$\bthetahat^* = \bthetahat \subSSL +{\bAhat}^{-1}  \sum_{k = 1}^K \sum_{i \in \Lscr_k} \frac{\what_i(G_i -1) }{\sum_{j=1}^n \what_j}\left[\bx_iy_i - \bWhat\bx_i g(\bgammatilde_{(\text{-}k)} \trans \bPhi_i) - (\bI-\bWhat)\bx_i  g(\bthetahat_{(\text{-}k)} \trans \bx_i)  \right], $$
where $\bAhat = N^{-1} \sum_{i = 1}^{N} \bx_i^{\otimes 2}  \dot{g}( \bthetahat \subSSL \trans  \bx_i)$.  We use CV to correct for variance underestimation due to overfitting. 
Next, we find the solution $\bgammatilde^*$ to the perturbed objective function 
\begin{equation}
\bQtilde_n^{*}(\bgamma)=  \frac{\sum_{i = 1}^n \what_i \bPhi_i [y_i - g(\bgamma \trans \bPhi_i) ] G_i }{\sum_{i = 1}^n \what_iG_i} - \lambda_n \bgamma= \bzero
\end{equation} 
and the solution $\bnutilde^*$ that solves
$$\bP_n^*(\bnu, \bthetahat^*) = \frac{ {\sum_{i = 1}^n} \what_i \{y_i -   g({\bgammatilde*}\trans \bPhi_i + \bnu \trans \bz_{i \bthetahat^* })  \}  \bz_{i \bthetahat^* }G_i }{{\sum_{i = 1}^n} \what_iG_i} = \bzero$$
to obtain perturbed counterparts of $\bgammatilde$ and $\bnutilde$, respectively.  We then compute $\mtilde_{11}^*(\bu_i)   = g(\bgammatilde^{*\intercal} \bPhi_i + \bnutilde^{*\intercal} \bz_{i \bthetahat^* }) $ and  obtain the perturbed estimator of $\Dhat \subSSL(\bthetahat \subSSL)$ as 
 \begin{align*}
 \Dhat\subSSL^*( \bthetahat^*) =  N^{-1} \sum_{i = 1}^{N}\left[ \mtilde_{11}^*(\bu_i)\{ 1- 2\yscrl(\bthetahat^{*\intercal}  \bx_i)\} + \yscrl^2(\bthetahat^{*\intercal}  \bx_i)\right].
 \end{align*} 
 Following arguments such as those in \cite{tian2007model}, one may verify that $\nhalf \{  \Dhat \subSSL^*( \bthetahat \subSSL^*) - \Dhat \subSSL( \bthetahat \subSSL)  \} | \Dscr$ converges to the same limiting distribution as $\Tschat \subSSL$.  Additionally, it may be shown that $\Tschat\supcv\subSSL =\nhalf \{ \Dhat \subSSL \supcv - D(\bthetabar)\}$ and hence $\Tschat^{\omega}\subSSL =\nhalf \{ \Dhat \subSSL ^{\omega} - D(\bthetabar)\}$  converge to the limiting distribution of $\Tschat \subSSL$.  We utilize these results to approximate the distribution of $\Tschat^{\omega}\subSSL$ with the empirical distribution of a large number of perturbed estimates using the above resampling procedure to base inference for $D(\bthetabar)$ on the proposed bias-corrected estimator.  The variance of $ \Dhat \subSSL ^{\omega}$ can correspondingly be estimated with the sample variance and confidence intervals may be constructed accordingly.

 

\def\Sbb{\mathbb{S}}
\def\Cbb{\bC}
\def\Mscr{\mathscr{M}}
\def\sublogit{_{\scriptscriptstyle\sf logistic}}
\def\subextreme{_{\scriptscriptstyle\sf extreme}}

\section{Simulation Studies}\label{sec-sim}
We conducted extensive simulation studies to evaluate the performance of the proposed SSL procedures and to compare to existing methods. Throughout, we generated $p=10$ dimensional covariates $\bx$ from $N(\bzero,\Cbb)$ with $\Cbb_{kl} = 3(0.4)^{|k-l|}$.  Stratified sampling was performed according to $\Ssc$ generated from the following two  mechanisms:
\begin{enumerate}[(1)]
\item $\Ssc \in \{1,S = 2\}$ with $\Ssc = 1 + I(x_1+\delta_1\leq0.5)$ and $\delta_1 \sim N(0,1)$.
\item $\Ssc \in \{1,2,3,S=4\}$ with $\Ssc = 1 + I(x_1+\delta_1 \le 0.5) + 2I(x_3+\delta_2 \le 0.5)$, $\delta_1 \sim N(0,1)$, $\delta_2 \sim N(0,1)$, and $\delta_1\perp \delta_2$.
\end{enumerate}
We let $\Sbb=(I(\Ssc=1), ..., I(\Ssc = S-1))\trans$.   For both settings, we sampled $n_s=100$ or $200$ observations from each stratum. Throughout, we let $\bv_1$ be the natural spline of $\bx$ with 3 knots and $\bv_2$ be the interaction terms $\{\bx_1:\bx_{-1},~\bx_2:\bx_{-(1,2)}\}$, where $\bx_1:\bx_{-1}$ and $\bx_2:\bx_{-(1,2)}$ represent interaction terms of $\bx_1$ with the remaining covariates and $\bx_2$ with covariates excluding $\bx_1$ and $\bx_2$, respectively. 
With $\btheta = \{0, 1, 1, 0.5, 0.5, \bzero_{(p - 4)\times 1}\}\trans$ and $\epsilon\sublogit$ and $\epsilon\subextreme$ denoting noise generated from the logistic and extreme value$(-2,0.3)$ distributions, we simulated $y$ from the following models:
\begin{enumerate}[(i)]
\item[] \hspace{-.35in} (i) ($\Msc_{\mbox{\tiny correct}}$, $\Isc_{\mbox{\tiny correct}}$) with correct outcome model and correct imputation model: 
 $$y = I( \btheta \trans  \bx + \epsilon\sublogit > 2)  \  \mbox{and}\  \bPhi=(1,\bx\trans,\bv_1\trans,\Sbb\trans)\trans; $$
\item[] \hspace{-.35in} (ii) ($\Msc_{\mbox{\tiny incorrect}}$, $\Isc_{\mbox{\tiny correct}}$) with incorrect outcome model and correct  imputation model: 
$$y = I[ \btheta \trans  \bx+0.5\{x_1x_2+x_1x_5-x_2x_6-I(\Ssc=1)\} + \epsilon\sublogit > 0] \ \mbox{and}\  \bPhi=(1,\bx\trans, \bv_2\trans,\Sbb\trans)\trans;$$
\item[] \hspace{-.35in} (iii) ($\Msc_{\mbox{\tiny incorrect}}$, $\Isc_{\mbox{\tiny incorrect}}$) with incorrect outcome model and incorrect imputation model:  
$$y = I \{\btheta \trans  \bx + x_1^2 + x_3^2 +   \exp(-2-3x_4-3x_6)\epsilon\subextreme > 2\} \ \mbox{and}\ 
\bPhi=(1,\bx\trans,\bv_1\trans,\Sbb\trans)\trans . $$
\end{enumerate}
While the outcome model is misspecified in both (ii) and (iii),  the misspecification is more severe in (iii) due to the higher magnitude of nonlinear effects. 
These configurations are chosen to mimic EHR settings where the signals are typically sparse and $S$ is small. The covariate effects of $1$ represent the strong signals from the main billing codes and free-text mentions of the disease of interest. The two weaker signals $0.5$ characterize features such as related medications, signs, symptoms and lab results relevant to the disease of interest. 

Across all settings, we compare our SS estimators to both the SL estimator and the alternative density ratio (DR) method \citep{kawakita2013semi,kawakita2014safe}. 
The basis function $\bvarphi(\bu)$ required in the DR method was chosen to be the same as $\bPhi(\bu)$ in our method for all settings. {\red{The details and theoretical properties of the DR method are further discussed in the Supplement.}} We employed the ensemble CV strategy to construct a bias corrected DR estimator for $\Dbar$, denoted as $\Dhat\subDR^{\omega}$, to ensure a fair comparison to our approach. The three settings of outcome and imputation models under (i), (ii), and (iii) allow us to verify the asymptotic efficiency of the proposed SSL procedures relative to the SL and DR methods under various scenarios of misspecification.



For each configuration, results are summarized with $500$ independent data sets.  The size of the unlabeled data was chosen to be $20,000$ across all settings.  For all our numerical studies including the real data application, CV was performed with either $K=3$ or $K=6$ and averaged over 20 replications. The estimated SEs were based on 500 perturbed realizations and the OMR was evaluated with $c = 0.5$.  We let $\lambda_n=\log(2p)/n^{1.5}$ when fitting the ridge penalized logistic regression.  We focus primarily on results for $S = 2$ and $K =6$, but include results for $S = 4$ and $K = 3$ in Section \ref{sec:add-sim} of the Supplement as they show similar patterns.  Additionally, our analyses concentrate on the performance of the accuracy metrics.  Results for the regression parameter estimates can be found in Section \ref{sec:regr-sim} of the Supplement.  The code to implement the proposed methods and run the simulation studies can be found at \url{https://github.com/jlgrons/Stratified-SSL}.


In Figure \ref{figure:bias:RE}, we present the percent biases of the apparent, CV, and ensemble CV estimators of the accuracy parameters. Although all three estimators have negligible biases, the SSL exhibits slightly less bias than its supervised counterpart and the DR estimator under $(\Msc_{\mbox{\tiny incorrect}},\Isc_{\mbox{\tiny incorrect}})$.  The ensemble CV method is effective in bias correction while the apparent estimators are optimistic and the standard CV estimator is pessimistic. For example, under $(\Msc_{\mbox{\tiny incorrect}},\Isc_{\mbox{\tiny incorrect}})$, $n_s=100$ and $S=2$, the percent bias of the SSL estimators for the OMR are $-8.2\%$, $8.8\%$ and $-0.5\%$ when we use the plug-in, 6-fold CV, and the ensemble CV methods, respectively.  The efficiency of $\Dhat\subSSL^w$ and $\Dhat\subDR^{\omega}$ relative to $\Dhat\subSL^w$ for both the Brier score and OMR are presented in Figure \ref{figure:RE}.  Again, $\Dhat\subSSL^w$ is substantially more efficient than $\Dhat\subSL^w$ and $\Dhat\subDR^{\omega}$, with efficiency gains of approximately 15\%--30\% under $(\Msc_{\mbox{\tiny correct}}, \Isc_{\mbox{\tiny correct}})$, 40\% under $(\Msc_{\mbox{\tiny incorrect}}, \Isc_{\mbox{\tiny correct}})$, and 40\%--80\% under $(\Msc_{\mbox{\tiny incorrect}},  \Isc_{\mbox{\tiny incorrect}})$. Results for $K=3$ have similar patterns and are presented in Figure \ref{figure:bias:RE-add} and \ref{figure:RE-add} of the Supplement. In Table \ref{table: resampling}, we present the results for the interval estimation obtained from the perturbation resampling procedure for the SS estimator $\Dhat \subSSL ^{\omega}$. The SEs are well approximated and the empirical coverage for the 95\% CIs is close to the nominal level across all settings.

{\remark{
{\red
While our simulation studies focus on the SSL estimators proposed in Section \ref{sec:est-acu}, we also investigated the finite sample performance of $\widehat\btheta\subintri$ and $\Dhat\subintri$ and compared them with our original proposals. The numerical studies are described in Section \ref{sec:sup:intri} of the Supplement and demonstrate that when the estimated coefficients for the imputation model of the original estimators are equal or close to those of the intrinsic efficient estimator, these two methods perform equivalently with respect to mean square errors (MSE). In contrast, under a setting where the coefficients for the imputation model differ across these two methods, $\widehat\btheta\subintri$ and $\Dhat\subintri$ have about $30\%$ smaller MSE than $\widehat\btheta\subplain$ and $\Dhat\subplain$, on average. The detailed 
results are presented in Table \ref{tab:app:intri} of the Supplement.
}}}

{\remark{
{\red To illustrate the benefit of stratified sampling in both the supervised and SS settings, we provide numerical studies of the optimal allocation in \ref{sec:strata} of the Supplement.  Mimicking our example in Section \ref{sec-data}, we let the risk of $y$ significantly differ across the $S=2$ sampling groups. The stratification variable is picked so that $\P(y = 1\mid\Ssc=1)$ is much lower than $\P(y = 1\mid\Ssc=2)$. We consider two sampling strategies: (i) uniform random sampling of $n$ subjects, and (ii) stratified sampling of $n/2$ subjects from each stratum. Since $\P(y = 1\mid\Ssc=1)$ is low and close to $0$, the variability of this stratum $\sigma^2_1 = \E\{f^2(\bF) \mid \Ssc =  1\}$ is smaller than that of $\Ssc=2$ with $\P(y = 1\mid\Ssc=2)$ not near $0$ and $1$. Connecting this with (\ref{neq}), the stratified sampling strategy oversamples within $\Ssc=2$ so that its allocation of $n_2$ is more close to the optimal choice. Consistent with this observation, our simulation results indicate that stratified sampling is more efficient than uniform random sampling in both the supervised and SS settings with an average relative efficiency $>1.45$ across different setups. We further inspect the supervised estimator of $D(\bthetabar)$ under setup (I) in Section \ref{sec:strata}, for which  the optimal allocation is $n_1 = 0.47n$ and $n_2 =0.53n$ and nearly coincides with our equal allocation of $n$ across the two stratum.} }}

\section{Example: EHR Study of Diabetic Neuropathy}\label{sec-data}
We applied the proposed SSL procedures to develop and evaluate an EHR phenotyping algorithm for classifying diabetic neuropathy (DN), a common and serious complication of diabetes resulting in nerve damage. The full study cohort consists of $N = 16, 826$ patients over age 18 with one or more of 12 ICD9 codes relating to DN identified from Partners HealthCare EHR.  {\red{An initial assessment of 100 charts by physicians revealed the prevalence of DN in the study cohort was approximately 17\%.  To obtain a labeled set with sufficient DN cases for model training and improve efficiency, the investigators decided to employ a stratified sampling scheme.  To do so, a binary filter variable $\Ssc$ indicating whether a patient had a neurological exam and a neurology note with at least 1,000 characters was created.  The prevalence of DN in the ``enriched'' stratum with $\Ssc=1$ was expected to be higher than that in the stratum with $\Ssc=0$.  As demonstrated in our theoretical analysis in Section \ref{sec:thm:opt:all} and our numerical studies in Section \ref{sec:strata}, oversampling within the enriched set can improve estimation efficiency relative to taking a uniform random sample and is a common approach taken in EHR-based analyses.  For this study, the investigators sampled $n_0 = 70$ and $n_1 = 538$ patients from the $N_0 = 13608$ patients with $\Ssc = 0$ and the $N_1=  3218$ patients with $\Ssc = 1$, respectively, for developing the phenotyping algorithm.      
}}

To train the model for classifying DN, a set of 11 codified and NLP features related to DN were selected from an original list of 75 via an unsupervised screening as described in \cite{yu2015toward}. The codified features included $\Ssc$, diagnostic codes for diabetes, type 2 diabetes mellitus, diabetic neuropathy, other idiopathic peripheral autonomic neuropathy, and diabetes mellitus with neurological manifestation as well as normal glucose lab values and prescriptions for anti-diabetic medications.  The NLP features included mentions of terms related to DN in the patient record including  glycosylated hemoglobin (HgA1c), diabetic, and neuropathy. As all these features (with the exception of $\Ssc$) are count variables and tend to be highly skewed, we used the transformation $x \to \log(x +1)$ to stabilize model fitting. 

We developed DN classification models by fitting a logistic regression with the above features based on  $\bthetahat\subSL$, $\bthetahat\subSSL$ and $\bthetahat\subDR$ obtained from density ratio weighted estimation. Since the proportion of observations with $\Ssc=0$ is relatively low in the labeled data, we implemented 100 replications of 6-fold CV procedure by splitting the data randomly within each strata to improve the stability of the CV procedure. To construct the basis for $\bthetahat\subSSL$ and $\bthetahat\subDR$, we used a natural spline with 3 knots on all covariates except $\Ssc$. To improve training stability, we set the ridge tuning parameter $\lambda_n=n^{-1}$ when fitting the imputation model. 

As shown in Table \ref{table: DN}(a), the point estimates are reasonably similar which confirms the consistency and stability of SS estimator in a real data setting. As expected, we find that the two most influential predictors are the diagnostic code for diabetic neuropathy and anti-diabetic medications. Importantly, we note substantial efficiency gains of $\bthetahat\subSSL$ compared to $\bthetahat\subSL$. The SSL estimates are $>50\%$ more efficient than the SL estimates for several features including six diagnostic code features and one NLP feature for DN. Additionally, $\bthetahat\subSSL$ is the most efficient estimator among all three approaches for nearly all variables. 


In Table \ref{table: DN}(b), we compare $\Dhat\subSL^w$, $\Dhat\subSSL^w$ and $\Dhat\subDR^w$ for the Brier score and OMR with $c=0.5$. While the point estimates for the accuracy measures based on these different approaches are relatively similar, $\Dhat\subSSL^w$ is 55\% more efficient than $\Dhat\subSL^w$ for the Brier score and 63\% more efficient for the OMR. Again, $\Dhat\subSSL^w$ is substantially more efficient than the DR estimator $\Dhat\subDR^w$. These results support the potential value of our method for EHR-based research as these gains in efficiency may be directly translated into requiring fewer labeled examples for model evaluation.




\section{Discussion}\label{sec-conc}


In this paper, we focused on the evaluation of a classification rule derived from a working regression model under stratified sampling in the SS setting.  In particular, we introduced a two-step imputation-based method for estimation of the Brier score and OMR that makes use of unlabeled data.  Additionally, as a by-product of our procedure, we obtained an efficient SS estimator of the regression parameter. Through theoretical and numerical studies, we demonstrated the advantage of the SS estimator over the SL estimator with respect to efficiency.  We also developed a weighted CV procedure to adjust for overfitting and a resampling procedure for making inference. Our numerical studies indicate that our proposed method outperforms the existing DR method for SSL in the finite sample studies {\red and we provide further discussion of this finding in Section \ref{sec:method-DR} of the Supplement}.  Importantly, this article is one of the first theoretical studies of labeling based on stratified sampling within the SSL literature.  We focus on the stratified sampling scheme due to its direct application to a variety of EHR-based analyses, including the development of a phenotyping algorithm for diabetic neuropathy presented in the previous section. 

In our numerical studies, we used spline functions with 3 or 4 knots and interaction terms for the imputation model. It would be possible to use more knots or add more features to the basis function for settings with a larger $n$.  However, care must be taken to avoid overfitting and potential loss in the efficiency gain of the SS estimator in finite sample.  Alternatively, other basis functions can be utilized provided that $\bPhi(\bu)$ contains $\bx$ in its components to ensure consistency of the regression parameter. In settings where nonlinear effects of $\bx$ on $y$ are present, it may be desirable to impose a more complex outcome model to improve the prediction performance.  A potential approach is to explicitly include nonlinear basis functions in the outcome model. In Section \ref{sec:sup:pc} of the Supplementary Materials, we consider using the leading principal components (PCs) of $\bx$ and $\bPsi(\bx)$ where $\bPsi(\cdot)$ is a vector of nonlinear transformation functions under a variety of settings.  This approach performs similarly or better than the commonly used random forest model with respect to predictive accuracy, suggesting the utility of our proposed methods in the presence of nonlinear effects. Our numerical results also illustrate the efficiency gain of the SS estimators of the Brier score and OMR relative to the SL and DR methods.  It is important to note, however, that the dimensions of both $\bx$ and $\bPhi$ were assumed to be fixed in our asymptotic analysis.  Accommodating more complex modeling with $p$ not small relative to $n$ requires extending the proposed SSL approach to settings where $\bx$ and $\bPhi$ are high dimensional.


For accuracy estimation, we proposed an ensemble CV estimator that eliminates first-order bias when the estimated regression parameter is the minimizer of the empirical performance measure. Though this condition may not hold when the outcome model is misspecified, we have found that the suggested weights perform well in our numerical studies.  Such ensemble methods that accommodate the more general case in both the supervised and SS settings warrant further research. Additionally, an important setting where our proposed SSL procedure would be of great use is in drawing inferences about two competing regression models.  As it is likely that at least one model is misspecified, we would expect to observe efficiency gains in estimating the difference in prediction error with the proposed method.

Lastly, while the present work focuses on the binary outcome along with the Brier score and OMR, the proposed SSL framework can potentially be extended to more general settings with continuous $y$ and/or other accuracy parameters. In particular, for binary $y$ and corresponding classification rule $\Ysc_2=I\{g(\btheta\trans\bx)>c\}$, it would be of interest to consider estimation of the sensitivity, specificity, and weighted OMR with different threshold values to analyze the costs associated with the false positive and false negative errors. 

\bibliographystyle{abbrvnat}      
\bibliography{SSLref.bib}  

\clearpage

\newpage

\begin{table}[!htbp]
    \caption{The 100$\times$ESE of $\Dhat\subSL^w$, $\Dhat\subSSL^w$ and $\Dhat\subDR^w$ under (i) ($\Msc_{\mbox{\tiny correct}}$, $\Isc_{\mbox{\tiny correct}}$); (ii)  ($\Msc_{\mbox{\tiny incorrect}}$, $\Isc_{\mbox{\tiny correct}}$), and (iii) ($\Msc_{\mbox{\tiny incorrect}}$, $\Isc_{\mbox{\tiny incorrect}}$). For $\Dhat\subSSL^w$, we also show the average of the 100$\times$ASE as well as the empirical coverage probability (CP) of the 95\% confidence intervals constructed based on the resampling procedure. }
    \label{table: resampling}
    \caption*{(i) ($\Msc_{\mbox{\tiny correct}}$, $\Isc_{\mbox{\tiny correct}}$)}
      \centering
      \begin{tabular}{c |c|cc|c|| c|cc|c|}
  \hline
 & \multicolumn{4}{c||}{Brier score} & \multicolumn{4}{c|}{OMR} \\ 
 \cline{2-9}
 & \multicolumn{1}{c|}{$\Dhat\subSL^w$}& \multicolumn{2}{c|}{$\Dhat\subSSL^w$}& \multicolumn{1}{c||}{$\Dhat\subDR^w$} 
 & \multicolumn{1}{c|}{$\Dhat\subSL^w$}& \multicolumn{2}{c|}{$\Dhat\subSSL^w$} & \multicolumn{1}{c|}{$\Dhat\subDR^w$}\\
  \hline
 & ESE & $\text{ESE}_ {\text{ASE}}$ & CP &ESE 
 & ESE& $\text{ESE}_ {\text{ASE}}$ & CP & ESE \\ \cline{2-9}
$S=2, n_s =100$  & $1.27$        & $1.19_{1.25}$ & $0.94$        & $1.31$        & $2.29$        & $2.10_{2.23}$ & $0.97$        & $2.35$  \\
$S=2, n_s =200$ & $0.97$        & $0.87_{0.86}$ & $0.93$        & $0.95$        & $1.67$        & $1.50_{1.60}$ & $0.97$        & $1.67$ \\ 
$S=4, n_s =100$  & $0.97$        & $0.89_{0.85}$ & $0.93$        & $0.95$        & $1.70$        & $1.54_{1.58}$ & $0.95$        & $1.64$    \\ 
$S=4, n_s =200$ & $0.67$        & $0.58_{0.60}$ & $0.95$        & $0.66$        & $1.16$        & $1.02_{1.11}$ & $0.97$        & $1.16$    \\ \hline
\end{tabular}
\end{table}

\begin{table}[!htb]
\ContinuedFloat
    \caption*{(ii) ($\Msc_{\mbox{\tiny incorrect}}$, $\Isc_{\mbox{\tiny correct}}$)}
    \centering
\begin{tabular}{c |c|cc|c|| c|cc|c|}
  \hline
 & \multicolumn{4}{c||}{Brier score} & \multicolumn{4}{c|}{OMR} \\ 
 \cline{2-9}
 & \multicolumn{1}{c|}{$\Dhat\subSL^w$}& \multicolumn{2}{c|}{$\Dhat\subSSL^w$}& \multicolumn{1}{c||}{$\Dhat\subDR^w$} 
 & \multicolumn{1}{c|}{$\Dhat\subSL^w$}& \multicolumn{2}{c|}{$\Dhat\subSSL^w$} & \multicolumn{1}{c|}{$\Dhat\subDR^w$}\\
  \hline
 & ESE & $\text{ESE}_ {\text{ASE}}$ & CP &ESE 
 & ESE& $\text{ESE}_ {\text{ASE}}$ & CP & ESE \\ \cline{2-9}
$S=2, n_s =100$  & $1.72$        & $1.47_{1.39}$ & $0.93$        & $1.85$        & $3.09$        & $2.56_{2.65}$ & $0.94$        & $3.42$  \\
$S=2, n_s =200$ & $1.13$        & $1.01_{0.94}$ & $0.92$        & $1.16$        & $2.14$        & $1.85_{1.85}$ & $0.95$        & $2.21$ \\ 
$S=4, n_s =100$& $1.21$        & $1.07_{0.96}$ & $0.92$        & $1.24$        & $2.13$        & $1.87_{1.88}$ & $0.95$        & $2.18$  \\ 
$S=4, n_s =200$ & $0.86$        & $0.73_{0.67}$ & $0.93$        & $0.87$        & $1.62$        & $1.37_{1.34}$ & $0.94$        & $1.62$    \\ \hline
\end{tabular}
\end{table}

\begin{table}[!htb]
\ContinuedFloat
    \caption*{(iii) ($\Msc_{\mbox{\tiny incorrect}}$, $\Isc_{\mbox{\tiny incorrect}}$)}
    \centering
\begin{tabular}{c |c|cc|c|| c|cc|c|}
  \hline
 & \multicolumn{4}{c||}{Brier score} & \multicolumn{4}{c|}{OMR} \\ 
 \cline{2-9}
 & \multicolumn{1}{c|}{$\Dhat\subSL^w$}& \multicolumn{2}{c|}{$\Dhat\subSSL^w$}& \multicolumn{1}{c||}{$\Dhat\subDR^w$} 
 & \multicolumn{1}{c|}{$\Dhat\subSL^w$}& \multicolumn{2}{c|}{$\Dhat\subSSL^w$} & \multicolumn{1}{c|}{$\Dhat\subDR^w$}\\
  \hline
 & ESE & $\text{ESE}_ {\text{ASE}}$ & CP &ESE 
 & ESE& $\text{ESE}_ {\text{ASE}}$ & CP & ESE \\ \cline{2-9}
$S=2, n_s =100$ & $1.57$        & $1.31_{1.29}$ & $0.94$        & $1.53$        & $2.68$        & $2.15_{2.31}$ & $0.96$        & $2.61$  \\
$S=2, n_s =200$ & $1.05$        & $0.86_{0.85}$ & $0.95$        & $1.01$        & $1.87$        & $1.45_{1.54}$ & $0.97$        & $1.83$   \\ 
$S=4, n_s =100$ & $1.11$        & $0.87_{0.88}$ & $0.96$        & $0.99$        & $1.97$        & $1.46_{1.57}$ & $0.96$        & $1.87$ \\ 
$S=4, n_s =200$ & $0.78$        & $0.61_{0.61}$ & $0.95$        & $0.73$        & $1.37$        & $1.02_{1.08}$ & $0.97$        & $1.30$  \\ \hline
\end{tabular}
\end{table}

\clearpage
\newpage
\begin{figure}[htpb!]
\centering
\caption{Percent biases of the apparent (AP), CV, and ensemble cross validation (eCV) estimators of the Brier score (BS) and overall misclassification rate (OMR) for SL, SSL and DR under (i) ($\Msc_{\mbox{\tiny correct}}$, $\Isc_{\mbox{\tiny correct}}$), (ii)  ($\Msc_{\mbox{\tiny incorrect}}$, $\Isc_{\mbox{\tiny correct}}$), and (iii)  ($\Msc_{\mbox{\tiny incorrect}}$, $\Isc_{\mbox{\tiny incorrect}}$). Shown are the results obtained with $K=6$ fold for CV.}
\label{figure:bias:RE}
\caption*{(i) ($\Msc_{\mbox{\tiny correct}}$, $\Isc_{\mbox{\tiny correct}}$)}
\begin{minipage}{1\textwidth}\centering
\mbox{
  \includegraphics[width=0.5\textwidth]{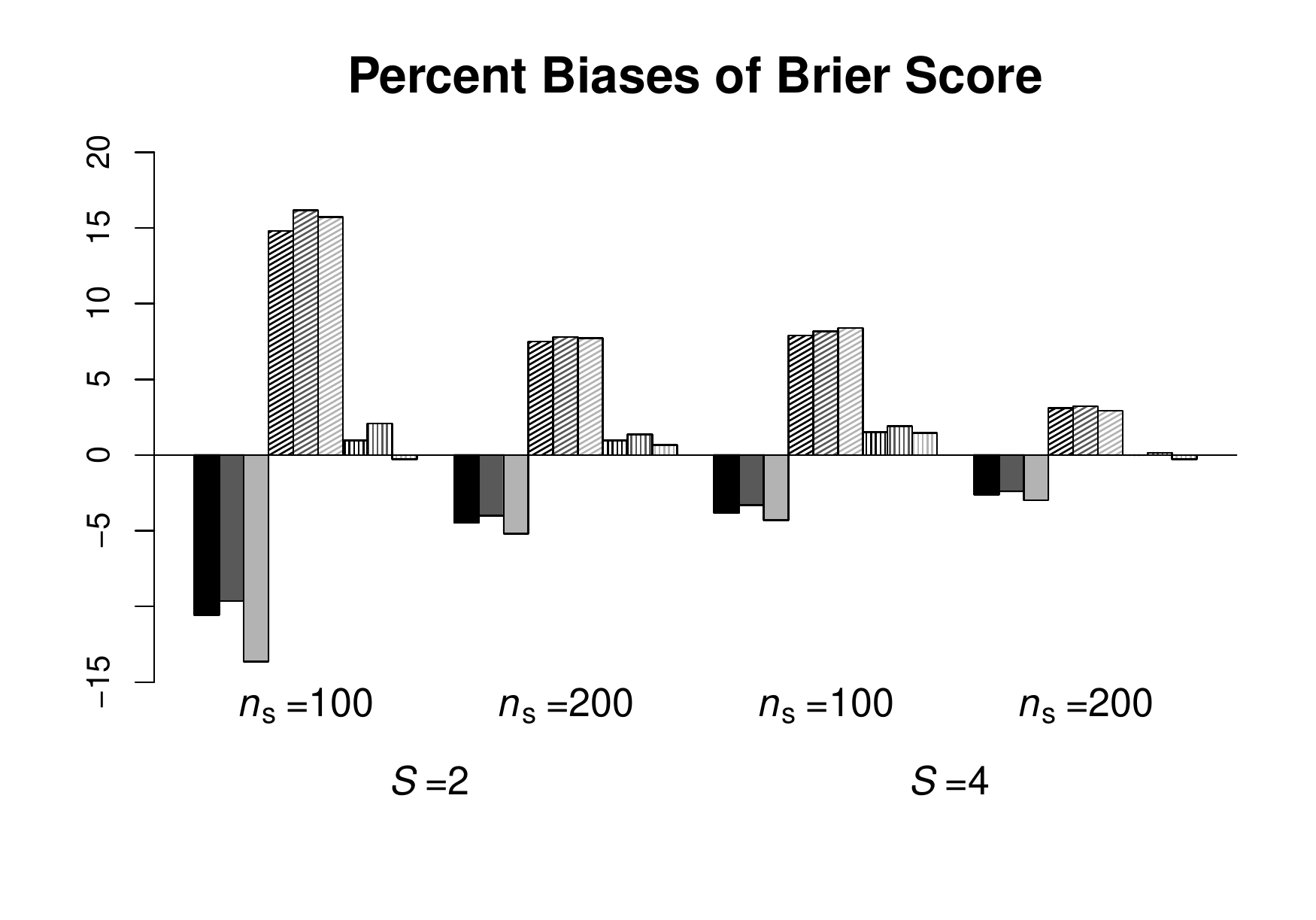}
  \includegraphics[width=0.5\textwidth]{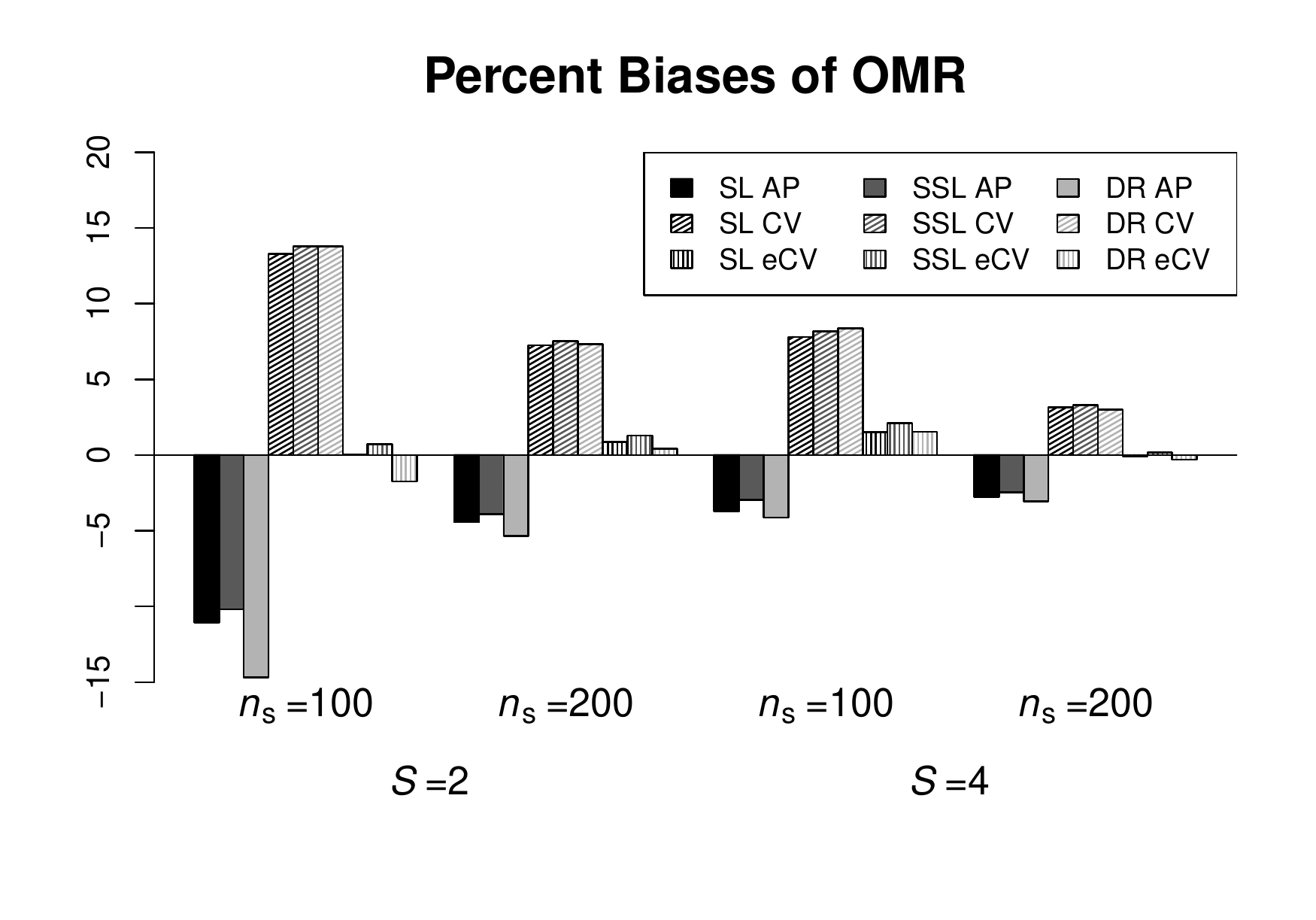}}
\end{minipage}
\caption*{(ii) ($\Msc_{\mbox{\tiny incorrect}}$, $\Isc_{\mbox{\tiny correct}}$)}
\begin{minipage}{1\textwidth}\centering
\mbox{
  \includegraphics[width=0.5\textwidth]{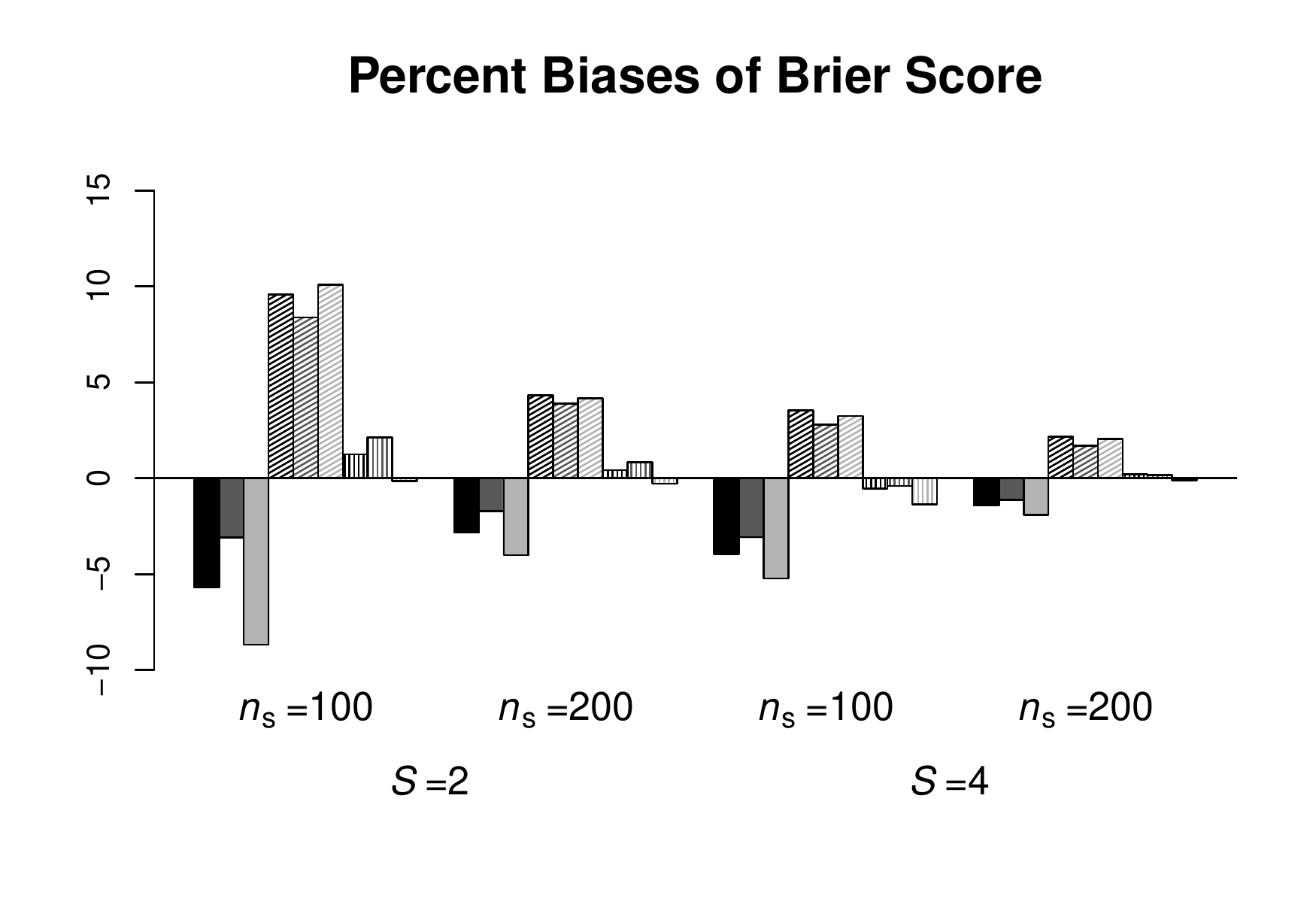}
  \includegraphics[width=0.5\textwidth]{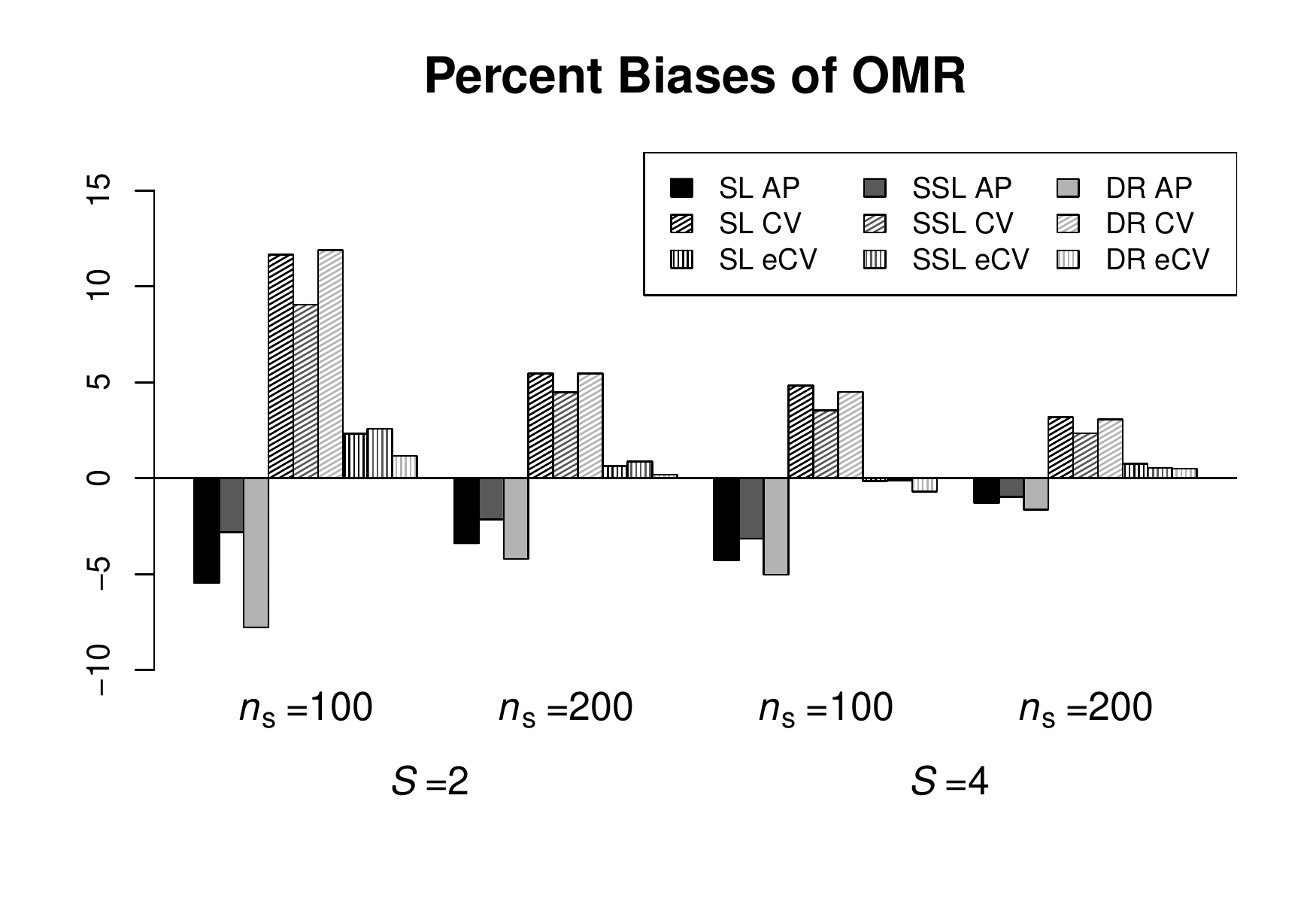}}
\end{minipage}
\caption*{(iii) ($\Msc_{\mbox{\tiny incorrect}}$, $\Isc_{\mbox{\tiny incorrect}}$)}
\begin{minipage}{1\textwidth}\centering
\mbox{
  \includegraphics[width=0.5\textwidth]{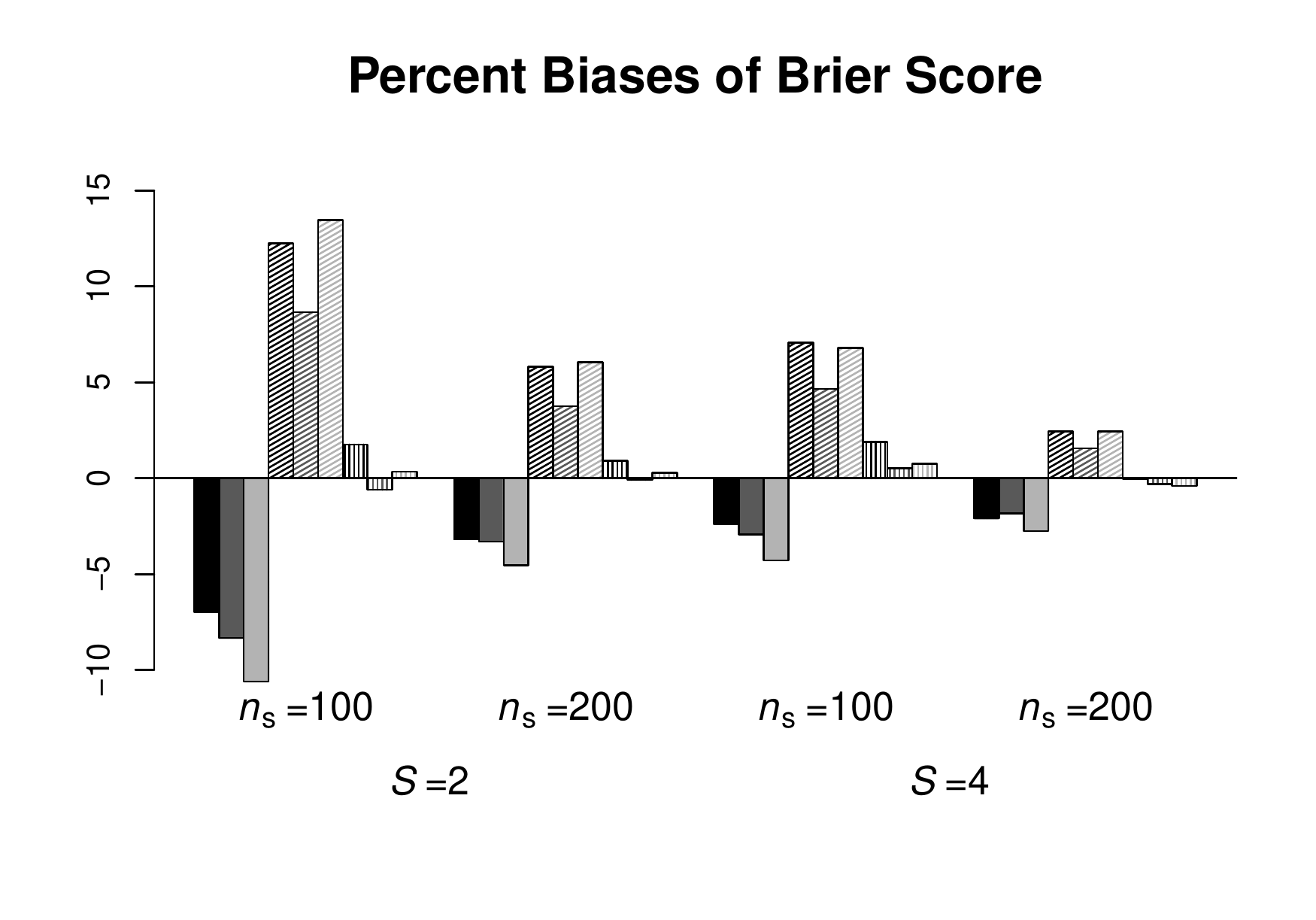}
  \includegraphics[width=0.5\textwidth]{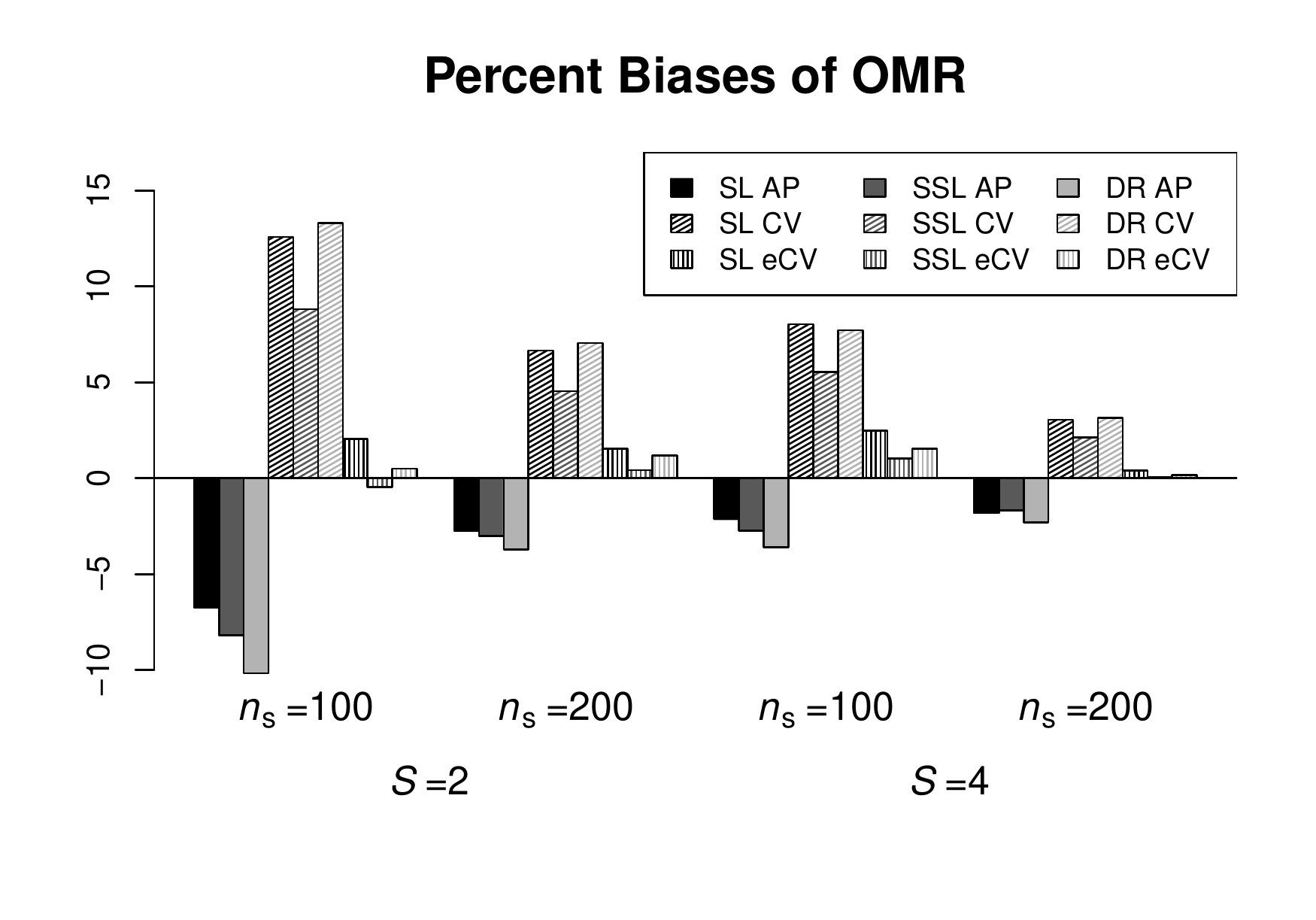}}
\end{minipage}
\end{figure}

\newpage

\begin{figure}[htpb!]
\centering
\caption{Relative efficiency (RE) of  $\Dhat\subSSL^{\omega}$ (SSL) and $\Dhat\subDR^{\omega}$ (DR) compared to $\Dhat\subSL^{\omega}$ for the Brier score (BS) and OMR under (i) ($\Msc_{\mbox{\tiny correct}}$, $\Isc_{\mbox{\tiny correct}}$), (ii) ($\Msc_{\mbox{\tiny incorrect}}$, $\Isc_{\mbox{\tiny correct}}$), and (iii) ($\Msc_{\mbox{\tiny incorrect}}$, $\Isc_{\mbox{\tiny incorrect}}$). Shown are the results obtained with $K=6$ fold CV.}
\label{figure:RE}
  \centering
  \includegraphics[scale = 0.46]  {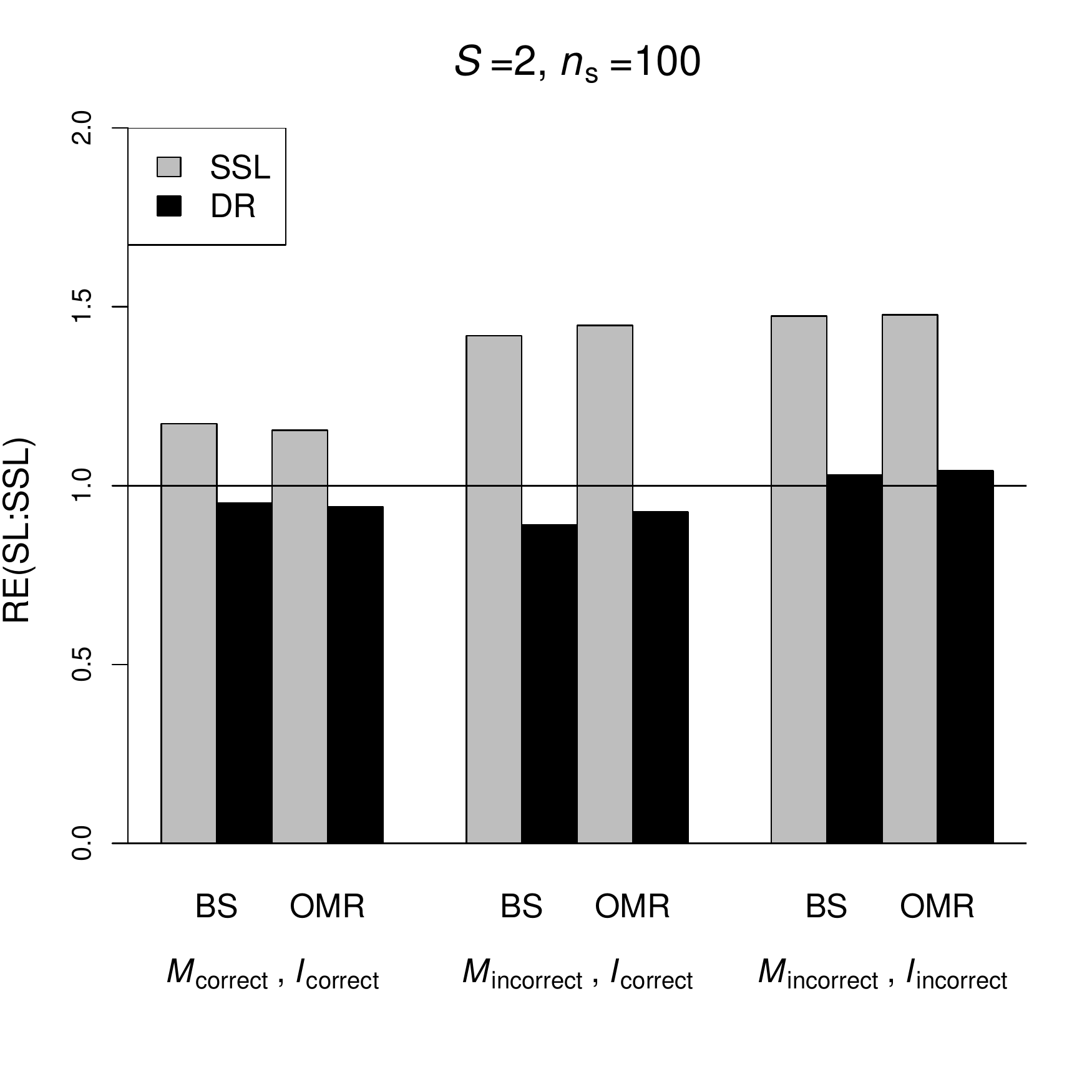}
  \includegraphics[scale = 0.46]  {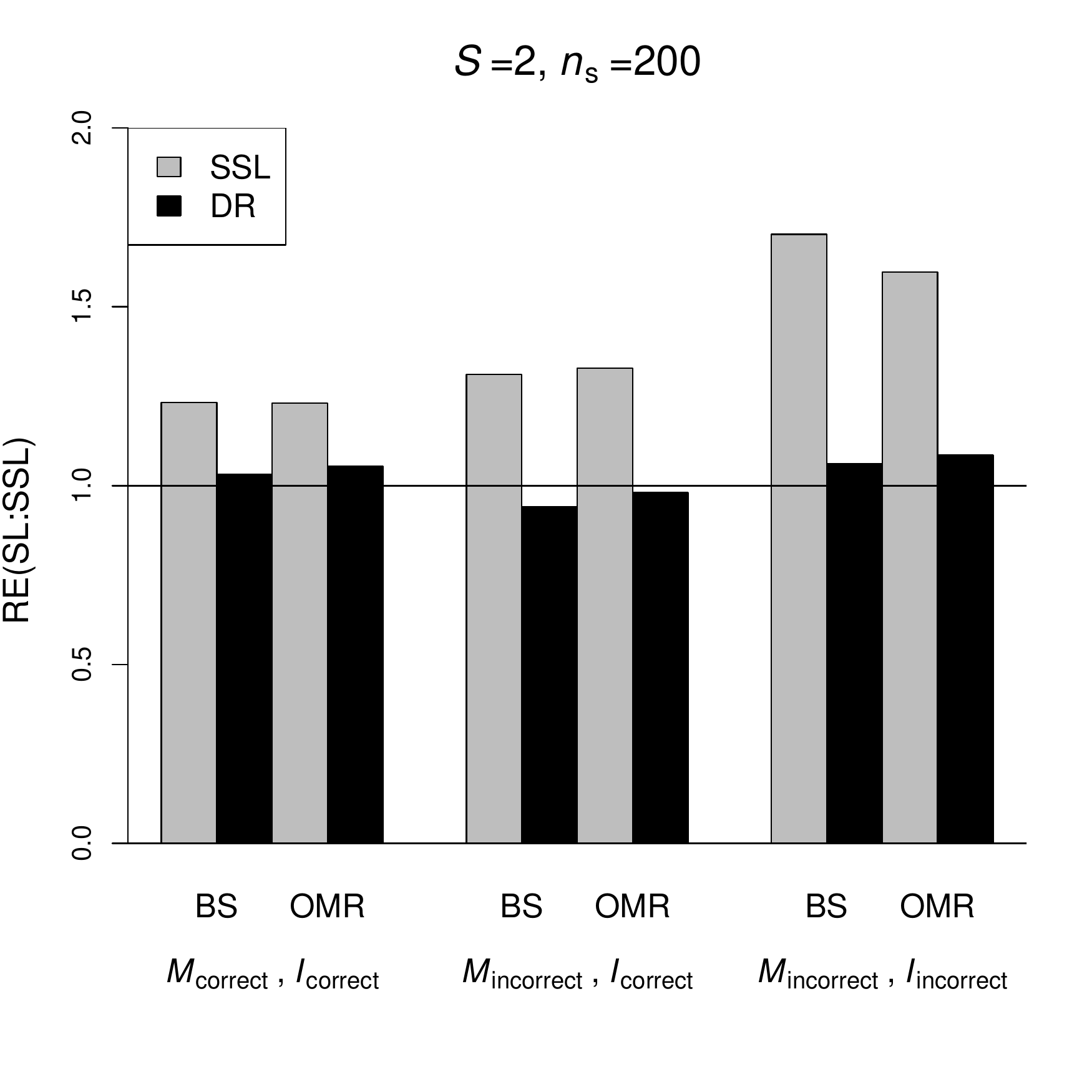}
    \includegraphics[scale = 0.46]  {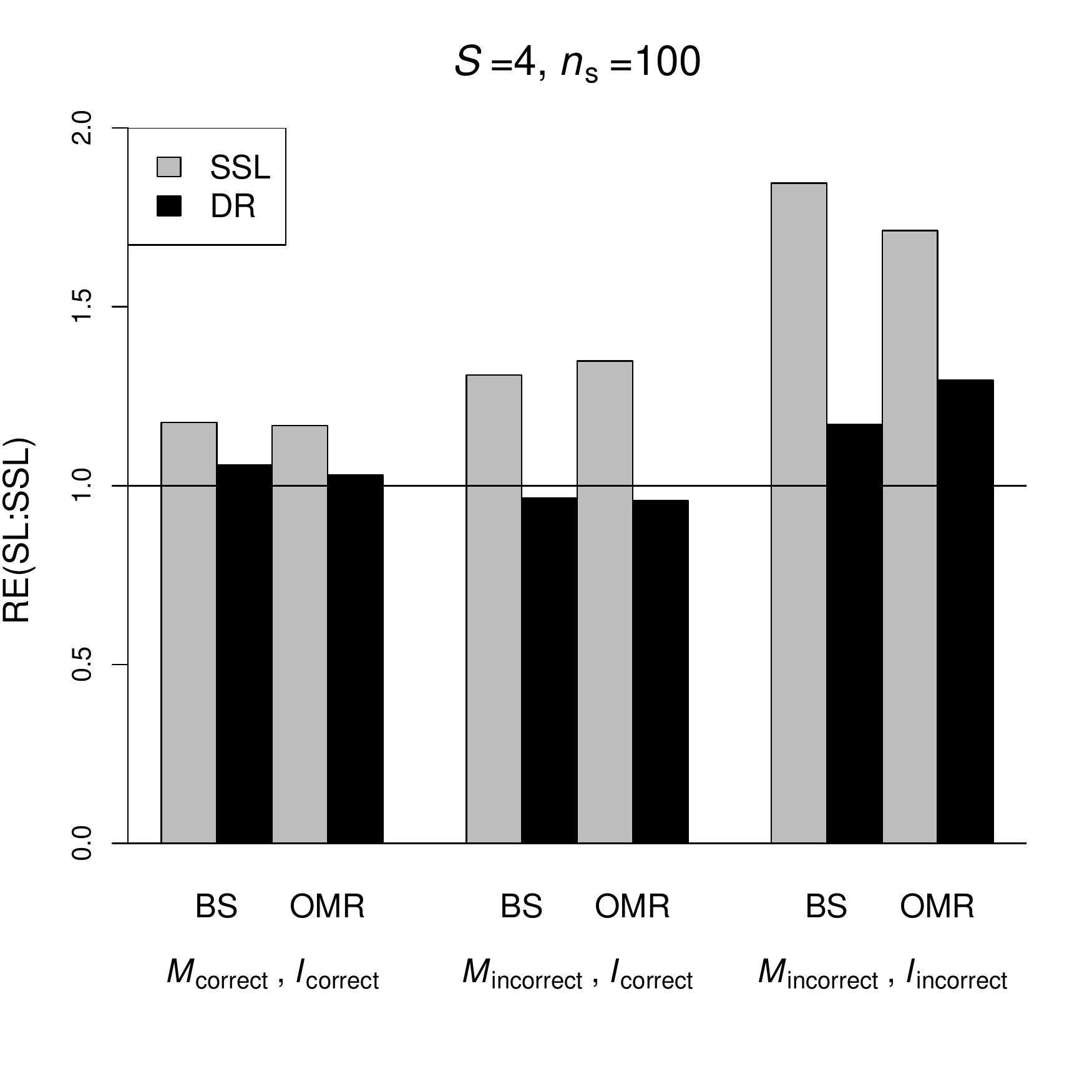}
  \includegraphics[scale = 0.46]  {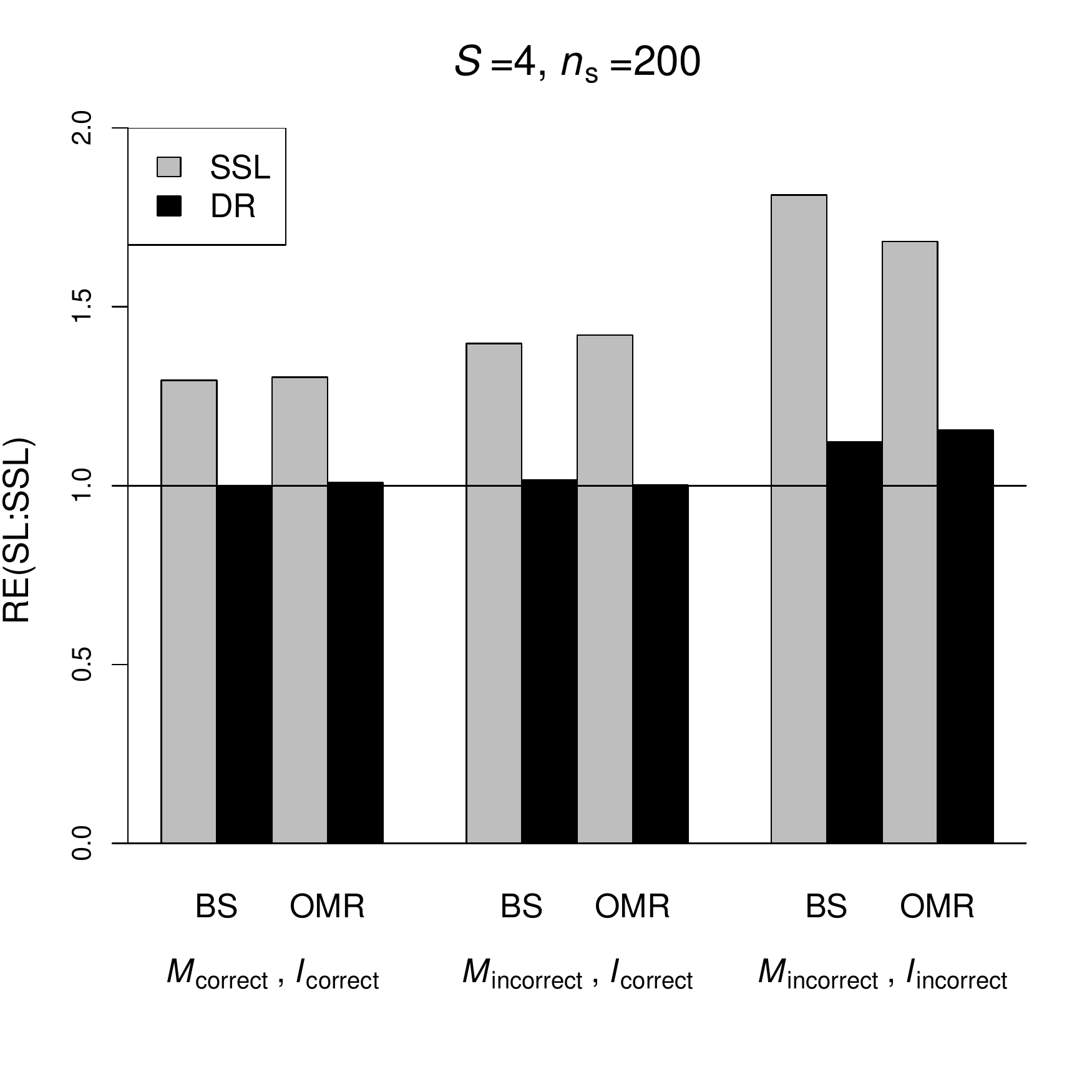}
  \label{fig:test1}
\end{figure}

\newpage

\begin{table}[htpb!]
\begin{center}
\caption{{Results from the diabetic neuropathy EHR study: (a) estimates (Est.) of the regression parameters for both codified (COD) and NLP features based on $\bthetahat\subSL$, $\bthetahat\subSSL$ and $\bthetahat\subDR$ along with their estimated SEs and the coordinate-wise relative efficiencies (RE) of $\bthetahat\subSSL$ compared to $\bthetahat\subSL$ and (b) $\Dhat\subSL^{\omega}$, $\Dhat\subSSL^{\omega}$ and $\Dhat\subDR^{\omega}$ along with their estimated SEs and relative efficiencies (RE) of $\Dhat\subSSL^{\omega}$ and $\Dhat\subDR^{\omega}$ compared to $\Dhat\subSL^{\omega}$.}}
\label{table: DN}
\centerline{(a) Estimates of the Regression Coefficients}
\begin{tabular}{cl rr | rrr| rrr}  \hline\hline
&   &  \multicolumn{2}{c|}{$\bthetahat\subSL$} & \multicolumn{3}{c|}{$\bthetahat\subSSL$}  & \multicolumn{3}{c}{$\bthetahat\subDR$}\\
&    & Est. & SE & Est. & SE & RE & Est. & SE & RE \\    \hline
&  Intercept & -3.89 & 0.80 & -3.85 & 0.73 & {1.21} & -3.26 & 0.66 & {\bf 1.45}\\  \hline
    & Neurological exam \& note & -1.09 & 0.81 & -0.90 & 0.61 & {\bf 1.77} & -2.38 & 1.16 & {0.48}\\ 
 	&  Diabetes & -0.24 & 0.57 & 0.02 & 0.43 & {\bf 1.79} & -0.09 & 0.48 & {1.42}\\ 
	&  Type 2 Diabetes Mellitus & -1.05 & 0.73 & -0.52 & 0.60 & {\bf{1.62}} & -0.62 & 0.86 & {0.73}\\ 
	&  Diabetic Neuropathy & 1.99 & 0.86 & 1.79 & 0.68 & {\bf{1.58}} & 1.66 & 1.08& {0.63} \\ 
COD	&  Anti-diabetic Meds & 1.70 & 0.59 & 1.12 & 0.44 & {\bf 1.79} & 1.50 & 0.51 & 1.32\\ 
	& Diabetes Mellitus with  & \raisebox{-2ex}[0cm][0cm]{0.36} & \raisebox{-2ex}[0cm][0cm]{1.17} &\raisebox{-2ex}[0cm][0cm]{0.60} & 
	\raisebox{-2ex}[0cm][0cm]{0.98} & \raisebox{-2ex}[0cm][0cm]{{1.42}} & \raisebox{-2ex}[0cm][0cm]{0.59} & \raisebox{-2ex}[0cm][0cm]{0.90}& \raisebox{-2ex}[0cm][0cm]{\bf 1.68} \\
	& \qquad Neuro Manifestation & \\
	& Other Idiopathic Peripheral  & \raisebox{-2ex}[0cm][0cm]{0.86} & \raisebox{-2ex}[0cm][0cm]{0.71} &\raisebox{-2ex}[0cm][0cm]{0.93} & 
	\raisebox{-2ex}[0cm][0cm]{0.68} & \raisebox{-2ex}[0cm][0cm]{{\bf 1.09}} & \raisebox{-2ex}[0cm][0cm]{1.01} & \raisebox{-2ex}[0cm][0cm]{0.76} & \raisebox{-2ex}[0cm][0cm]{0.89}\\
	& \qquad Autonomic Neuropathy & \\
	& Normal Glucose & -0.56 & 0.65 & -0.20 & 0.51 & {\bf{1.64}} & -1.27 & 0.80 & {0.66}\\ \hline
 	&  Diabetic & 0.30 & 0.58 &-0.57 & 0.49 & {\bf{1.39}} & 0.14 & 0.65 & {0.80}\\ 
NLP	&  HgA1c & -0.52 & 0.75 & -0.64 & 0.70 & {\bf{1.16}} & -0.67 & 0.85 & {0.79}\\ 
	&  Neuropathy & 0.27 & 0.58 & 0.30 & 0.47 & {\bf{1.55}} & 0.37 & 0.54 & {1.15}\\   
	\hline
\end{tabular}\vspace{.2in}

\centerline{(b) Estimates of the Accuracy Parameters ($\times 100$)}
\begin{tabular}{crr|rrr|rrr}
  \hline
\multicolumn{1}{r}{} &  \multicolumn{2}{c|}{$\Dhat\subSL^{\omega}$} & \multicolumn{3}{c|}{$\Dhat\subSSL^{\omega}$} & \multicolumn{3}{c}{ $\Dhat\subDR^{\omega}$}\\
    &  Est. & SE  & Est. & SE & RE & Est. & SE & RE\\   \hline
Brier score &  8.97 & 2.09 & 9.59 & 1.68 & {\bf{1.55}} & 9.60 & 1.87 & 1.26\\ 
OMR & 12.87 & 3.25  & 14.01 &  2.55 & {\bf{1.63}} & 12.29 & 3.04 & 1.14 \\ 
   \hline
\end{tabular}
\end{center}
\end{table}

\clearpage
\newpage

\appendix

\section*{Appendix}
\numberwithin{equation}{section}
\setcounter{lemma}{0}
\setcounter{cond}{0}
\setcounter{table}{0}
\setcounter{figure}{0}
\setcounter{remark}{0}
\setcounter{theorem}{0}

\renewcommand{\thelemma}{A\arabic{lemma}}
\renewcommand{\thetheorem}{A\arabic{theorem}}
\renewcommand{\thecond}{A\arabic{cond}}
\renewcommand{\thefigure}{A\arabic{figure}}
\renewcommand{\thetable}{A\arabic{table}}
\renewcommand{\theremark}{A\arabic{remark}}

Here we provide justifications for our main theoretical results. The following lemma confirming the existence and uniqueness of the limiting parameters $\bthetabar$, $\bgammabar$, and $\bnubar_{\btheta}$ will be used in our subsequent derivations.
\begin{lemma}
Under conditions \ref{cond:1}-\ref{cond:2}, unique $\bthetabar$ and $\bgammabar$ exist. In addition, there exists $\delta>0$ such that a unique $\bnubar_{\btheta}$ exists for any $\btheta$ satisfying $\|\btheta-\bthetabar\|_2<\delta$.
\label{lem:1}
\end{lemma}

\begin{proof}
Conditions \ref{cond:2} (A) and (B) imply that there is no $\btheta$ and $\bgamma$ such that with non-trivial probability,
\[
I(Y_1>Y_2)=I(\btheta\trans\bx_1>\btheta\trans\bx_2),
\]
and
\[
I(Y_1>Y_2)=I(\bgamma\trans\bPhi_1>\bgamma\trans\bPhi_2).
\]
It follows directly from Appendix I of \cite{tian2007model} that there exist finite $\bthetabar$ and $\bgammabar$ that solve $U(\bthetabar)$ and $Q(\bgammabar)$, respectively,  and that $\bthetabar$ and $\bgammabar$ are unique. For $\btheta \in \bTheta,$ there exists no $\bnu$ such that
\[
I(Y_1>Y_2)=I(\bgammabar\trans\bPhi_1+\bnu\trans\bz_{1\btheta}>\bgammabar\trans\bPhi_2+\bnu\trans \bz_{2\btheta}),
\]
or 
\[
I(Y_1>Y_2)=I(\bgammabar\trans\bPhi_1+\bnu\trans \bz_{1\btheta}<\bgammabar\trans\bPhi_2+\bnu\trans\bz_{2\btheta}),
\]
which similarly implies there exists a finite $\bnubar_{\btheta}$ that is the solution to $\bR(\bnu|\btheta)=0$. The solution is also unique as  $E \left[\bz_{\btheta}^{\otimes2}\dot{g}\{\bar{\bgamma}\trans \bPhi+\bnu\trans\bz_{\btheta}\}\right] \succ\bzero$ for any $\bnu.$
\end{proof}

\section{Estimation Procedure for $\bthetahat \subSSL$ }\label{sec:ss-mod-param}
We propose to obtain a simple SS estimator for $\bthetabar$, $\widecheck \btheta \subSSL$, as the solution to  
 \begin{equation}
 \bUhat_{N}(\btheta)= \frac{1}{N}\sum_{i = 1}^{N} \bx_i \{ g(\bgammatilde \trans \bPhi_i) - g(\btheta \trans \bx_i)  \} =\bzero.
 \end{equation}
Note that when $\bu$ includes the stratum information $\Ssc$ and the imputation model (\ref{impmod}) is correctly specified, there is no need to use the inverse probability weighted (IPW) estimating equation with $\what_i$ as in (\ref{impee}). However, under the general scenario where the imputation model may be misspecified, the unweighted estimating equation is not guaranteed to provide an asymptotically unbiased estimate of $\bgammabar$ and necessitates the use of the IPW approach. 

As detailed in the subsequent sections, when the imputation and outcome models are misspecified, the efficiency gain of $\widecheck\btheta\subSSL$ relative to $\bthetahat\subSL$, is not theoretically guaranteed. We therefore obtain the final SS estimator, denoted as $\bthetahat\subSSL=(\thetahat_{\SSL,0},  \dots,\thetahat_{\SSL,p})\trans$, as a linear combination of $\bthetahat\subSL$ and $\widecheck \btheta\subSSL$ to minimize the asymptotic variance. For simplicity we consider here the component-wise optimal combination of the two estimators. That is, the $j$th component of $\bthetabar$, $\bar{\theta}_{j}$, is estimated with
$$\thetahat_{\SSL,j} =  \What_{1j} \widecheck \theta_{\SSL,j} +  (1-\What_{1j}) \thetatilde_{\SL, j}$$
where $\What_{1j}$ is the first component of the vector $\bWhat_j =\bone\trans \bSigmahat_j^{-1}/(\bone\trans \bSigmahat_j^{-1} \bone)$ and  $\bSigmahat_j$ is a consistent estimator for $\cov\{(\widecheck \theta_{\SSL,j}, \thetatilde_{\SL,j})\trans\}$.  To estimate the variance of $\bthetahat \subSSL$, one may rely on estimates of the influence functions of $\widecheck \btheta \subSSL$ and $\bthetahat \subSL$. To avoid under-estimation in a finite sample, we obtain bias-corrected estimates of the influence functions via $K$-fold CV. Details on the $K$-fold CV procedure, as well as the computations for the aforementioned estimation of $\bSigmahat_j$ and $\What_{1j}$, are given in Appendix \ref{sec: inf-thetaSSL}.  

\section{Cross-validation Based Inference for $\bthetahat \subSSL$}\label{sec: inf-thetaSSL}
Here we provide the details of the procedure to obtain $\bthetahat \subSSL$ as well as an estimate of its variance.  We employ $K$-fold CV in the proposed procedure to adjust for overfitting in finite sample and denote each fold of $\Lscr$ as $\Lscr_k$ for $k = 1, \dots, K$.  First, we estimate $\bSigmahat_j$ with 
\begin{equation*}
{n^{-1}}\sum_{k = 1}^K \sum_{i \in \Lscr_k} \{ \Wsc_j( \bgammatilde_{(\text{-}k)} ,\bD_i),  \Vsc_j(\bthetahat_{(\text{-}k)}, \bD_i) \}\{ \Wsc_j(\bgammatilde_{(\text{-}k)} ,\bD_i),  \Vsc_j(\bthetahat_{(\text{-}k)}, \bD_i) \} \trans 
\end{equation*} 
where  $\bgammatilde_{(\text{-}k)}$ is the estimator of $\bgammabar$ based on $\Lscr/\Lscr_k$,  $\bthetahat_{(\text{-}k)}$ is the supervised estimator of $\bthetabar$ based on $\Lscr/\Lscr_k$,
\begin{align*}
 \Wsc(\bgammatilde_{(\text{-}k)} ,\bD_i) =    \bAhat^{-1}  [\varpi_i \bx_i\{ y_i - g(\bgammatilde_{(\text{-}k)} \trans  \bPhi_i )\}] ,   \quad& \Vsc(\bthetahat_{(\text{-}k)}, \bD_i)  =  \bAhat^{-1}  [\varpi_i \bx_i\{ y_i - g(\bthetahat_{(\text{-}k)}\trans  \bx_i)\}],  \\ 
 \bAhat  =  N^{-1} \sum_{i = 1}^{N} \bx_i^{\otimes 2}  \dot{g}(\widecheck \btheta \subSSL \trans  \bx_i) &\quad \mbox{and} \quad \varpi_i = \what_i n/ N
\end{align*}
In practice,  $\bSigmahat_j$ may be unstable due to the high correlation of $\widecheck \btheta \subSSL$ and $\bthetahat \subSL$.  One may use a regularized estimator $(\bSigmahat_j + \delta_n \bI)^{-1}$ with some $\delta_n = O(n^{-\frac{1}{2}})$ to stabilize the estimation and obtain $\bthetahat \subSSL$ accordingly.  The covariance of $\bthetahat \subSSL$ may then  be consistently estimated with
\begin{equation}
n^{-1}\sum_{k = 1}^K \sum_{i \in \Lscr_k}  \{ \Zsc( \bgammatilde_{(\text{-}k)}, \bthetahat_{(\text{-}k)}, \bD_i)\}\{ \Zsc( \bgammatilde_{(\text{-}k)}, \bthetahat_{(\text{-}k)}, \bD_i)\} \trans \quad \mbox{where}
\end{equation}
\begin{equation}
 \Zsc(  \bgammatilde_{(\text{-}k)}, \bthetahat_{(\text{-}k)}, \bD_i) = \bWhat \{\Wsc(  \bgammatilde_{(\text{-}k)} ,\bD_i) \} + ( \bI - \bWhat) \{\Vsc(\bthetahat_{(\text{-}k)},\bD_i) \}
 \end{equation} 
$\bWhat = \text{diag}(\What_{10},  \dots, \What_{1p})$ is an estimate of  $\bW = \text{diag}(W_{10},  \dots, W_{1p})$, $\What_{1j}$ is the first component of $\bWhat_j =\bone\trans \bSigmahat_j^{-1}/(\bone\trans \bSigmahat_j^{-1} \bone)$ and $W_{1j}$ is the first component of $\bW_j =\bone\trans \bSigma_j^{-1}/(\bone\trans \bSigma_j^{-1} \bone)$.  The confidence intervals for the regression parameters can be constructed with the proposed variance estimates and the asymptotic normal distribution of the SS estimator.

\section{Asymptotic Properties of $ \bthetahat \subSL$}\label{sec: asym-thetaSL}
The main complication in deriving the asymptotic properties of the SL estimators arises from the fact that 
$P(V_i = 1 \mid \Fscr) = \pihat_{\Ssc_i}  \to 0$ as $n \to \infty$ and hence $\what_i=V_i/\pihat_{\Ssc_i}$ is an ill behaved random variable tending to infinity in the limit for those with $V_i=1$. {This substantially distinguishes the SS setting from the standard missing data literature}. To overcome this complication, we note that for subjects in the labeled set, $\Ninv \what_i= \sum_{s=1}^S I(\Ssc_i = s) \rhohat_s\ninv_s V_i$ and $V_i \mid \Ssc_i = s,$ are independent identically distributed (i.i.d) random variables since the labeled observations are drawn randomly within each stratum. {\red Also note that $ \rhohat_s\overset{p}\rightarrow\rho_s$ as assumed in Section \ref{sec-data-struc}.} Hence for any function $f$ with $\var\{f(\bF_i) \mid S_i=s\} < \infty$, 
\begin{align}
 N^{-1}\sumiN \what_i f(\bF_i)  & = \sum_{s=1}^S \rhohat_s \ninv_s\sum_{i=1}^N  V_iI(\Ssc_i = s) f(\bF_i)\\
 &= \sum_{s=1}^S {\red \{\rho_{s} +o_p(1)\}} \left\{\ninv_s \sum_{i=1}^N  V_i I(\Ssc_i = s) f(\bF_i) \right\} + o_p(1)  \nonumber \\
& = \sumsS \rho_s  {\rm E}\{ f(\bF_i) \mid \Ssc_i = s\} + o_p(1) = {\rm E}\{ f(\bF_i) \}+ o_p(1)  . \label{eq-IPW}
\end{align}

We begin by verifying that $\bthetahat \subSL$ is consistent for $\bthetabar$.  It suffices to show that  (i) $ \sup_{\btheta\in\mathbf{\Theta}} \| \bU_n(\btheta) - \bU(\btheta)   \|_2 = o_p(1)$ and (ii) $\inf_{\|\btheta-\bthetabar\|_2>\epsilon}\|\bU(\btheta)\|_2>0$ for any $\epsilon>0$ \citep[Lemma 2.8]{newey1994large}.  To this end, we write
\begin{align*}
\bU_n( \btheta) &= \frac{1}{N}\sum_{i = 1}^N \what_i \bx_i \{ y_i -g(\btheta \trans\bx_i) \} =  \sum_{s=1}^S  \rho_s   \left[n_s^{-1} \sum_{V_i = 1} I(\Ssc_i = s) \bx_i \{ y_i -g(\btheta \trans\bx_i) \}\right] + o_p(1) .
\end{align*}
Under Conditions \ref{cond:1}-\ref{cond:2}, $\bx$ belongs to a compact set and $\dot g(\btheta\trans\bx)$ is continuous and uniformly bounded for $\btheta\in\mathbf{\Theta}$. From the uniform law of large numbers (ULLN) \citep[Theorem 8.2]{pollard1990empirical}, 
$n_s^{-1} \sum_{V_i = 1} I(\Ssc_i = s) \bx_i \{ y_i -g(\btheta \trans\bx_i) \}$ converges to $E\left[\bx_i \{ y_i -g(\btheta \trans\bx_i) \}\mid \Ssc_i=s \right]$ in probability uniformly as $n\rightarrow \infty$ and
$ \sup_{\btheta\in\mathbf{\Theta}} \| \bU_n(\btheta) - \bU(\btheta)   \|_2 = o_p(1)$. Furthermore, (ii) follows directly from Lemma \ref{lem:1} and consequently $\bthetahat \subSL \overset{p} \rightarrow \bthetabar$ as $n \to \infty$.

Next we consider the asymptotic normality of $\Wschat\subSL = \nhalf( \bthetahat \subSL - \bthetabar)$. Noting that $\bthetahat \subSL \overset{p} \rightarrow \bthetabar$ {\red and $ \rhohat_s\overset{p}\rightarrow\rho_s$}, we apply Theorem 5.21 of \cite{van2000asymptotic} to obtain the Taylor expansion
\begin{align*}
\Wschat\subSL = \nhalf( \bthetahat \subSL - \bthetabar) &=  \nhalf \frac{1}{N}   \sum_{i = 1}^N \what_i\bA^{-1}\bx_i \{y_i - g(\bthetabar \trans \bx_i)   \} + o_p(1)   \\
&= \nhalf \sum_{s=1}^S {\red \{\rho_s+o_p(1)\}} \left\{ \ninv_s \sum_{V_i=1} I(\Ssc_i = s) \be_{\SL i}  \right\}+ o_p(1),
\end{align*}
where $\be_{\SL i} = \bA^{-1}\bx_i \{y_i - g(\bthetabar \trans \bx_i) \}$ and $\bA = {\rm E} \{ \bx_i^{\otimes 2}  \dot{g}(\bthetabar \trans \bx_i)\}$. It then follows by the {\red{classical Central Limit Theorem}} that $\Wschat\subSL \to N(0, \bSigma\subSL)$ in distribution and {\red 
\[
\Wschat\subSL=\nhalf \sum_{s=1}^S \rho_s \left\{ \ninv_s \sum_{V_i=1} I(\Ssc_i = s) \be_{\SL i}  \right\}+ o_p(1),
\]}
where
$\bSigma\subSL =  \sum_{s=1}^S \rho_s^2 \rho_{1s}^{-1} E\{\be_{\SL i}^{\otimes 2} \mid \Ssc_i = s \}.$

\section{Asymptotic Properties of $\Dhat \subSL( \bthetahat \subSL)$}\label{sec: asym-DSL}
We begin by showing that  $\Dhat \subSL( \bthetahat \subSL) \overset{p} \to D(\bthetabar)$ as $n \to \infty$.  We first note that {\red since $ \rhohat_s\overset{p}\rightarrow\rho_s$},
\begin{equation}
\Dhat \subSL( \btheta) = \sum_{s= 1}^S \rho_s  \left[ \nsinv  \sum_{V_i=1} I(\Ssc_i = s) d\{y_i, \yscrl(\btheta \trans \bx_i)\} \right] + o_p(1).
\end{equation}
It follows by the ULLN that $\sup_{\btheta\in\mathbf{\Theta}} | \Dhat\subSL( \btheta) - D( \btheta) | = o_p(1)$ since $d\{y,\yscrl(\btheta \trans \bx\}$ is continuously differentiable in $\btheta$ and uniformly bounded.  The consistency of $\Dhat \subSL( \bthetahat\subSL)$ for $D(\bthetabar)$ then follows from the fact that $\bthetahat \subSL$ converges in probability to $\bthetabar$ as $n \to \infty$.

To establish the asymptotic distribution of $\Tschat\subSL =\nhalf \{ \Dhat \subSL( \bthetahat \subSL)  - D(\bthetabar)\}$, we first consider $\Tsctilde\subSL(\btheta) = \nhalf \{ \Dhat\subSL( \btheta)  - D(\btheta)\}$. We verify that there exists $\delta>0$, such that the classes of functions indexed by $\btheta$:
\begin{equation*}
\begin{split}
&\Bsc_1 = \{I(\Ssc = s) |y - I\{g(\btheta \trans \bx) > c\}| :  \| \btheta - \bthetabar \|_2 < \delta  \} \ \\
\mbox{and}\quad&\Bsc_2 = \{I(\Ssc = s) [y - g(\btheta \trans \bx) ]^2 :  \| \btheta - \bthetabar \|_2 < \delta  \}
\end{split}    
\end{equation*}
are Donsker classes. For $\Bsc_1$, we note that $\{ I\{g(\btheta \trans \bx) > c\} :  \| \btheta - \bthetabar \|_2 < \delta  \}$ is a Vapnik-Chervonenkis class  \citep[Page 275]{van2000asymptotic} and thus 
$$\Bsc_1 = \{ I(s = s_0) [ I(y = 0) I\{g(\btheta \trans \bx) > c\} + I(y = 1) I\{g(\btheta \trans \bx) \le  c\}  ]:   \| \btheta - \bthetabar \|_2 < \delta  \}$$
is a Donsker class.  For $\Bsc_2$, $[y - g(\btheta \trans \bx) ]^2$ is continuously differentiable in $\btheta$ and uniformly bounded by a constant.  It follows that $\Bsc_2$ is a Donsker class \citep[Example 19.7]{van2000asymptotic}. By Theorem 19.5 of \cite{van2000asymptotic} we then have 
\[
\nhalf \{ \Dhat \subSL( \btheta)  - D(\btheta)\} =  \sum_{s= 1}^S \rho_s\rho_{1s}^{-\half}   \nsnhalf \sum_{V_i = 1}  I(\Ssc_i = s) [d\{y_i, \yscrl(\btheta \trans \bx_i)\} - D(\btheta)]   + o_p(1),
\]
which converges weakly to a mean zero Gaussian process index by $\btheta$. Thus, $\nhalf \{ \Dhat \subSL ( \btheta)  - D(\btheta)\}$ is stochastically equicontinuous at $\bthetabar$. In addition, note that $D(\btheta)$ is continuously differentiable at $\bthetabar$ {\red and $\rhohat_s\overset{p}{\to}\rho_s$}. It then follows that
\begin{align*}
\Tschat\subSL =&\nhalf \{ \Dhat \subSL( \bthetahat \subSL)  - D(\bthetahat \subSL)\} +\nhalf \{ D(\bthetahat \subSL)  - D(\bthetabar)\} \\
=&\nhalf \sum_{s= 1}^S \rho_s  \left( \ninv_s \sum_{V_i = 1}  I(\Ssc_i = s) \left[d(y_i, \Yscbar_i)   - D(\bthetabar) + \dot{D}(\bthetabar)\trans \be_{\SL i}  \right] \right) + o_p(1),
\end{align*}
which converges in distribution to $N(0, \sigma^2\subSL)$ where 
$$\sigma^2\subSL =  \sum_{s= 1}^S \rho_s^2\rho_{1s}^{-1}  {\rm E}[ \{ d(y_i, \Yscbar_i)   - D(\bthetabar) + \dot{\bD}(\bthetabar)\trans \be_{\SL i} \}^2  \mid \Ssc_i = s ].$$



\section{Asymptotic Properties of $\widecheck \btheta \subSSL$}\label{sec: asym-thetaSSL}
We first consider the asymptotic properties of $\bgammatilde$. Under Conditions \ref{cond:1}-\ref{cond:2} and using that $\lambda_n = o(n^{-\frac{1}{2}})$ and {\red $ \rhohat_s\overset{p}\rightarrow\rho_s$}, we can adapt the same procedure as in Appendix \ref{sec: asym-thetaSL} to show that $\bgammatilde \overset{p} \to \bgammabar$ as $n \to \infty$.  We obtain the following Taylor series expansion
\begin{align*}
\nhalf( \bgammatilde- \bgammabar) &=    \nhalf \sum_{s= 1}^S \rho_s  \left[ \ninv_s \sum_{i = 1}^n  I(\Ssc_i = s) \bC^{-1}\bPhi_i \{y_i - g(\bgammabar \trans \bPhi_i) \} \right] + o_p(1)
\end{align*}
where $\Cbb = {\rm E} \{ \bPhi_i^{\otimes 2} \dot{g}(\bgammabar \trans \bPhi_i) \} $.  We then have that $\nhalf( \bgammatilde- \bgammabar)$ converges to zero-mean Gaussian distribution. To verify that $\widecheck \btheta \subSSL$ is consistent for $\bthetabar$, it suffices to show that (i) $ \sup_{\btheta\in\mathbf{\Theta}} \| \bUhat_{N}(\btheta) - \bU^0(\btheta) \|_2 = o_p(1)$ and (ii) {$ \inf_{\|\btheta-\bthetabar\|_2\ge \epsilon|}\|\bU(\btheta)\|_2 >0,$ for any $\epsilon>0$} \citep[Lemma 2.8]{newey1994large}, where
\begin{equation*}
\bU^0(\btheta) = {\rm E} [ \bx_i \{ g(\bgammabar \trans  \bPhi_i) - g(\btheta \trans \bx_i) \} ] .
\end{equation*}
For (i), we first note that since $\bgammatilde \overset{p} \to \bgammabar$, $\dot g(\cdot)$ is continuous and $\bPhi$ is bounded,
\begin{equation}
\begin{split}
\bUhat_{N}(\btheta)&= \Ninv\sum_{i = 1}^{N} \bx_i \{ g(\bgammatilde \trans \bPhi_i) - g(\btheta \trans \bx_i)  \}=\Ninv\sum_{i = 1}^{N} \bx_i \{ g(\bgammabar \trans \bPhi_i) - g(\btheta \trans \bx_i)\}+o_p(1).
\end{split}
\label{equ:app-c1}
\end{equation}
Note that $g(\bgammabar \trans \bPhi) - g(\btheta \trans \bx)$ is bounded and continuously differentiable in $\btheta$. We then apply the ULLN to have that $ \sup_{\btheta\in\mathbf{\Theta}} \| \bUhat_{N}(\btheta) - \bU^0(\btheta)\|_2 = o_p(1).$ For (ii), we note
\begin{equation*}
\bU^0(\btheta) = -{\rm E}  [ \bx_i \{ y_i - g(\bgammabar \trans  \bPhi_i) \} ]+ {\rm E}  [ \bx_i \{ y_i - g(\btheta \trans  \bx_i) \} ]={\rm E}  [ \bx_i \{ y_i - g(\btheta \trans  \bx_i) \} ].
\end{equation*}
Therefore, (ii) holds by Lemma \ref{lem:1} and $\widecheck \btheta \subSSL$ is consistent for $\bthetabar$. 

Now we consider the weak convergence of $\nhalf( \widecheck \btheta \subSSL - \bthetabar)$. Under Conditions \ref{cond:1}-\ref{cond:2}, we have the Taylor expansion
  \begin{align*}
 \nhalf( \widecheck \btheta \subSSL - \bthetabar) = \nhalf \bA^{-1} \left[ N^{-1} \sum_{i = 1}^{N} \bx_i \{ g(\bgammabar \trans \bPhi_i) - g(\bthetabar \trans \bx_i)  \} +    \bB( \bgammatilde- \bgammabar)   \right] +o_p(1),
  \end{align*}
where $\bB =  {\rm E} \{  \bx_i \bPhi_i\trans   \dot{g}(\bgammabar \trans \bPhi_i) \}$. This expansion coupled with the fact that
$N^{-1} \sum_{i = 1}^{N} \bx_i \{ g(\bgammabar \trans \bPhi_i) - g(\bthetabar \trans \bx_i)  \}  = O_p(N^{-\frac{1}{2}})=o_p(n^{-\frac{1}{2}})$ imply that
\begin{align*}
\nhalf( \widecheck \btheta \subSSL - \bthetabar) &=  \nhalf \sum_{s= 1}^S \rho_s  \left[ \ninv_s \sum_{i = 1}^n \bA^{-1}  \bB \bC^{-1} I(\Ssc_i = s) \bPhi_i \{y_i - g(\bgammabar \trans \bPhi_i) \} \right]  + o_p(1).
\end{align*}
Letting $\bx_i=(x_{i1},...,x_{ip})\trans$, we note that for $j = 1, \dots, p$,
\begin{align*}
[\bB  \Cbb^{-1} ]_j = {\rm E} \{  x_{ij}\bPhi_i\trans   \dot{g}(\bgammabar \trans \bPhi_i) \}  {\rm E} [\{ \bPhi_i^{\otimes 2} \dot{g}(\bgammabar \trans \bPhi_i) \}]^{-1}
=   \argmin{\bbeta} {\rm E} \{  \dot{g}(\bgammabar \trans \bPhi_i) (x_{ij} - \bbeta \trans\bPhi_i)^2 \}.
\end{align*}
Since $x_{ij}$ is {a component of the vector $\bPhi(\bu)$, the minimizer $\bbeta$ can be chosen such that $ x_{ij} - \bbeta \trans\bPhi_i = 0$ for $i=1, \cdots, N$ which implies that $\bx_i =   \bB  \Cbb^{-1} \bPhi_i$.} Thus 
\begin{align*}
\nhalf( \widecheck \btheta \subSSL - \bthetabar) &=   \nhalf \sum_{s= 1}^S \rho_s   \left\{ \ninv_s\sum_{i = 1}^n  I(\Ssc_i = s) \be_{\SSL i}  \right\}+ o_p(1)
\end{align*}
where $\be_{\SSL i} = \bA^{-1}\bx_i \{y_i - g(\bgammabar \trans \bPhi_i) \}$.
It then follows from the {\red{classical Central Limit Theorem}} that $\nhalf( \widecheck \btheta \subSSL  - \bthetabar) \to N(0, \bSigma\subSSL)$ in distribution where
$$\bSigma\subSSL =  \sum_{s=1}^S \rho_s^2 \rho_{1s}^{-1}E\left\{\be_{\SSL i}^{\otimes 2} \mid \Ssc_i = s\right\}.$$ 
We then see that
\begin{align*}
& \bSigma\subSL  - \bSigma\subSSL  =\sum_{s=1}^S \rho_s^2 \rho_{1s}^{-1}  [{\rm E} \{\be_{\SL i}^{\otimes 2} \mid \Ssc_i = s \} - {\rm E} \{\be_{\SSL i}^{\otimes 2} \mid \Ssc_i = s \} ]=\\
&\sum_{s=1}^S\rho_s^2 \rho_{1s}^{-1}\bA^{-1} {\rm E} \left[ \bx_i^{\otimes 2} \{g(\bgammabar \trans \bPhi_i) -  g(\bthetabar \trans \bx_i)\}^2   
+ 2      \bx_i^{\otimes 2} \{y_i - g(\bgammabar \trans \bPhi_i)\}  \{g(\bgammabar \trans \bPhi_i) -  g(\bthetabar \trans \bx_i)\}\mid\Ssc_i  = s \right] .
\end{align*}
Therefore, when the imputation model is correctly specified, it follows that 
$$\bSigma\subSL  - \bSigma\subSSL =  \sum_{s=1}^S \rho_s^2 \rho_{1s}^{-1} E \left[ \bA^{-1} \bx_i^{\otimes 2} \{g(\bgammabar \trans \bPhi_i) -  g(\bthetabar \trans \bx_i)\}^2  \mid \Ssc_i  = s   \right] \succeq \bzero .$$

\section{Asymptotic Properties of $\Dhat \subSSL( \bthetahat \subSSL)$ and $\Dhat \subSSL( \widecheck\btheta \subSSL)$}\label{sec: asym-DSSL}
First note that by Lemma \ref{lem:1}, there exists $\delta>0$ such that for all $\btheta$ satisfying $\|\btheta-\bthetabar\|_2<\delta$, so that $\bnubar_{\btheta}$ is unique. Then, similar to the derivations in Appendices \ref{sec: asym-thetaSL} and \ref{sec: asym-thetaSSL}, we may show that $\bnutilde_{\btheta}$ is consistent for $\bnubar_{\btheta}$ and $\nhalf (\bnutilde_{\btheta} - \bnubar_{\btheta})$ is asymptotically Gaussian with mean zero.

Let $\mathbf{\Theta}_{\delta}=\mathbf{\Theta}\cap\{\btheta:\|\btheta-\bthetabar\|_2<\delta\}$. For the consistency of $\Dhat \subSSL ( \bthetahat \subSSL)$ for $D(\bthetabar)$, we note that the uniform consistency of $\bnutilde_{\btheta}$ for $\bnubar_{\btheta}$ and $\bgammatilde$ for $\bgammabar$, together with the ULLN and under regularity Conditions \ref{cond:1}-\ref{cond:2} {\red and $ \rhohat_s\overset{p}\rightarrow\rho_s$}, imply $\sup_{\btheta\in\mathbf{\Theta}_{\delta}} | \Dhat \subSSL(\btheta) - D( \btheta) | = o_p(1)$.  It then follows from the consistency of $\bthetahat \subSSL$ and $\widecheck\btheta \subSSL$ for $\bthetabar$ that $\Dhat \subSSL( \bthetahat \subSSL) \overset{p} \to D( \bthetabar)$ and  $\Dhat \subSSL( \widecheck\btheta \subSSL) \overset{p} \to D( \bthetabar)$ as $n \to \infty$.

To derive the asymptotic distribution for $\Tschat\subSSL = \nhalf \{ \Dhat\subSSL( \bthetahat \subSSL)  - D(\bthetabar)\}$ and $\widecheck\Tsc\subSSL =  \nhalf\{\Dhat\subSSL( \widecheck\btheta \subSSL) - D( \bthetabar)\}$, we first consider
\begin{align*}
\Tsctilde\subSSL(\btheta) = \nhalf \{ \Dhat \subSSL( \btheta)  -  D(\btheta) \} &= \nhalf  \left [ N^{-1} \sum_{i = 1}^{N} d\{g(\bgammatilde \trans \bPhi_i + \bnutilde_{\btheta}), \yscrl(\btheta \trans \bx_i)\} - D(\btheta) \right]
\end{align*}
Under Conditions \ref{cond:1}-\ref{cond:2} and by Taylor series expansion about $\bgammabar$ and $\bnubar$ and the ULLN, 
\begin{align*}
\Tsctilde\subSSL(\btheta) = & \nhalf  \left[  N^{-1} \sum_{i = 1}^{N} d\{g(\bgammabar \trans \bPhi_i + \bnubar_{\btheta} \trans \bz_{i\btheta}), \yscrl(\btheta \trans \bx_i)\} - D(\btheta)  +  \bG_{\btheta} (\bgammatilde - \bgammabar)  + \bH_{\btheta}  (\bnutilde_{\btheta} - \bnubar_{\btheta})  \right]
\end{align*}
where $ \bG_{\btheta} = {\rm E} [\bPhi_i \trans \dot{g}(\bgammabar \trans \bPhi_i + \bnubar_{\btheta} \trans \bz_{i\btheta}) \{ 1- 2\yscrl(\btheta \trans \bx_i)\} ]$ and $\bH_{\btheta} = {\rm E}[\bz_{i\btheta} \trans\dot{g}(\bgammabar \trans \bPhi_i + \bnubar_{\btheta}\trans \bz_{i\btheta}) \{ 1- 2\yscrl(\btheta \trans \bx_i)\}]$. From the previous section we have 
\begin{align*}
\nhalf( \bgammatilde- \bgammabar) &=   \nhalf  \sum_{s= 1}^S \rho_s \left[ \ninv_s \sum_{i = 1}^n \bC^{-1} I(\Ssc_i = s) \bPhi_i \{y_i - g(\bgammabar \trans \bPhi_i) \} \right] + o_p(1) 
\end{align*}
Similar arguments can be used to verify that
\begin{align*}
 \nhalf (\bnutilde_{\btheta} - \bnubar_{\btheta}) \trans  &= \nhalf \bJ_{\btheta}^{-1}   \left[\sum_{s= 1}^S \rho_s   \ninv_s \sum_{i = 1}^n  I(\Ssc_i = s) \bz_{i \btheta } \{y_i - g(\bgammabar \trans \bPhi_i  + \bnubar\trans\bz_{i \btheta })  \}  +    \bK(\bgammatilde - \bgammabar) \trans \right] + o_p(1) 
\end{align*}
where $\bJ_{\btheta} = {\rm E}\{   \bz_{i \btheta } ^{\otimes 2} \dot{g}(\bgammabar \trans \bPhi_i + \bnubar \trans\bz_{i \btheta }) \}$ and $\bK_{\btheta} =- {\rm E}\{ \bz_{i \btheta } \bPhi_i \trans  \dot{g}(\bgammabar \trans \bPhi_i + \bnubar \trans\bz_{i \btheta })  \}$.
These results, together with the fact that 
$N^{-\frac{1}{2}} \sum_{i = 1}^{N} d\{g(\bgammabar \trans \bPhi_i + \bnubar_{\btheta} \trans \bz_{i\btheta}), \yscrl(\btheta \trans \bx_i)\} - D(\btheta) )$
converges weakly to zero-mean Gaussian process in $\btheta$, imply that 
\begin{align*}
\Tsctilde\subSSL(\btheta) =& \nhalf \sum_{s= 1}^S \rho_s\Bigg( \ninv_s \sum_{i = 1}^n  I(\Ssc_i = s)  \Big[ (\bG_{\btheta}   + \bH_{\btheta}  \bJ_{\btheta}^{-1} \bK_{\btheta} )  \bC^{-1}  \bPhi_i \{y_i - g(\bgammabar \trans \bPhi_i) \}      \\
&\hspace{2.2in} +  \bH_{\btheta}  \bJ_{\btheta}^{-1}  \bz_{i \btheta }  \{y_i - g(\bgammabar \trans \bPhi_i  + \bnubar_{\btheta} \trans\bz_{i \btheta }) \}\Big] \Bigg) + o_p(1) .
\end{align*}
We may simplify the above expression by noting that $\{1- 2\yscrl(\btheta \trans \bx_i)\}$ is a linear combination of $\bz_{i \btheta }$ and {hence $[\bH_{\btheta}  \bJ_{\btheta}^{-1}]\bz_{i \btheta } = \{1- 2\yscrl(\btheta \trans \bx_i)\}$}.  Additionally, $ \bH_{\btheta}  \bJ_{\btheta}^{-1} \bK_{\btheta}  = - \bG_{\btheta}$ which implies that $(\bG_{\btheta}    + \bH_{\btheta}  \bJ_{\btheta}^{-1} \bK_{\btheta})  \bC^{-1}= 0$. Thus,
\begin{align*}
\Tsctilde\subSSL(\btheta)=&  \nhalf \sum_{s= 1}^S \rho_s  \left[ \ninv_s \sum_{i = 1}^n  I(\Ssc_i = s) \{1- 2\yscrl(\btheta \trans \bx_i) \} \{y_i - g(\bgammabar \trans \bPhi_i  + \bnubar\trans\bz_{i \btheta }) \}\right].
\end{align*}
This combined with the fact that $D(\btheta)$ is continuously differentiable at $\bthetabar$, $\bWhat$ is consistent for its limiting value $\bW$ introduced in Appendix \ref{sec: inf-thetaSSL}, and Conditions \ref{cond:1}-\ref{cond:2} then give that
\begin{align*}
\Tschat\subSSL =&\nhalf \{ \Dhat \subSSL( \bthetahat \subSSL) - D(\bthetahat\subSSL)\} +\nhalf \{ D(\bthetahat \subSSL)  - D(\bthetabar)\} \\
=& \nhalf  \sum_{s= 1}^S \rho_s \Bigg(\ninv_s \sum_{i = 1}^n  I(\Ssc_i = s) \Big[ \{1- 2\Yscbar_i \} \{y_i - g(\bgammabar \trans \bPhi_i  + \bnubar_{\bthetabar} \trans\bz_{i \bthetabar }) \}  \\
& \hspace{2.5in} + \dot \bD(\bthetabar)\trans\{\bW \be_{\SSL i} +  (\bI - \bW)  \be_{\SL i} \} \Big]\Bigg) +o_p(1) \\
= & \nhalf \sum_{s= 1}^S \rho_s \Bigg(\ninv_s \sum_{i = 1}^n I(\Ssc_i = s) \Big[ \{d(y_i, \Yscbar_i)- d(m_{\sII,i} , \Yscbar_i)\} \\
& \hspace{2.5in} + \dot \bD(\bthetabar)\trans\{\bW \be_{\SSL i} +  (\bI - \bW)  \be_{\SL i} \} \Big]\Bigg)+o_p(1).
\end{align*}
{Note that the existence of $\dot\bD(\bthetabar)$ is implied by Condition \ref{cond:1}, namely, that the density function of $\bthetabar\trans\bx$ is continuously differentiable in $\bthetabar\trans\bx$ and ${\rm P}(y=1|\bu)$ is continuously differentiable in the continuous components of $\bu$.} We then have that $\nhalf \{ \Dhat \subSSL( \bthetahat \subSSL)  - D(\bthetabar)\}$ converges to a zero mean normal random variable by the {\red{classical Central Limit Theorem}}. Using similar arguments as those for $\widecheck\Tsc\subSL$, we have that $\widecheck\Tsc\subSSL=\nhalf \{ \Dhat \subSSL( \widecheck\btheta \subSSL)  - D(\bthetabar)\}$ can be expanded as
\[
\nhalf \sum_{s= 1}^S \rho_s \left( \ninv_s \sum_{i = 1}^n  I(\Ssc_i = s) \Big[ \{d(y_i, \Yscbar_i)  - 
d(m_{\sII,i} , \Yscbar_i) \}+ 
\dot \bD(\bthetabar)\trans\be_{\SSL i} \Big] \right),
\]
which also converges to a zero mean normal random variable. 

Comparing the asymptotic variance of $\widecheck\Tsc\subSSL$ with $\Tschat\subSL$, first note that
\begin{align*}
\Tschat\subSL=&\nhalf \sum_{s= 1}^S \rho_s  \left( \ninv_s \sum_{i = 1}^n  I(\Ssc_i = s) \left[d(y_i, \Yscbar_i)- D(\bthetabar) + \dot{D}(\bthetabar)\trans \be_{\SL i}  \right] \right) + o_p(1)\\
=&\nhalf \sum_{s= 1}^S \rho_s  \Bigg( \ninv_s \sum_{i = 1}^n  I(\Ssc_i = s) \Big[d(y_i, \Yscbar_i)- d(m_{\sII,i} , \Yscbar_i) + d(m_{\sII,i} , \Yscbar_i) - D(\bthetabar)\\
& \hspace{2.5in} +\dot{D}(\bthetabar)\trans \bA^{-1}\bx_i \{y_i - g(\bthetabar \trans \bx_i)\} \Big] \Bigg) + o_p(1)\\
=&\nhalf \sum_{s= 1}^S \rho_s  \Bigg( \ninv_s \sum_{i = 1}^n  I(\Ssc_i = s) \Big[(1- 2\Yscbar_i )\{y_i - g(\bgammabar \trans \bPhi_i  + \bnubar_{\bthetabar} \trans\bz_{i \bthetabar }) \} + d(m_{\sII,i} , \Yscbar_i) - D(\bthetabar)\\
& \hspace{1.4in} +\dot{D}(\bthetabar)\trans \bA^{-1}\bx_i \{y_i -g(\bgammabar \trans \bPhi_i) + g(\bgammabar \trans \bPhi_i)- g(\bthetabar \trans \bx_i)\} \Big] \Bigg) + o_p(1).
\end{align*}
Letting $h_1(\bPhi_i)=1- 2\Yscbar_i+\dot{D}(\bthetabar)\trans \bA^{-1}\bx_i$ and 
\[
h_2(\bPhi_i)=d(m_{\sII,i} , \Yscbar_i) - D(\bthetabar)+\dot{D}(\bthetabar)\trans \bA^{-1}\bx_i\{g(\bgammabar \trans \bPhi_i)- g(\bthetabar \trans \bx_i)\}.
\] 
we note that $h_1(\bPhi_i)$ and $h_2(\bPhi_i)$ are functions of $\bPhi_i$ and do not depend on $y_i$. Thus, when $P(y = 1 \mid \bu) = g(\bgammabar \trans \bPhi)$, we have $\bnubar=\bzero$ and 
\begin{equation*}
\begin{split}
\sigma^2\subSL =&  \sum_{s= 1}^S \rho_s^2\rho_{1s}^{-1}  {\rm E}[ h_1^2(\bPhi_i)\{y_i - g(\bgammabar \trans \bPhi_i)\}^2+ 2 h_1(\bPhi_i) h_2(\bPhi_i)\{y_i - g(\bgammabar \trans \bPhi_i)\}+ h_2^2(\bPhi_i)\mid \Ssc_i = s ]\\
=&\sum_{s= 1}^S \rho_s^2\rho_{1s}^{-1}  {\rm E}[ h_1(\bPhi_i)^2\{y_i - g(\bgammabar \trans \bPhi_i)\}^2+h_2(\bPhi_i)^2 \mid \Ssc_i = s ]
\end{split}
\end{equation*}
while the asymptotic variance of $\widecheck\Tsc\subSSL$ is
\begin{equation*}
\begin{split}
\sigma^2\subSSL =\sum_{s= 1}^S \rho_s^2\rho_{1s}^{-1}  {\rm E}[ h_1^2(\bPhi_i)\{y_i - g(\bgammabar \trans \bPhi_i)\}^2\mid \Ssc_i = s ].
\end{split}
\end{equation*}
Therefore, when $P(y = 1 \mid \bu) = g(\bgammabar \trans \bPhi)$, it follows that $\Delta_{aVar}:=\sigma^2\subSL-\sigma^2\subSSL>0$. Additionally, when model (\ref{model}) is correct and $P(y = 1 \mid \bu)= g(\bgammabar \trans \bPhi) = g(\bthetabar \trans \bx)$, we have $h_2(\bPhi_i)=d(m_{\sII,i} , \Yscbar_i) - D(\bthetabar)$ which is not equal to $0$ with probability 1, so that again $\Delta_{aVar}>0$.

\section{Intrinsic Efficient Estimation}\label{sec:app:intri:all}

\subsection{Intrinsic Efficient Estimator for $\Dbar$}\label{sec:app:intri:cons}
{\red 
We first introduce the intrinsic efficient estimator of the accuracy measures. Without loss of generality, we set the imputation basis for both $\btheta$ and $D(\btheta)$ as $\bPsi_{\btheta i}=[\bPhi_i\trans,\yscrl(\btheta \trans \bx_i)]\trans$, where $\btheta$ is plugged in with some preliminary estimator for $\btheta$, denoted as $\widetilde\btheta$. In practice, one may take $\widetilde\btheta$ as either the simple SL estimator or the SSL estimator obtained following Section \ref{sec:flex:imp}. We include $\yscrl(\btheta\trans \bx_i)$ in the imputation basis to simplify the notation and presentation in this section. Although this distinguishes the following discussion from the proposal in Section \ref{sec:aug}, it is straightforward to extend our results to the original proposal.

Recall that for the original SSL estimator of the regression parameter, one first obtains $\widetilde\bgamma_{\bthetatilde}$ as the solution to
\[
N^{-1}\sum_{i = 1}^N \what_i\bPsi_{\bthetatilde i} \{y_i - g(\bgamma \trans \bPsi_{\bthetatilde i}) \}  - \lambda_n \bgamma = \bzero
\]
and then solves $N^{-1}\sum_{i = 1}^{N} \bx_i \{ g(\bgammatilde_{\bthetatilde} \trans \bPsi_{\bthetatilde i}) - g(\btheta \trans \bx_i)  \} =\bzero$ to obtain the estimator of $\bthetabar$.  Despite the change in basis, we still denote this estimator as $\widehat\btheta\subplain$ with a slight abuse in the notation. Adapting the augmentation procedure in Section \ref{sec:aug}, we then find $\bgammatilde_{\widehat\btheta\subplain}$ as the solution to 
\[
N^{-1}\sum_{i = 1}^N \what_i\bPsi_{\widehat\btheta\subplain i} \{y_i - g(\bgamma \trans \bPsi_{\widehat\btheta\subplain i}) \}  - \lambda_n \bgamma = \bzero,
\]
and estimate $\Dbar$ with $\Dhat\subplain=\Dhat\subplain(\bthetahat\subplain)$ where $\Dhat\subplain( \btheta) =N^{-1} \sum_{i =1}^{N} d\{g(\bgammatilde_{\btheta} \trans \bPsi_{\btheta i}),  \yscrl(\btheta\trans\bx_i) \}.$ Extending Theorem \ref{thm:2}, the asymptotic variance of $\nhalf\{\Dhat\subplain(\bthetahat\subplain)- \Dbar\} $ may be expressed as
\begin{equation}
\frac{1}{n}\sum_{i=1}^n{\rm E}\zeta_i\{1- 2\Yscbar_i +\dot\bD(\bthetabar)\trans\bA^{-1}\bx_i\}^2\{y_i - g(\bgammabar_{\bthetabar} \trans \bar\bPsi_i) \}^2,
\label{equ:var:D}
\end{equation}
where $\bgammabar_{\bthetabar}$ represents the limits of $\bgammatilde_{\widehat\btheta\subplain}$ (or $\bgammatilde_{\bthetatilde}$), and $\bar\bPsi_i=\bar\bPsi_{\bthetabar i}=[\bPhi_i\trans,\yscrl(\bthetabar\trans \bx_i)]\trans$. Analogous to the construction of $\e\trans\bthetahat\subintri$, we consider minimizing the asymptotic variance given by (\ref{equ:var:D}) to estimate $\Dbar$. Specifically, we first solve for $\bgammatilde\suptwo_{\bthetatilde}$ with
\begin{equation}
\begin{split}
\argmin{\bgamma}\frac{1}{2n}&\sum_{i=1}^n\widehat\zeta_i\{1- 2\yscrl(\bthetatilde \trans \bx_i) +\widehat{\dot\bD}\trans\widehat\bA^{-1}\bx_i\}^2\{y_i - g(\bgamma\trans \bPsi_{\bthetatilde i}) \}^2+\lambda_n\suptwo\|\bgamma\|_2^2,\\
\mbox{ s.t. }\frac{1}{N}&\sum_{i=1}^N\what_i[\bx_i\trans,\yscrl(\bthetatilde \trans \bx_i)]\trans\{y_i-g(  \bgamma\trans\bPsi_{\bthetatilde i})\}=\bzero,
\end{split}
\label{equ:min:D}
\end{equation}
where $\widehat{\dot\bD}$ is an estimation of $\dot\bD(\bthetabar)$ and the tuning parameter $\lambda\suptwo=o(n^{-\frac{1}{2}})$. Similar to (\ref{impee}) and (\ref{equ:min:theta}), moment constraints in (\ref{equ:min:D}) calibrate potential bias of the estimators for $\bthetabar$ and $\Dbar(\btheta)$. 

Next, we present the construction of $\widehat{\dot\bD}$ for the Brier score and OMR separately. For the Brier score, $\bar D_1$, we take
\[
\widehat{\dot\bD}=\widehat{\dot\bD}_1=\frac{1}{N}\sum_{i = 1}^N -2\what_i\dot g(\bthetatilde\trans\bx_i)\{y_i-g(\bthetatilde\trans\bx_i)\}\bx_i.
\]
For the the OMR, $\bar D_2$, recall that a simple estimator is given by the empirical average
\[
\frac{1}{N}\sum_{i = 1}^N\what_i\{y_i+(1-2y_i)I(g(\bthetatilde\trans\bx_i)>c)\}.
\]
Since $I(g(\btheta\trans\bx_i)>c)$ is not a differentiable function of $\btheta$, we first smooth each $I(g(\bthetatilde\trans\bx_i)>c)$ as $\int_c^{+\infty}K_h\{g(\bthetatilde\trans\bx_i)-u\}du,$ where $K(\cdot)$ represents the Gaussian kernel function, and $K_h(a):=h^{-1}K(a/h)$ with some bandwidth $h>0$. Then, $\dot\bD(\bthetabar)$ is estimated with
\[
\widehat{\dot\bD}=\widehat{\dot\bD}_2=\frac{1}{N}\sum_{i = 1}^N\what_i(1-2y_i)\dot g(\bthetatilde\trans\bx_i)K_h\{g(\bthetatilde\trans\bx_i)-c\}\bx_i.
\]
With $\bgammatilde\suptwo_{\bthetatilde}$, we then solve 
\[
\Ninv\sum_{i = 1}^{N} \bx_i \{ g(  \bgammatilde\suptwotrans_{\bthetatilde}\bPsi_{\bthetatilde i}) - g(\btheta \trans \bx_i)  \} =\bzero
\] 
to obtain $\widehat\btheta\subintri^D$ for estimation of $\Dbar$  and employ the augmentation procedure in Section \ref{sec:aug}.  That is, we solve $\bgammatilde\suptwo_{\widehat\btheta^D\subintri}$ from
\[
N^{-1}\sum_{i = 1}^N \what_i\bPsi_{\widehat\btheta\subintri^D i} \{y_i - g(\bgamma \trans \bPsi_{\widehat\btheta\subintri^D i}) \}  - \lambda_n\suptwo \bgamma = \bzero,
\]
and estimate $\Dbar$ by $\widehat D\subintri=\widehat D\subintri(\widehat\btheta\subintri^D)$ where $\widehat D\subintri(\btheta)=\Ninv\sum_{i=1}^{N} d\{g(\bgammatilde_{\btheta}\suptwo\bPsi_{\btheta i}), \yscrl(\btheta \trans \bx_i) \}.$ 

To present the asymptotic properties of $\widehat D\subintri$, we define
\begin{align*}
\bar\bgamma\suptwo_{\bthetabar}=\argmin{\bgamma} &{\rm E}[R\{1- 2\bar\Ysc+\dot\bD(\bthetabar)\trans\bA^{-1}\bx\}^2\{y - g(\bgamma_{\bthetabar}\trans \bar\bPsi) \}^2],\\
\mbox{ s.t. }&{\rm E}[\bx\trans,\yscrl(\bthetabar \trans \bx)]\trans\{y-g(  \bgamma_{\bthetabar}\trans\bar\bPsi)\}=\bzero.
\end{align*}
Theorem \ref{thm:a1} provides the asymptotic expansion of $\Dhat\subintri$ and its proof, together with the proof of Theorem \ref{thm:3} from the main text, is detailed in Appendix \ref{sec:app:intri}.
\begin{theorem}
Under Conditions \ref{cond:1}, Conditions \ref{cond:a1} and \ref{cond:a2} from Appendix \ref{sec:app:intri}, and with the bandwidth $h\asymp n^{-\frac{1}{4}}$,  $\nhalf(\Dhat\subintri -\Dbar)$ weakly converges to a Gaussian distribution with mean zero, and is asymptotically equivalent to $\widehat\Tsc(\bgammabar\suptwo_{\bthetabar})$ where
\[
\widehat\Tsc(\bgamma)= \nhalf \sum_{s= 1}^S \rho_s \left[ \ninv_s \sum_{i = 1}^N V_i I(\Ssc_i = s) \{1- 2\bar\Ysc_i+\dot\bD(\bthetabar)\trans\bA^{-1}\bx_i\}\{y_i - g(\bgamma\trans \bar\bPsi_i) \} \right].
\]
This implies that (i) $\Dhat\subintri$ is asymptotically equivalent to $\Dhat\subplain$ when the imputation model ${\rm P}(y = 1 \mid \bu) = g(\bgamma \trans \bar\bPsi)$ is correctly specified and (ii) the asymptotic variance of $\nhalf(\Dhat\subintri -\Dbar)$ is minimized among $\{\widehat\Tsc(\bgamma):{\rm E}[\bx\trans,\yscrl(\bthetabar \trans \bx)]\trans\{y-g(  \bgamma\trans\bar\bPsi)\}=\bzero\}$. Consequently, the asymptotic variance of the intrinsic efficient estimator is always less than or equal to the asymptotic variance of $\nhalf(\Dhat\subplain -\Dbar)$ and $\nhalf(\Dhat\subSL -\Dbar)$.

\label{thm:a1}
\end{theorem}

}


{
\red
\subsection{Asymptotic Properties of $\widehat\btheta\subintri$ and $\widehat D\subintri$}\label{sec:app:intri}
We first introduce the smoothness condition on the link function $g(\cdot)$, which is stronger than Condition \ref{cond:3}, but still holds for the most commonly used link functions such as the logit and probit functions.
\begin{cond}
The link function $g(\cdot)\in(0,1)$ is continuously twice differentiable with derivative $\dot g(\cdot)$ and the second order derivative $\ddot g(\cdot)$.
\label{cond:a1}
\end{cond}
Given Condition \ref{cond:a1}, we let $\bgammabar\suptwo=\bgammabar\suptwo_{\bthetabar}$ and define 
\begin{align*}
\bA_1&={\rm E}\left[R(\e\trans\bA^{-1}\bx)^2\bPhi^{\otimes 2}\{\dot{g}^2(\bgammabar\suponetrans \bPhi)+\ddot{g}(\bgammabar\suponetrans\bPhi)[y-g(\bgammabar\suponetrans\bPhi)]\}\right],\\
\bA_2&={\rm E}\left[R\{1- 2\bar\Ysc+\dot\bD(\bthetabar)\trans\bA^{-1}\bx\}^2\bar\bPsi^{\otimes 2}\{\dot{g}^2(\bgammabar\suptwotrans \bar\bPsi)+\ddot{g}(\bgammabar\suptwotrans \bar\bPsi)[y-g(\bgammabar\suptwotrans \bar\bPsi)]\}\right],
\end{align*}
$\bB_1={\rm E}[\bPhi\bx\trans \dot g(\bgammabar \suponetrans \bPhi)]$ and $\bB_2={\rm E}[\bar\bPsi\{\bx\trans,\yscrl(\bthetabar \trans \bx)\}\dot g(\bgammabar \suptwotrans \bar\bPsi)]$. We next present the regularity condition on the covariates and regression coefficients required by Theorem \ref{thm:3}.
\begin{cond}
There exists $\bTheta'=\{\btheta:\|\btheta-\bthetabar\|_2<\delta'\}$ for some $\delta'>0$,  such that for any $\btheta\in \bTheta'$, there is no $\bgamma$ such that  $P (\bgamma \trans \bPhi_1> \bgamma \trans \bPhi_2\mid y_1>y_2)=1$ or $P (\bgamma \trans \bPsi_{\btheta 1} > \bgamma \trans \bPsi_{\btheta 2}\mid y_1>y_2)=1$. It is also the case that $\bA\succ\bzero$, $\bA_1\succ\bzero$, $\bA_2\succ\bzero$, $\bB_1\trans\bA_1^{-1}\bB_1\succ\bzero$, and $\bB_2\trans\bA_2^{-1}\bB_2\succ\bzero$.
\label{cond:a2}
\end{cond}

\begin{remark}
Condition \ref{cond:a2} is analog to Condition \ref{cond:2}. It assumes there is no linear combination of $\bPhi$ or $\bar\bPsi$ perfectly separating the samples based on $y$, and the Hessian matrices of the constrained least square problems for $\bgammabar\supone$ and $\bgammabar\suptwo$ are positive definite. Again, these assumptions are mild and common in the M-estimation literature \citep{van2000asymptotic}.

\end{remark}

Under these regularity conditions, we present the proofs of Theorem \ref{thm:3} and \ref{thm:a1}. In our development, we take $\bthetatilde$ as the SSL estimator for $\btheta$ introduced in Section \ref{sec:flex:imp}, but note that the proof remains basically unchanged when taking $\bthetatilde$ as the SL estimator. We first derive the consistency (error rates) of $\widehat{\dot\bD}_1$ and $\widehat{\dot\bD}_2$. For the Brier score, let $\bar{\dot\bD}_1$ denote the derivative of $D_1(\btheta)$ evaluated at $\bthetabar$. We then use $\rhohat_s\overset{p}\rightarrow\rho_s$, Theorem \ref{thm:1}, Conditions \ref{cond:1} and \ref{cond:a1}, and the classical Central Limit Theorem to derive that 
\begin{align*}
\widehat{\dot\bD}_1-\bar{\dot\bD}_1=&\frac{1}{N}\sum_{i = 1}^N -2\what_i\dot g(\bthetatilde\trans\bx_i)\{y_i-g(\bthetatilde\trans\bx_i)\}\bx_i-\frac{1}{N}\sum_{i = 1}^N -2w_i\dot g(\bthetabar\trans\bx_i)\{y_i-g(\bthetabar\trans\bx_i)\}\bx_i\\
&+\frac{1}{N}\sum_{i = 1}^N -2w_i\dot g(\bthetabar\trans\bx_i)\{y_i-g(\bthetabar\trans\bx_i)\}\bx_i-{\rm E}[2\dot g(\bthetabar\trans\bx)\{y-g(\bthetabar\trans\bx)\}\bx]\\
=&O_p(\|\bthetatilde-\bthetabar\|_2)+O_p(n^{-\frac{1}{2}})=O_p(n^{-\frac{1}{2}}).
\end{align*}
For the estimator of the derivative for the OMR, let $\bar{\dot\bD}_2$ be the limiting value of $\widehat{\dot\bD}_2$. We then have
\begin{align*}
\widehat{\dot\bD}_2-\bar{\dot\bD}_2=&\frac{1}{N}\sum_{i = 1}^N(1-2y_i)[\what_i\dot g(\bthetatilde\trans\bx_i)K_h\{g(\bthetatilde\trans\bx_i)-c\}-w_i\dot g(\bthetabar\trans\bx_i)K_h\{g(\bthetabar\trans\bx_i)-c\}]\bx_i\\    
&+\frac{1}{N}\sum_{i = 1}^Nw_i(1-2y_i)\dot g(\bthetabar\trans\bx_i)K_h\{g(\bthetabar\trans\bx_i)-c\}\bx_i-{\rm E}[(1-2y)\dot g(\bthetabar\trans\bx)\bx|g(\bthetabar\trans\bx)=c]f_g(c)\\
=:&\bDelta_1+\bDelta_2
\end{align*}
where $f_g(c)$ represent the density function of $g(\bthetabar\trans\bx)$ evaluated at $c$. This follows from the fact that $\bar{\dot\bD}_2={\rm E}[(1-2y)\dot g(\bthetabar\trans\bx)\bx|g(\bthetabar\trans\bx)=c]f_g(c)$. Since the Gaussian kernel $K(\cdot)$ is continuously differentiable and by Theorem \ref{thm:1}, Conditions \ref{cond:1} and \ref{cond:a1}, we have
\[
\|\bDelta_1\|_2=h^{-1}O_p(\|\bthetatilde-\bthetabar\|_2)=O_p(n^{-\frac{1}{2}}h^{-1}).
\]
For $\bDelta_2$, Condition \ref{cond:1} and the classical Central Limit Theorem imply that 
\[
\frac{1}{N}\sum_{i = 1}^Nw_i(1-2y_i)\dot g(\bthetabar\trans\bx_i)K_h\{g(\bthetabar\trans\bx_i)-c\}\bx_i-{\rm E}(1-2y)\dot g(\bthetabar\trans\bx)K_h\{g(\bthetabar\trans\bx)-c\}\bx=O_p\{(nh)^{-\frac{1}{2}}\},
\]
and from Condition \ref{cond:1},
\begin{align*}
&{\rm E}(1-2y)\dot g(\bthetabar\trans\bx)K_h\{g(\bthetabar\trans\bx)-c\}\bx-{\rm E}[(1-2y)\dot g(\bthetabar\trans\bx)\bx|g(\bthetabar\trans\bx)=c]f_g(c)\\
=&\int_{0}^1\{\r(u)f_g(u)K_h(u-c)-\r(c)f_g(c)\}du\\
=&\int_{-c/h}^{(1-c)/h}\{\r(c+hv)f_g(c+hv)K(v)-\r(c)f_g(c)\}dv=O(h),
\end{align*}
where $\r(u)=\int_{\{\bx:g(\bthetabar\trans\bx)=u\}}\{1-2{\rm P}(y=1|\bx)\}\dot g(\bthetabar\trans\bx)\bx f_{x|g}(\bx|u)d\bx$, and $f_{x|g}(\cdot|u)$ represent the density of $\bx$ given that $g(\bthetabar\trans\bx)=u$. By Condition \ref{cond:1}, there exists $C>0$ such that $\|\r(a)-\r(b)\|_2\leq C|a-b|$ for any $a,b\in\mathbb{R}$. Thus, we have $\|\bDelta_2\|_2=O_p\{(nh)^{-\frac{1}{2}}+h\}$ and with $h\asymp n^{-\frac{1}{4}}$, we obtain $\|\widehat{\dot\bD}_2-\bar{\dot\bD}_2\|_2=O_p(n^{-\frac{1}{4}})$. It then follows that for both Brier score and OMR,  $\|\widehat{\dot\bD}-\bar{\dot\bD}\|_2=O_p(n^{-\frac{1}{4}})=o_p(1)$.

Leveraging these results, we establish the asymptotic normality of $\bgammatilde\supone$ and $\bgammatilde\suptwo_{\bthetatilde}$. Similar to Appendices \ref{sec: asym-thetaSL} and \ref{sec: asym-thetaSSL}, we apply the ULLN \citep{pollard1990empirical}, together with Conditions \ref{cond:1}, \ref{cond:a1}, and \ref{cond:a2}, and the facts that $\rhohat_s\overset{p}\rightarrow\rho_s$, $ \rhohat_{1s}\overset{p}\rightarrow\rho_{1s}$, and that $\bthetatilde$, $\widehat\bA^{-1}$ and $\widehat{\dot\bD}$ are consistent for their respective limits, to obtain
\begin{align*}
\sup_{\bgamma\in\Gamma\supone}\left|\frac{1}{n}\sum_{i=1}^n\widehat\zeta_i(\e\trans\widehat\bA^{-1}\bx_i)^2\{y_i - g(\bgamma \trans \bPhi_i) \}^2-{\rm E}[R(\e\trans\bA^{-1}\bx)^2\{y - g(\bgamma \trans\bPhi) \}^2]\right|=&o_p(1);\\
\sup_{\bgamma\in\Gamma\supone}\left\|\frac{1}{N}\sum_{i=1}^N\what_i\bx_i\{y_i-g(  \bgamma\trans\bPhi_i)\}-{\rm E}[\bx\{y-g(  \bgamma\trans\bPhi)\}]\right\|_2=&o_p(1);\\
\sup_{\bgamma\in\Gamma\suptwo}\Bigg|\frac{1}{n}\sum_{i=1}^n\widehat\zeta_i\{1- 2\yscrl(\bthetatilde \trans \bx_i) +\widehat{\dot\bD}\trans\widehat\bA^{-1}\bx_i\}^2\{y_i - g(\bgamma\trans \bPsi_{\bthetatilde i}) \}^2\quad\quad\quad\quad\quad\quad\quad\quad\quad~&\\
-{\rm E}[R\{1- 2\bar\Ysc+\dot\bD(\bthetabar)\trans\bA^{-1}\bx\}^2\{y - g(\bgamma\trans \bar\bPsi) \}^2]\Bigg|=&o_p(1);\\
\sup_{\bgamma\in\Gamma\suptwo}\left\|\frac{1}{N}\sum_{i=1}^N\what_i[\bx_i\trans,\yscrl(\bthetatilde \trans \bx_i)]\trans\{y_i-g(  \bgamma\trans\bPsi_{\bthetatilde i})\}-{\rm E}[\bx\trans,\yscrl(\bthetabar \trans \bx)]\trans\{y-g(  \bgamma\trans\bar\bPsi)\}\right\|_2=&o_p(1),
\end{align*}
where $\Gamma\supone$ and $\Gamma\suptwo$ are two compact sets containing $\bgammabar\supone$ and $\bgammabar\suptwo$, respectively. This implies that $\|\bgammatilde\supone-\bgammabar\supone\|_2=o_p(1)$ and $\|\bgammatilde\suptwo_{\bthetatilde}-\bgammabar\suptwo\|_2=o_p(1)$. We then expand (\ref{equ:min:theta}) and (\ref{equ:min:D}) to derive that
\begin{align*}
\widetilde\bgamma\supone=\argmin\bgamma&(\bgamma-\bgammabar\supone)\trans\left[\bA_1(\bgamma-\bgammabar\supone)+2\{1+o_p(1)\}\bXi_{11}+o_p\left(\|\widetilde\bgamma\supone-\bgammabar\supone\|_2+n^{-\frac{1}{2}}\right)\right],\\
\mbox{ s.t. }&\bB_1\trans(\bgamma-\bgammabar\supone)-\{1+o_p(1)\}\bXi_{12}+o_p\left(\|\widetilde\bgamma\supone-\bgammabar\supone\|_2+n^{-\frac{1}{2}}\right)=\bzero;\\
\bgammatilde\suptwo_{\bthetatilde}=\argmin\bgamma&(\bgamma-\bgammabar\suptwo)\trans\left[\bA_2(\bgamma-\bgammabar\suptwo)+2\{1+o_p(1)\}\bXi_{21}+o_p\left(\|\widetilde\bgamma\suptwo-\bgammabar\suptwo\|_2+n^{-\frac{1}{2}}\right)\right],\\
\mbox{ s.t. }&\bB_2\trans(\bgamma-\bgammabar\suptwo)-\{1+o_p(1)\}\bXi_{22}+o_p\left(\|\widetilde\bgamma\suptwo-\bgammabar\suptwo\|_2+n^{-\frac{1}{2}}\right),\\
\mbox{where}\quad\bXi_{11}=\frac{1}{n}\sum_{i=1}^n\zeta_i&(\e\trans\bA^{-1}\bx_i)^2\dot g(\bgammabar\suponetrans\bPhi_i)\bPhi_i\{y_i - g(\bgammabar\suponetrans\bPhi_i) \};\\
\bXi_{12}=\frac{1}{N}\sum_{i=1}^Nw_i&\bx_i\{y_i-g(\bgammabar\suponetrans\bPhi_i)\};\\
\bXi_{21}=\frac{1}{n}\sum_{i=1}^n\zeta_i&\{1- 2\yscrl(\bthetabar \trans \bx_i) +{\dot\bD}\trans\bA^{-1}\bx_i\}^2\dot g(\bgammabar\suptwotrans \bar\bPsi_i)\bar\bPsi_i\{y_i - g(\bgammabar\suptwotrans \bar\bPsi_i) \}+\dot\bXi_{\btheta,21}\trans(\bthetatilde-\bthetabar);\\
\bXi_{22}=\frac{1}{N}\sum_{i=1}^Nw_i&[\bx_i\trans,\yscrl(\bthetabar \trans \bx_i)]\trans\{y_i-g(  \bgammabar\suptwotrans\bar\bPsi_i)\}+\dot\bXi_{\btheta,22}\trans(\bthetatilde-\bthetabar),
\end{align*}
and $\dot\bXi_{\btheta,21}$, $\dot\bXi_{\btheta,22}$ are two fixed loading matrices of the order $O(1)$. By Condition \ref{cond:1} and the classical Central Limit Theorem, $n^{\frac{1}{2}}(\bXi_{11}\trans,\bXi_{12}\trans,\bXi_{21}\trans,\bXi_{22}\trans)\trans$ converges to a Gaussian distribution with mean $\bzero$. By Theorem \ref{thm:1}, $\nhalf(\bthetatilde-\bthetabar)$ also converges to a mean-zero Gaussian distribution.  Analogous to the proof of Theorem 5.21 of \cite{van2000asymptotic}, we then obtain
\begin{equation}
\begin{split}
&\widetilde\bgamma\supone-\bgammabar\supone=[\bA_1^{-1}-\bA_1^{-1}\bB_1(\bB_1\trans\bA_1^{-1}\bB_1)^{-1}\bB_1\trans\bA_1^{-1}]\bXi_{11}+\bA_1^{-1}\bB_1(\bB_1\trans\bA_1^{-1}\bB_1)^{-1}\bXi_{12}=O_p(n^{-\frac{1}{2}});\\
&\widetilde\bgamma\suptwo_{\bthetatilde}-\bgammabar\suptwo=[\bA_2^{-1}-\bA_2^{-1}\bB_2(\bB_2\trans\bA_2^{-1}\bB_2)^{-1}\bB_2\trans\bA_2^{-1}]\bXi_{21}+\bA_2^{-1}\bB_2(\bB_2\trans\bA_2^{-1}\bB_2)^{-1}\bXi_{22}=O_p(n^{-\frac{1}{2}}).
\end{split}    
\label{equ:app:G:1}
\end{equation}

 By Conditions \ref{cond:1}, \ref{cond:a1} and \ref{cond:a2}, the consistency of $\widehat\rho_1$ for its limit, and the asymptotic expansion of $\widetilde\bgamma\supone-\bgammabar\supone$ derived above, we can use the argument of Appendix \ref{sec: asym-thetaSSL} to show that $\widehat\btheta\subintri\overset{p}{\rightarrow}\bthetabar$ and obtain the expansion
\begin{align*}
n^{\frac{1}{2}}(\widehat\btheta\subintri-\bthetabar)=& \nhalf \bA^{-1} \left[ N^{-1} \sum_{i = 1}^{N} \bx_i \{ g(\bgammabar \suponetrans \bPhi_i) - g(\bthetabar \trans \bx_i)  \} +\bB_1\trans( \widetilde\bgamma\supone- \bgammabar\supone) \right] +o_p(1),\\
=&\nhalf \sum_{s= 1}^S \rho_s   \left[\ninv_s\sum_{i = 1}^n  I(\Ssc_i = s) \bA^{-1}\bx_i \{y_i - g(\bgammabar \suponetrans \bPhi_i) \}  \right]+ o_p(1)=\widehat\Wsc(\bgammabar\supone)+o_p(1).
\end{align*}
The second equality follows from the fact that
\[
\bB_1\trans(\widetilde\bgamma\supone-\bgammabar\supone)=\bzero+\bXi_{12}.
\]
Thus, the asymptotic variance of $n^{\frac{1}{2}}(\e\trans\widehat\btheta\subintri-\e\trans\bthetabar)$ is ${\rm E}[R(\e\trans\bA^{-1}\bx)^2\{y - g(\bgammabar \suponetrans \bPhi) \}^2]$, which is minimized among those of $\{\e\trans\widehat\Wsc(\bgamma):{\rm E}[\bx\{y - g(\bgamma \trans \bPhi) \}]=\bzero\}$. From Theorem \ref{thm:1}, $n^{\frac{1}{2}}(\widehat\btheta\subplain-\bthetabar)$ is asymptotically equivalent with $\Wsc(\bgammabar)$. Therefore, when the imputation model is correctly specified, that is,  there exists $\bgamma_0$ such that ${\rm P}(y=1\mid \bu)=g(\bPhi\trans\bgamma_0)$, $\bgammabar=\bgammabar\supone=\bgamma_0$, it follows that $n^{\frac{1}{2}}(\widehat\btheta\subintri-\bthetabar)$ is asymptotically equivalent to $n^{\frac{1}{2}}(\widehat\btheta\subplain-\bthetabar)$. This completes the proof of Theorem \ref{thm:3}.

Using our previous arguments, we next establish Theorem \ref{thm:a1}. Similar to (\ref{equ:app:G:1}), we expand $n^{\frac{1}{2}}(\widehat\btheta\subintri^D-\bthetabar)$ as 
\begin{align*}
n^{\frac{1}{2}}(\widehat\btheta\subintri^D-\bthetabar)=&\nhalf \bA^{-1} \left[ N^{-1} \sum_{i = 1}^{N} \bx_i \{ g(\bgammabar \suptwotrans \bPsi_{\bthetatilde i}) - g(\bthetabar \trans \bx_i)  \} +\bB_1\trans(\widetilde\bgamma\suptwo_{\bthetatilde}-\bgammabar\suptwo) \right] +o_p(1),\\
=&\nhalf \sum_{s= 1}^S \rho_s   \left[\ninv_s\sum_{i = 1}^n  I(\Ssc_i = s) \bA^{-1}\bx_i \{y_i - g(\bgammabar \suptwotrans\bar\bPsi_i) \}  \right]\\
&+\nhalf \bA^{-1}\left[N^{-1} \sum_{i = 1}^{N} \bx_i \{g(\bgammabar \suptwotrans \bPsi_{\bthetatilde i})-g(\bgammabar \suptwotrans\bar\bPsi_i)\}+ \dot\bXi_{\btheta,22}\trans(\bthetatilde-\bthetabar)\right]+o_p(1)\\
=&\nhalf \sum_{s= 1}^S \rho_s   \left[\ninv_s\sum_{i = 1}^n  I(\Ssc_i = s) \bA^{-1}\bx_i \{y_i - g(\bgammabar \suptwotrans\bar\bPsi_i) \}  \right]+o_p(1),
\end{align*}
The third equality follows from the fact that $\nhalf(\bthetatilde-\bthetabar)=O_p(1)$ and
\[
\partial\left(N^{-1} \sum_{i = 1}^{N} \bx_i \{g(\bgammabar \suptwotrans \bPsi_{\bthetatilde i})-g(\bgammabar \suptwotrans\bar\bPsi_i)\}\right)/\partial{\btheta}+ \dot\bXi_{\btheta,22}=o_p(1).
\]
Using this result, and applying similar arguments as those used for $\widetilde\bgamma\suptwo_{\bthetatilde}$, we have that
\begin{align*}
\widetilde\bgamma\suptwo_{\bthetahat\subintri^D}-\bgammabar\suptwo=&[\bA_2^{-1}-\bA_2^{-1}\bB_2(\bB_2\trans\bA_2^{-1}\bB_2)^{-1}\bB_2\trans\bA_2^{-1}]\bXi_{21}'+\bA_2^{-1}\bB_2(\bB_2\trans\bA_2^{-1}\bB_2)^{-1}\bXi_{22}'=O_p(n^{-\frac{1}{2}}),\\
\mbox{where}\quad\bXi_{21}'=&\frac{1}{n}\sum_{i=1}^n\zeta_i\{1- 2\yscrl(\bthetabar \trans \bx_i) +{\dot\bD}\trans\bA^{-1}\bx_i\}^2\dot g(\bgammabar\suptwotrans \bar\bPsi_i)\bar\bPsi_i\{y_i - g(\bgammabar\suptwotrans \bar\bPsi_i) \}+\dot\bXi_{\btheta,21}\trans(\bthetahat\subintri^D-\bthetabar);\\
\bXi_{22}'=&\frac{1}{N}\sum_{i=1}^Nw_i[\bx_i\trans,\yscrl(\bthetabar \trans \bx_i)]\trans\{y_i-g(  \bgammabar\suptwotrans\bar\bPsi_i)\}+\dot\bXi_{\btheta,22}\trans(\bthetahat\subintri^D-\bthetabar).
\end{align*}
We then follow the same procedure as in Appendix \ref{sec: asym-DSSL} (specifically, noting that $\widetilde\bgamma\suptwo_{\bthetahat\subintri^D}$ corresponds to the $\bthetahat^D\subintri$ plugged into $\Dhat\subintri(\btheta)$, the derivation for the augmentation approach in Section \ref{sec:aug} can be used directly) to derive that $\Dhat\subintri\overset{p}{\rightarrow}\Dbar$ and 
\[
\nhalf(\Dhat\subintri-\Dbar) =  \nhalf  \left[  N^{-1} \sum_{i = 1}^{N} \{1- 2\Yscbar_i +\bar{\dot\bD}\trans \bA^{-1}\bx_i \} \{y_i - g(\bgammabar\suptwotrans\bar\bPsi_{i})\}\right]+o_p(1)=\widehat\Tsc(\bgammabar\suptwo_{\bthetabar})+o_p(1).
\]
By the definition of $\bgammabar\suptwo_{\bthetabar}$, the asymptotic variance of  $\widehat\Tsc(\bgammabar\suptwo_{\bthetabar})$ is minimized among those of $\{\widehat\Tsc(\bgamma):{\rm E}[\bx\trans,\yscrl(\bthetabar \trans \bx)]\trans\{y-g(  \bgamma\trans\bar\bPsi)\}=\bzero\}$. Additionally, we may use a similar procedure as that in Appendix \ref{sec: asym-DSSL} to derive that 
\[
\nhalf(\Dhat\subplain-\Dbar)=\widehat\Tsc(\bgammabar_{\bthetabar})+o_p(1).
\]
Thus, when the imputation model for estimating $D$, i.e. ${\rm P}(y = 1 \mid \bu) = g(\bgamma \trans \bar\bPsi)$ is correct, we have $\bgammabar_{\bthetabar}=\bgammabar\suptwo_{\bthetabar}$ and that $\nhalf(\Dhat\subintri-\Dbar) $ is asymptotically equivalent to $\nhalf(\Dhat\subplain-\Dbar)$. These arguments establish Theorem \ref{thm:a1}.
}

\section{Justification for Weighted CV Procedure}\label{sec: asym-cvW}
To provide a heuristic justification for the weights for our ensemble CV method, consider an arbitrary smooth loss function $d(\cdot, \cdot)$ and let $\Dsc(\btheta) = {\rm E}[d\{y_0, \Ysc(\btheta\trans\bx_0)\}]$. Let $\Dschat(\btheta)$ denote the empirical unbiased estimate of $\Dsc(\btheta)$ and suppose that $\bthetahat$ minimizes $\Dschat(\btheta)$ (i.e. $\dot \Dschat(\bthetahat)=\bzero$). Suppose that $\nhalf(\bthetahat - \bthetabar) \to N(0, \bSigma)$ in distribution. Then by a Taylor series expansion {of $\Dschat(\bthetabar)$ at $\bthetahat$}, 
$$\Dschat(\bthetahat)=\Dschat(\bthetabar)-\frac{1}{2} (\bthetahat - \bthetabar)\trans  {\ddot \Dschat(\bthetahat)}(\bthetahat-\bthetabar) +o_p(n^{-1}) \quad \mbox{and}$$
\begin{align*}
{\rm E}\{\Dschat(\bthetahat)\} =\Dsc(\bthetabar)  - \frac{1}{2} n^{-1} \text{Tr} \{ \ddot \Dsc(\bthetabar) \bSigma \} +o_p(n^{-1})
\end{align*}
where $\ddot \Dsc(\bthetabar) = \partial \Dsc(\btheta)/\partial\btheta \partial \btheta \trans$. For the $K$-fold CV estimator, $\Dschat_{cv}= K^{-1} \sum_{k = 1}^K \Dschat_{k}(\bthetahat_{(\text{-}k)})$, we note that since $\Dschat_{k}(\btheta)$ is independent of $\bthetahat_{(\text{-}k)}$
 \begin{align*}
{\rm E}(\Dschat_{cv}) &= \Dsc(\bthetabar) +  K^{-1} \sum_{k = 1}^K {\rm E}\{\dot\Dschat_{k}(\bthetabar)\}{\rm E}(\bthetahat_{(\text{-}k)}-\bthetabar) + \frac{1}{2} \frac{K}{K-1}n^{-1} \text{Tr} \{ \ddot \bD(\bthetabar) \bSigma \} + o_p(n^{-1})\\
&=\Dsc(\bthetabar)  + \frac{1}{2} \frac{K}{K-1}n^{-1} \text{Tr} \{ \ddot \bD(\bthetabar) \bSigma \} + o_p(n^{-1}),
\end{align*}
where the second equality follows from the fact that ${\rm E}\{\dot\Dschat_{k}(\bthetabar)\}=\dot\Dsc(\bthetabar)=\bzero$ when $\bthetabar$ minimizes $\Dsc(\btheta)$. Letting $\Dschat_{\omega} =  \omega \Dschat(\bthetahat) +  (1-\omega)\Dschat_{cv}$ with $\omega = K/(2K-1)$, it follows that $\omega n^{-1}  - (1 - \omega)  Kn^{-1}/(K-1) =  0$ and thus
$$
{\rm E}(\Dschat_{\omega}) = \Dsc(\bthetabar)  + o_p(n^{-1}).
$$

\newpage

\begin{center}
    \Large{Supplementary Materials}
\end{center}
\renewcommand \thesection{S\arabic{section}}
\renewcommand \thetable{S\arabic{table}}
\renewcommand \thefigure{S\arabic{figure}}
\renewcommand{\thelemma}{S\arabic{lemma}}
\renewcommand{\thetheorem}{S\arabic{theorem}}
\renewcommand{\thecond}{S\arabic{cond}}
\renewcommand{\theremark}{S\arabic{remark}}
\renewcommand{\theprop}{S\arabic{prop}}

\setcounter{section}{0}

\newpage
\section{Notation List}\label{sec:note-list}
{We present the list of notations used in the paper in Table \ref{tab:notation}.}
\begin{table}[!htb]
\centering
\caption{Main notations in the paper.}
\label{tab:notation}
{\footnotesize
\begin{tabular}{|ll|ll|}
\hline
       Notation & Description & Notation & Description \\
\hline
$\Lscr/\Uscr/\Dscr$ & Set of labelled/unlabelled/all data. & $\btheta$ & Outcome model coefficients. \\
$n/N$ & Size of labelled/all data. & $\bthetabar$ & Limiting parameter for $\btheta$. \\
$S$ & Number of strata. & $\bthetahat\subSL$ & Supervised estimation of $\bthetabar$. \\
$n_s/N_s$ & Size of labelled/all data from strata $s$. & $\bthetahat\subDR$ & Density ratio (DR) estimator. \\
$(y,\bx)$ & Response and regressors. & $\bgamma$ & Imputation model coefficients. \\
$p$ & Dimension of $\bx$. & $\bgammabar/\bgammatilde$ & Limitation/Estimation of $\bgamma$. \\
$\Ssc$ & Index for the strata. & $\bWhat/\bWhat_j$ & Weight for the SSL estimation. \\
$\bu$ & $\bu=(\bx\trans,\Ssc)\trans$. & $\widecheck\btheta\subSSL/\bthetahat\subplain$ & (Plain) SSL estimation for $\bthetabar$.\\
& & $\widehat\btheta\subintri$ & Intrinsic efficient SSL estimation for $\bthetabar$.\\
$\bF$ & $\bF=(y,\bu\trans)\trans$. & $\bthetahat\subSSL$ & $\bthetahat\subSSL=\bWhat\bthetahat\subSL+(\bI-\bWhat)\widecheck\btheta\subSSL$.  \\
$V_i$ & Indicator: $V_i=1$ if $y_i$ is observed. & $D(\btheta)$ & Accuracy measure. \\
$\pi_s$ & Proportion of labelled data in strata $s$. & $\Dbar/\Dbar_1/\Dbar_2$ & $\Dbar=\Dbar(\bthetabar)$, limitation of $D(\btheta)$. \\
$\rho_{1s}$, $\rho_{s}$ & Proportion of $s$ in labelled/all data. & $\bnu_{\btheta}$ & The augmented term coefficients. \\
$\pihat_s$, $\rhohat_{1s}$, $\rhohat_{s}$ & $\pihat_s=n_s/N_s$, $\rhohat_{1s}=n_s/n$, $\rhohat_{s}=N_s/N$. & $\bnubar_{\btheta}$, $\bnutilde_{\btheta}$ & Limitation/Estimation of $\bnu_{\btheta}$. \\
$\widehat{\omega}_i$ & Inverse probability weights $\widehat{\omega}_i=V_i/\pihat_{\Ssc_i}$. & $\Dhat\subSL(\bthetahat\subSL)$ & SL estimation of $\Dbar(\bthetabar)$. \\
$g(\cdot)$ & Link function of the model. & $\Dhat\subSSL(\bthetahat\subSSL)$ & SSL estimation of $\Dbar(\bthetabar)$. \\
& & $\Dhat\subplain$ & Plain SSL estimation of $\Dbar$. \\
& & $\Dhat\subintri$ & Intrinsic efficient SSL estimation of $\Dbar$. \\
$m(\cdot)$ & Conditional mean: $m(\bu)={\rm P}(y=1|\bu)$. & $\Dhat\subSL\supcv$/$\Dhat\subSSL\supcv$ & SL/SSL CV estimation of $\Dbar$. \\
$d(\cdot,\cdot)$ & Function: $d(y,z)=y(1-2z)+z^2$. & $\Dhat\subSL^w$/$\Dhat\subSSL^w$ & SL/SSL ensemble estimation of $\Dbar$. \\
$\bPhi(\bu)$,$\bPhi$ & Basis for the imputation model. & $\Dhat\subDR^w$ & DR (ensemble) estimation of $\Dbar$. \\
$\Ysc$ $(\Ysc_1/\Ysc_2)$ & Prediction (mean/label) for $y$. & $\Dhat\subSL^w$/$\Dhat\subSSL^w$ & SL/SSL ensemble estimation of $\Dbar$. \\
$\bz_{\btheta}$ & Augmented variable $\bz_{\btheta}=[1,\yscrl(\btheta \trans \bx)]\trans$. & $\bthetahat_{\tiny (-k)}/\bgammatilde_{\tiny (-k)}$ & Estimated with data except $\Lscr_k$. \\
$\mtilde/\mtilde_I$/$\mtilde_{II}$ & Imputation/(augmented) of $m(\cdot)$. & $\bthetahat^*/\bgammatilde^*/\Dhat^*$ & Perturbation estimators. \\
$\Lscr_k$ & The $k$-th fold labelled data. & $\Wschat\subSL$ & $\Wschat\subSL=\nhalf(\bthetahat\subSL - \bthetabar)$.\\
$K$ & Number of CV folds. & $\Wschat\subSSL$ & $\Wschat\subSSL=\nhalf(\widecheck\btheta\subSSL - \bthetabar)$.\\
$w$ & The ensemble weight $w=K/(2K-1)$ & $\Tschat\subSL$ & $\Tschat\subSL =  \nhalf\{\Dhat\subSL( \bthetahat \subSL) - D( \bthetabar)\}$. \\
$\bA$ & $\bA = {\rm E} \{ \bx_i^{\otimes 2} \dot{g}(\bthetabar \trans  \bx_i) \}$. & $\widecheck\Tsc\subSSL$ & $\widecheck\Tsc\subSSL = \nhalf\{\Dhat\subSSL( \widecheck\btheta \subSSL) - D( \bthetabar)\} $. \\
$\bSigma\subSL,\bSigma\subSSL$ & Asymptotic variance of $\Wschat\subSL$ and $\Wschat\subSSL$. & $\Tschat\subSSL$ & $\Tschat\subSSL = \nhalf\{\Dhat\subSSL( \bthetahat \subSSL) - D( \bthetabar)\} $. \\
\hline
\end{tabular}
}
\end{table}

\newpage

\def\subori{_{\scriptscriptstyle \sf ori}}
\def\suborik{_{{\scriptscriptstyle \sf ori}, k}}
\def\subtra{_{\scriptscriptstyle \sf tra}}
\def\bzero{\mathbf{0}}
\def\bPsi{\boldsymbol{\Psi}}

\section{Simulation Results: Regression Parameter}\label{sec:regr-sim}

In Table \ref{table: betas}, we summarize the results for $\bthetahat\subSL$, $\bthetahat\subSSL$ and the DR estimator of $\bthetabar$, $\bthetahat\subDR$, when $S =2$ and $n_s=100$. Results for $(S=2,n_s=200)$, $(S=4,n_s=100)$, and $(S=4,n_s=200)$ present similar patterns and are shown in Table \ref{tab:app:11}-\ref{tab:app:13}.  All estimators have negligible biases of comparable magnitudes for estimating $\bthetabar$.  The proposed variance estimation procedure performs well in finite sample with empirical coverage of the 95\% CIs close to the nominal level in all settings. Consistent with our theoretical findings, $\bthetahat\subSL$ and $\bthetahat\subSSL$ are equally efficient under setting $(\Msc_{\mbox{\tiny correct}}$, $\Isc_{\mbox{\tiny correct}})$, and $\bthetahat\subSSL$ is substantially more efficient under settings  $(\Msc_{\mbox{\tiny incorrect}}, \Isc_{\mbox{\tiny correct}})$ and $(\Msc_{\mbox{\tiny incorrect}}, \Isc_{\mbox{\tiny incorrect}})$. In contrast, the DR estimator $\bthetahat\subDR$ has efficiency similar to that of $\bthetahat\subSL$ across all settings and hence performs worse than $\bthetahat\subSSL$ under $(\Msc_{\mbox{\tiny incorrect}}, \Isc_{\mbox{\tiny correct}})$ and $(\Msc_{\mbox{\tiny incorrect}}, \Isc_{\mbox{\tiny incorrect}})$.  This behavior is most likely due to the difficulty in obtaining a good approximation to the density ratio.

\begin{table}[!htbp]
    \caption{Bias, empirical SE (ESE), average of the estimated standard errors (ASE) and coverage probabilities (CP) of the 95\% CIs under (i) ($\Msc_{\mbox{\tiny correct}}$, $\Isc_{\mbox{\tiny correct}}$),  (ii)  ($\Msc_{\mbox{\tiny incorrect}}$, $\Isc_{\mbox{\tiny correct}}$), and (iii) ($\Msc_{\mbox{\tiny incorrect}}$, $\Isc_{\mbox{\tiny incorrect}}$) when $S = 2$ and $n_s=100$.  Shown also is the relative efficiency (RE) of the SS estimators $\bthetahat\subSSL$ and $\bthetahat\subDR$ to the supervised estimator $\bthetahat\subSL$ to with respect to mean square error.}
    \label{table: betas}
      \centering
       \begin{tabular}{crrr|rrrr|rrr}
       \hline
     & \multicolumn{1}{r}{} & \multicolumn{2}{c|}{$\bthetahat\subSL$}& \multicolumn{4}{c|}{$\bthetahat\subSSL$} & \multicolumn{3}{c}{$\bthetahat\subDR$} \\ 
  \hline
 && Bias & $\text{ESE}$ & Bias & $\text{ESE}_{\text{ASE}}$ & CP & RE & Bias & $\text{ESE}$ & RE\\  \hline  \hline
(i) & $\theta_0$ & $0.31$ & $0.58$ & $0.30$ & $0.58_{0.48}$ & $0.96$ & $1.01$ & $0.34$ & $0.62$ & $0.87$  \\
& $\theta_1$ & $0.15$ & $0.32$ & $0.14$ & $0.32_{0.27}$ & $0.94$ & $0.98$ & $0.16$ & $0.34$ & $0.88$  \\
& $\theta_2$ & $0.15$ & $0.32$ & $0.14$ & $0.32_{0.28}$ & $0.96$ & $1.04$ & $0.17$ & $0.34$ & $0.89$  \\
& $\theta_3$ & $0.08$ & $0.24$ & $0.07$ & $0.24_{0.22}$ & $0.94$ & $1.02$ & $0.09$ & $0.26$ & $0.91$  \\
& $\theta_4$ & $0.08$ & $0.25$ & $0.07$ & $0.25_{0.23}$ & $0.94$ & $1.01$ & $0.08$ & $0.26$ & $0.91$  \\
& $\theta_5$ & $0.00$ & $0.22$ & $0.01$ & $0.22_{0.20}$ & $0.94$ & $0.98$ & $0.01$ & $0.23$ & $0.92$  \\
& $\theta_6$ & $0.02$ & $0.20$ & $0.01$ & $0.19_{0.20}$ & $0.96$ & $1.02$ & $0.01$ & $0.21$ & $0.89$  \\
& $\theta_7$ & $0.01$ & $0.19$ & $0.01$ & $0.19_{0.20}$ & $0.96$ & $0.98$ & $0.01$ & $0.21$ & $0.85$  \\
& $\theta_8$ & $0.01$ & $0.21$ & $0.01$ & $0.21_{0.20}$ & $0.95$ & $1.01$ & $0.01$ & $0.23$ & $0.87$  \\
& $\theta_9$ & $0.01$ & $0.21$ & $0.01$ & $0.21_{0.20}$ & $0.94$ & $0.97$ & $0.01$ & $0.22$ & $0.91$  \\
& $\theta_{10}$ & $0.00$ & $0.19$ & $0.00$ & $0.19_{0.19}$ & $0.95$ & $0.96$ & $0.00$ & $0.20$ & $0.91$  \\
  \hline\hline
(ii)& $\theta_0$ & $0.01$ & $0.21$ & $0.01$ & $0.19_{0.19}$ & $0.96$ & $1.27$ & $0.01$ & $0.21$ & $1.02$  \\
& $\theta_1$ & $0.05$ & $0.18$ & $0.02$ & $0.14_{0.14}$ & $0.95$ & $1.67$ & $0.07$ & $0.19$ & $0.89$  \\
& $\theta_2$ & $0.06$ & $0.19$ & $0.03$ & $0.15_{0.15}$ & $0.94$ & $1.62$ & $0.08$ & $0.21$ & $0.81$  \\
& $\theta_3$ & $0.03$ & $0.15$ & $0.02$ & $0.13_{0.13}$ & $0.95$ & $1.38$ & $0.04$ & $0.17$ & $0.81$  \\
& $\theta_4$ & $0.05$ & $0.16$ & $0.03$ & $0.13_{0.13}$ & $0.95$ & $1.42$ & $0.06$ & $0.18$ & $0.81$  \\
& $\theta_5$ & $0.00$ & $0.15$ & $0.00$ & $0.12_{0.13}$ & $0.96$ & $1.54$ & $0.01$ & $0.17$ & $0.81$  \\
& $\theta_6$ & $0.00$ & $0.15$ & $0.01$ & $0.13_{0.13}$ & $0.96$ & $1.50$ & $0.00$ & $0.16$ & $0.90$  \\
& $\theta_7$ & $0.00$ & $0.14$ & $0.00$ & $0.12_{0.13}$ & $0.96$ & $1.38$ & $0.01$ & $0.15$ & $0.80$  \\
& $\theta_8$ & $0.00$ & $0.14$ & $0.00$ & $0.13_{0.13}$ & $0.95$ & $1.28$ & $0.00$ & $0.16$ & $0.79$  \\
& $\theta_9$ & $0.01$ & $0.14$ & $0.01$ & $0.12_{0.13}$ & $0.94$ & $1.27$ & $0.01$ & $0.16$ & $0.82$  \\
& $\theta_{10}$ & $0.00$ & $0.13$ & $0.00$ & $0.11_{0.12}$ & $0.94$ & $1.27$ & $0.00$ & $0.14$ & $0.81$  \\
\hline  \hline
(iii)& $\theta_0$ & $0.16$ & $0.35$ & $0.09$ & $0.29_{0.26}$ & $0.95$ & $1.45$ & $0.17$ & $0.35$ & $0.99$  \\
& $\theta_1$ & $0.13$ & $0.29$ & $0.10$ & $0.24_{0.20}$ & $0.94$ & $1.45$ & $0.14$ & $0.29$ & $0.94$  \\
& $\theta_2$ & $0.05$ & $0.20$ & $0.01$ & $0.16_{0.17}$ & $0.95$ & $1.54$ & $0.06$ & $0.22$ & $0.86$  \\
& $\theta_3$ & $0.00$ & $0.18$ & $0.05$ & $0.17_{0.17}$ & $0.93$ & $1.21$ & $0.01$ & $0.19$ & $0.98$  \\
& $\theta_4$ & $0.02$ & $0.18$ & $0.01$ & $0.16_{0.16}$ & $0.94$ & $1.30$ & $0.03$ & $0.19$ & $0.87$  \\
& $\theta_5$ & $0.10$ & $0.28$ & $0.05$ & $0.24_{0.21}$ & $0.93$ & $1.41$ & $0.11$ & $0.30$ & $0.89$  \\
& $\theta_6$ & $0.08$ & $0.27$ & $0.02$ & $0.23_{0.20}$ & $0.90$ & $1.35$ & $0.10$ & $0.29$ & $0.88$  \\
& $\theta_7$ & $0.01$ & $0.18$ & $0.02$ & $0.15_{0.16}$ & $0.96$ & $1.47$ & $0.01$ & $0.20$ & $0.87$  \\
& $\theta_8$ & $0.00$ & $0.19$ & $0.01$ & $0.15_{0.16}$ & $0.95$ & $1.51$ & $0.00$ & $0.20$ & $0.89$  \\
& $\theta_9$ & $0.00$ & $0.18$ & $0.00$ & $0.15_{0.16}$ & $0.94$ & $1.35$ & $0.00$ & $0.19$ & $0.87$  \\
& $\theta_{10}$ & $0.01$ & $0.17$ & $0.00$ & $0.15_{0.15}$ & $0.94$ & $1.30$ & $0.01$ & $0.18$ & $0.91$  \\
\hline
\end{tabular}
\end{table}

\newpage

\begin{table}[!htbp]
\caption{Bias, empirical SE (ESE), average of the estimated standard errors (ASE) and coverage probabilities (CP) of the 95\% CIs under (i) ($\Msc_{\mbox{\tiny correct}}$, $\Isc_{\mbox{\tiny correct}}$), (ii)  ($\Msc_{\mbox{\tiny incorrect}}$, $\Isc_{\mbox{\tiny correct}}$), and (iii) ($\Msc_{\mbox{\tiny incorrect}}$, $\Isc_{\mbox{\tiny incorrect}}$) when $S = 2$ and $n_s=200$. Shown also is the relative efficiency (RE) of the SS estimators $\bthetahat\subSSL$ and $\bthetahat\subDR$ to the supervised estimator $\bthetahat\subSL$ to with respect to mean square error.}
\label{tab:app:11}
      \centering
       \begin{tabular}{crrr|rrrr|rrr}
       
       \hline
      & \multicolumn{1}{r}{} & \multicolumn{2}{c|}{$\bthetahat\subSL$}& \multicolumn{4}{c|}{$\bthetahat\subSSL$} & \multicolumn{3}{c}{$\bthetahat\subDR$} \\ 
  \hline
 && Bias & ESE & Bias & $\text{ESE}_{\text{ASE}}$ & CP & RE & Bias & ESE & RE\\ 
  \hline
  \hline
(i) & $\theta_0$ & $0.14$ & $0.31$ & $0.14$ & $0.32_{0.28}$ & $0.95$ & $0.97$ & $0.14$ & $0.32$ & $0.97$  \\
& $\theta_1$ & $0.07$ & $0.18$ & $0.07$ & $0.18_{0.16}$ & $0.95$ & $0.97$ & $0.07$ & $0.18$ & $0.96$  \\
& $\theta_2$ & $0.07$ & $0.18$ & $0.07$ & $0.18_{0.17}$ & $0.95$ & $0.98$ & $0.08$ & $0.19$ & $0.95$  \\
& $\theta_3$ & $0.03$ & $0.14$ & $0.03$ & $0.14_{0.14}$ & $0.95$ & $1.00$ & $0.03$ & $0.14$ & $0.95$  \\
& $\theta_4$ & $0.04$ & $0.15$ & $0.04$ & $0.15_{0.14}$ & $0.95$ & $0.97$ & $0.04$ & $0.15$ & $0.98$  \\
& $\theta_5$ & $0.00$ & $0.13$ & $0.00$ & $0.13_{0.13}$ & $0.93$ & $0.96$ & $0.00$ & $0.13$ & $0.95$  \\
& $\theta_6$ & $0.00$ & $0.13$ & $0.00$ & $0.13_{0.13}$ & $0.93$ & $0.98$ & $0.00$ & $0.13$ & $0.94$  \\
& $\theta_7$ & $0.00$ & $0.13$ & $0.00$ & $0.13_{0.13}$ & $0.95$ & $0.98$ & $0.00$ & $0.13$ & $0.93$  \\
& $\theta_8$ & $0.00$ & $0.12$ & $0.00$ & $0.12_{0.13}$ & $0.97$ & $0.98$ & $0.00$ & $0.12$ & $0.93$  \\
& $\theta_9$ & $0.00$ & $0.13$ & $0.00$ & $0.13_{0.13}$ & $0.95$ & $0.99$ & $0.00$ & $0.13$ & $0.97$  \\
& $\theta_{10}$ & $0.00$ & $0.12$ & $0.00$ & $0.12_{0.12}$ & $0.95$ & $0.99$ & $0.00$ & $0.12$ & $0.96$  \\
  \hline
  \hline
(ii)& $\theta_0$ & $0.01$ & $0.14$ & $0.00$ & $0.12_{0.12}$ & $0.95$ & $1.43$ & $0.00$ & $0.13$ & $1.20$  \\
& $\theta_1$ & $0.03$ & $0.11$ & $0.01$ & $0.09_{0.09}$ & $0.95$ & $1.63$ & $0.03$ & $0.11$ & $1.03$  \\
& $\theta_2$ & $0.03$ & $0.12$ & $0.01$ & $0.09_{0.09}$ & $0.97$ & $1.77$ & $0.04$ & $0.12$ & $0.97$  \\
& $\theta_3$ & $0.02$ & $0.10$ & $0.01$ & $0.09_{0.09}$ & $0.95$ & $1.37$ & $0.02$ & $0.11$ & $0.89$  \\
& $\theta_4$ & $0.02$ & $0.10$ & $0.01$ & $0.09_{0.09}$ & $0.95$ & $1.35$ & $0.02$ & $0.11$ & $0.93$  \\
& $\theta_5$ & $0.00$ & $0.10$ & $0.00$ & $0.08_{0.08}$ & $0.95$ & $1.43$ & $0.00$ & $0.10$ & $0.94$  \\
& $\theta_6$ & $0.01$ & $0.10$ & $0.01$ & $0.08_{0.08}$ & $0.94$ & $1.58$ & $0.01$ & $0.10$ & $1.03$  \\
& $\theta_7$ & $0.00$ & $0.09$ & $0.00$ & $0.08_{0.08}$ & $0.94$ & $1.29$ & $0.00$ & $0.10$ & $0.96$  \\
& $\theta_8$ & $0.00$ & $0.09$ & $0.01$ & $0.08_{0.08}$ & $0.95$ & $1.35$ & $0.01$ & $0.10$ & $0.91$  \\
& $\theta_9$ & $0.00$ & $0.09$ & $0.00$ & $0.08_{0.08}$ & $0.95$ & $1.31$ & $0.00$ & $0.10$ & $0.91$  \\
& $\theta_{10}$ & $0.00$ & $0.09$ & $0.00$ & $0.08_{0.08}$ & $0.93$ & $1.31$ & $0.01$ & $0.09$ & $0.91$  \\
  \hline
  \hline
(iii)& $\theta_0$ & $0.09$ & $0.22$ & $0.05$ & $0.17_{0.16}$ & $0.96$ & $1.63$ & $0.08$ & $0.20$ & $1.19$  \\
& $\theta_1$ & $0.06$ & $0.15$ & $0.04$ & $0.12_{0.12}$ & $0.95$ & $1.56$ & $0.06$ & $0.14$ & $1.16$  \\
& $\theta_2$ & $0.03$ & $0.13$ & $0.01$ & $0.10_{0.10}$ & $0.95$ & $1.54$ & $0.03$ & $0.13$ & $0.95$  \\
& $\theta_3$ & $0.01$ & $0.13$ & $0.02$ & $0.10_{0.10}$ & $0.95$ & $1.49$ & $0.01$ & $0.12$ & $1.10$  \\
& $\theta_4$ & $0.01$ & $0.12$ & $0.00$ & $0.10_{0.10}$ & $0.95$ & $1.52$ & $0.01$ & $0.13$ & $0.96$  \\
& $\theta_5$ & $0.05$ & $0.15$ & $0.02$ & $0.13_{0.13}$ & $0.94$ & $1.32$ & $0.05$ & $0.15$ & $1.01$  \\
& $\theta_6$ & $0.05$ & $0.16$ & $0.02$ & $0.13_{0.13}$ & $0.93$ & $1.50$ & $0.05$ & $0.16$ & $1.00$  \\
& $\theta_7$ & $0.01$ & $0.11$ & $0.01$ & $0.10_{0.10}$ & $0.95$ & $1.33$ & $0.01$ & $0.11$ & $0.94$  \\
& $\theta_8$ & $0.01$ & $0.11$ & $0.01$ & $0.09_{0.10}$ & $0.96$ & $1.47$ & $0.01$ & $0.11$ & $0.95$  \\
& $\theta_9$ & $0.00$ & $0.11$ & $0.00$ & $0.09_{0.10}$ & $0.96$ & $1.56$ & $0.01$ & $0.12$ & $0.91$  \\
& $\theta_{10}$ & $0.01$ & $0.10$ & $0.01$ & $0.08_{0.09}$ & $0.97$ & $1.54$ & $0.01$ & $0.11$ & $0.91$  \\
  \hline
\end{tabular}
\end{table}

\newpage

\begin{table}[!htbp]
    \caption{Bias, empirical SE (ESE), average of the estimated standard errors (ASE) and coverage probabilities (CP) of the 95\% CIs under (i) ($\Msc_{\mbox{\tiny correct}}$, $\Isc_{\mbox{\tiny correct}}$), (ii)  ($\Msc_{\mbox{\tiny incorrect}}$, $\Isc_{\mbox{\tiny correct}}$), and (iii) ($\Msc_{\mbox{\tiny incorrect}}$, $\Isc_{\mbox{\tiny incorrect}}$) when $S = 4$ and $n_s=100$.  Shown also is the relative efficiency (RE) of the SS estimators $\bthetahat\subSSL$ and $\bthetahat\subDR$ to the supervised estimator $\bthetahat\subSL$ to with respect to mean square error.}
    \label{tab:app:12}
      \centering
       \begin{tabular}{crrr|rrrr|rrr}
       
       \hline
      & \multicolumn{1}{r}{} & \multicolumn{2}{c|}{$\bthetahat\subSL$}& \multicolumn{4}{c|}{$\bthetahat\subSSL$} & \multicolumn{3}{c}{$\bthetahat\subDR$} \\ 
  \hline
 && Bias & ESE & Bias & $\text{ESE}_{\text{ASE}}$ & CP & RE & Bias & ESE & RE\\ 
  \hline
  \hline
(i) & $\theta_0$ & $0.14$ & $0.33$ & $0.14$ & $0.33_{0.28}$ & $0.92$ & $1.00$ & $0.15$ & $0.34$ & $0.97$  \\
& $\theta_1$ & $0.07$ & $0.18$ & $0.07$ & $0.18_{0.16}$ & $0.94$ & $0.99$ & $0.07$ & $0.18$ & $0.97$  \\
& $\theta_2$ & $0.06$ & $0.18$ & $0.06$ & $0.18_{0.17}$ & $0.94$ & $0.99$ & $0.06$ & $0.19$ & $0.93$  \\
& $\theta_3$ & $0.03$ & $0.15$ & $0.03$ & $0.15_{0.14}$ & $0.93$ & $0.99$ & $0.04$ & $0.15$ & $0.97$  \\
& $\theta_4$ & $0.04$ & $0.15$ & $0.04$ & $0.15_{0.14}$ & $0.93$ & $0.98$ & $0.04$ & $0.15$ & $0.96$  \\
& $\theta_5$ & $0.01$ & $0.13$ & $0.01$ & $0.13_{0.13}$ & $0.93$ & $0.97$ & $0.01$ & $0.13$ & $0.96$  \\
& $\theta_6$ & $0.00$ & $0.12$ & $0.00$ & $0.12_{0.13}$ & $0.95$ & $1.01$ & $0.00$ & $0.13$ & $0.97$  \\
& $\theta_7$ & $0.00$ & $0.12$ & $0.00$ & $0.12_{0.13}$ & $0.94$ & $0.99$ & $0.00$ & $0.13$ & $0.94$  \\
& $\theta_8$ & $0.01$ & $0.12$ & $0.01$ & $0.12_{0.13}$ & $0.96$ & $0.96$ & $0.01$ & $0.12$ & $0.97$  \\
& $\theta_9$ & $0.01$ & $0.12$ & $0.01$ & $0.12_{0.13}$ & $0.95$ & $1.00$ & $0.01$ & $0.12$ & $0.93$  \\
& $\theta_{10}$ & $0.01$ & $0.12$ & $0.01$ & $0.12_{0.12}$ & $0.93$ & $0.99$ & $0.01$ & $0.12$ & $0.93$  \\
  \hline
  \hline
(ii) & $\theta_0$ & $0.07$ & $0.15$ & $0.07$ & $0.13_{0.12}$ & $0.93$ & $1.27$ & $0.07$ & $0.14$ & $1.09$  \\
& $\theta_1$ & $0.05$ & $0.13$ & $0.04$ & $0.10_{0.09}$ & $0.93$ & $1.59$ & $0.06$ & $0.13$ & $0.96$  \\
& $\theta_2$ & $0.03$ & $0.12$ & $0.02$ & $0.10_{0.10}$ & $0.95$ & $1.57$ & $0.03$ & $0.12$ & $1.04$  \\
& $\theta_3$ & $0.00$ & $0.10$ & $0.01$ & $0.09_{0.09}$ & $0.94$ & $1.27$ & $0.00$ & $0.10$ & $0.92$  \\
& $\theta_4$ & $0.02$ & $0.10$ & $0.01$ & $0.09_{0.09}$ & $0.95$ & $1.26$ & $0.02$ & $0.11$ & $0.89$  \\
& $\theta_5$ & $0.01$ & $0.10$ & $0.00$ & $0.08_{0.08}$ & $0.95$ & $1.50$ & $0.01$ & $0.10$ & $0.93$  \\
& $\theta_6$ & $0.01$ & $0.10$ & $0.01$ & $0.08_{0.09}$ & $0.96$ & $1.58$ & $0.02$ & $0.11$ & $0.93$  \\
& $\theta_7$ & $0.00$ & $0.10$ & $0.00$ & $0.09_{0.09}$ & $0.94$ & $1.36$ & $0.00$ & $0.11$ & $0.89$  \\
& $\theta_8$ & $0.00$ & $0.09$ & $0.00$ & $0.08_{0.09}$ & $0.96$ & $1.32$ & $0.00$ & $0.10$ & $0.90$  \\
& $\theta_9$ & $0.00$ & $0.10$ & $0.00$ & $0.08_{0.09}$ & $0.95$ & $1.28$ & $0.00$ & $0.10$ & $0.96$  \\
& $\theta_{10}$ & $0.00$ & $0.08$ & $0.00$ & $0.07_{0.08}$ & $0.97$ & $1.29$ & $0.00$ & $0.09$ & $0.92$  \\
  \hline
  \hline
(iii) & $\theta_0$ & $0.08$ & $0.20$ & $0.04$ & $0.16_{0.17}$ & $0.96$ & $1.55$ & $0.07$ & $0.19$ & $1.18$  \\
& $\theta_1$ & $0.06$ & $0.16$ & $0.03$ & $0.13_{0.13}$ & $0.94$ & $1.48$ & $0.05$ & $0.15$ & $1.10$  \\
& $\theta_2$ & $0.03$ & $0.13$ & $0.01$ & $0.10_{0.11}$ & $0.96$ & $1.52$ & $0.03$ & $0.12$ & $1.01$  \\
& $\theta_3$ & $0.01$ & $0.12$ & $0.01$ & $0.10_{0.10}$ & $0.94$ & $1.45$ & $0.00$ & $0.12$ & $1.04$  \\
& $\theta_4$ & $0.01$ & $0.12$ & $0.00$ & $0.09_{0.10}$ & $0.96$ & $1.53$ & $0.01$ & $0.12$ & $0.94$  \\
& $\theta_5$ & $0.04$ & $0.16$ & $0.01$ & $0.13_{0.13}$ & $0.95$ & $1.56$ & $0.04$ & $0.15$ & $1.06$  \\
& $\theta_6$ & $0.05$ & $0.16$ & $0.02$ & $0.13_{0.13}$ & $0.95$ & $1.51$ & $0.05$ & $0.15$ & $1.09$  \\
& $\theta_7$ & $0.00$ & $0.11$ & $0.01$ & $0.09_{0.10}$ & $0.96$ & $1.49$ & $0.00$ & $0.12$ & $0.90$  \\
& $\theta_8$ & $0.00$ & $0.11$ & $0.00$ & $0.09_{0.10}$ & $0.97$ & $1.58$ & $0.00$ & $0.12$ & $0.95$  \\
& $\theta_9$ & $0.00$ & $0.12$ & $0.00$ & $0.09_{0.10}$ & $0.97$ & $1.56$ & $0.01$ & $0.12$ & $0.94$  \\
& $\theta_{10}$ & $0.00$ & $0.10$ & $0.00$ & $0.09_{0.09}$ & $0.96$ & $1.39$ & $0.00$ & $0.11$ & $0.93$  \\
\hline
\end{tabular}
\end{table}

\newpage
\begin{table}[!htbp]
\centering
 \caption{Bias, empirical SE (ESE), average of the estimated standard errors (ASE) and coverage probabilities (CP) of the 95\% CIs under (i) ($\Msc_{\mbox{\tiny correct}}$, $\Isc_{\mbox{\tiny correct}}$), (ii)  ($\Msc_{\mbox{\tiny incorrect}}$, $\Isc_{\mbox{\tiny correct}}$), and (iii) ($\Msc_{\mbox{\tiny incorrect}}$, $\Isc_{\mbox{\tiny incorrect}}$) when $S = 4$ and $n_s=200$. Shown also is the relative efficiency (RE) of the SS estimators $\bthetahat\subSSL$ and $\bthetahat\subDR$ to the supervised estimator $\bthetahat\subSL$ to with respect to mean square error.}
    \label{tab:app:13}
      \centering
       \begin{tabular}{crrr|rrrr|rrr}
       
       \hline
     & \multicolumn{1}{r}{} & \multicolumn{2}{c|}{$\bthetahat\subSL$}& \multicolumn{4}{c|}{$\bthetahat\subSSL$} & \multicolumn{3}{c}{$\bthetahat\subDR$} \\ 
  \hline
 && Bias & ESE & Bias & $\text{ESE}_{\text{ASE}}$ & CP & RE & Bias & ESE & RE\\ 
  \hline
  \hline
(i)& $\theta_0$ & $0.05$ & $0.19$ & $0.05$ & $0.20_{0.19}$ & $0.95$ & $0.99$ & $0.06$ & $0.20$ & $0.94$  \\
& $\theta_1$ & $0.03$ & $0.11$ & $0.03$ & $0.11_{0.11}$ & $0.95$ & $0.99$ & $0.03$ & $0.11$ & $0.96$  \\
& $\theta_2$ & $0.03$ & $0.12$ & $0.03$ & $0.12_{0.11}$ & $0.95$ & $0.97$ & $0.03$ & $0.12$ & $0.97$  \\
& $\theta_3$ & $0.01$ & $0.09$ & $0.01$ & $0.09_{0.09}$ & $0.95$ & $1.00$ & $0.01$ & $0.09$ & $0.97$  \\
& $\theta_4$ & $0.02$ & $0.10$ & $0.02$ & $0.10_{0.09}$ & $0.95$ & $0.99$ & $0.02$ & $0.10$ & $0.95$  \\
& $\theta_5$ & $0.00$ & $0.08$ & $0.00$ & $0.08_{0.08}$ & $0.96$ & $1.00$ & $0.00$ & $0.08$ & $0.97$  \\
& $\theta_6$ & $0.00$ & $0.08$ & $0.00$ & $0.08_{0.08}$ & $0.96$ & $0.99$ & $0.00$ & $0.08$ & $0.97$  \\
& $\theta_7$ & $0.00$ & $0.08$ & $0.00$ & $0.08_{0.08}$ & $0.95$ & $1.01$ & $0.00$ & $0.08$ & $0.99$  \\
& $\theta_8$ & $0.01$ & $0.08$ & $0.01$ & $0.08_{0.08}$ & $0.95$ & $1.00$ & $0.01$ & $0.08$ & $0.99$  \\
& $\theta_9$ & $0.00$ & $0.08$ & $0.00$ & $0.08_{0.08}$ & $0.95$ & $1.00$ & $0.00$ & $0.08$ & $0.97$  \\
& $\theta_{10}$ & $0.00$ & $0.08$ & $0.00$ & $0.08_{0.08}$ & $0.94$ & $1.01$ & $0.00$ & $0.08$ & $0.97$  \\
 \hline
  \hline
(ii) & $\theta_0$ & $0.07$ & $0.12$ & $0.07$ & $0.10_{0.08}$ & $0.91$ & $1.26$ & $0.08$ & $0.12$ & $1.04$  \\
& $\theta_1$ & $0.04$ & $0.09$ & $0.04$ & $0.07_{0.06}$ & $0.92$ & $1.80$ & $0.05$ & $0.09$ & $1.07$  \\
& $\theta_2$ & $0.01$ & $0.08$ & $0.00$ & $0.06_{0.06}$ & $0.94$ & $1.59$ & $0.01$ & $0.08$ & $1.08$  \\
& $\theta_3$ & $0.01$ & $0.07$ & $0.02$ & $0.06_{0.06}$ & $0.94$ & $1.36$ & $0.01$ & $0.07$ & $0.97$  \\
& $\theta_4$ & $0.01$ & $0.07$ & $0.00$ & $0.06_{0.06}$ & $0.93$ & $1.29$ & $0.01$ & $0.07$ & $0.95$  \\
& $\theta_5$ & $0.00$ & $0.07$ & $0.00$ & $0.06_{0.06}$ & $0.95$ & $1.56$ & $0.01$ & $0.07$ & $1.00$  \\
& $\theta_6$ & $0.01$ & $0.07$ & $0.01$ & $0.06_{0.06}$ & $0.94$ & $1.53$ & $0.01$ & $0.07$ & $0.99$  \\
& $\theta_7$ & $0.00$ & $0.06$ & $0.00$ & $0.05_{0.06}$ & $0.96$ & $1.47$ & $0.00$ & $0.06$ & $0.92$  \\
& $\theta_8$ & $0.00$ & $0.07$ & $0.00$ & $0.05_{0.06}$ & $0.95$ & $1.45$ & $0.00$ & $0.07$ & $0.95$  \\
& $\theta_9$ & $0.00$ & $0.06$ & $0.00$ & $0.06_{0.06}$ & $0.94$ & $1.32$ & $0.00$ & $0.07$ & $0.96$  \\
& $\theta_{10}$ & $0.00$ & $0.06$ & $0.00$ & $0.05_{0.05}$ & $0.93$ & $1.39$ & $0.00$ & $0.06$ & $0.96$  \\
  \hline
  \hline
(iii) & $\theta_0$ & $0.04$ & $0.13$ & $0.02$ & $0.11_{0.11}$ & $0.96$ & $1.49$ & $0.03$ & $0.12$ & $1.21$  \\
& $\theta_1$ & $0.03$ & $0.11$ & $0.02$ & $0.09_{0.09}$ & $0.95$ & $1.55$ & $0.02$ & $0.10$ & $1.18$  \\
& $\theta_2$ & $0.02$ & $0.08$ & $0.01$ & $0.07_{0.07}$ & $0.96$ & $1.52$ & $0.02$ & $0.09$ & $0.98$  \\
& $\theta_3$ & $0.00$ & $0.08$ & $0.00$ & $0.07_{0.07}$ & $0.94$ & $1.49$ & $0.00$ & $0.08$ & $1.05$  \\
& $\theta_4$ & $0.01$ & $0.08$ & $0.00$ & $0.06_{0.07}$ & $0.97$ & $1.63$ & $0.00$ & $0.08$ & $0.98$  \\
& $\theta_5$ & $0.02$ & $0.10$ & $0.01$ & $0.09_{0.07}$ & $0.95$ & $1.39$ & $0.02$ & $0.10$ & $1.10$  \\
& $\theta_6$ & $0.02$ & $0.10$ & $0.00$ & $0.09_{0.09}$ & $0.95$ & $1.38$ & $0.02$ & $0.10$ & $1.05$  \\
& $\theta_7$ & $0.00$ & $0.07$ & $0.00$ & $0.06_{0.07}$ & $0.96$ & $1.50$ & $0.00$ & $0.08$ & $0.94$  \\
& $\theta_8$ & $0.01$ & $0.08$ & $0.01$ & $0.06_{0.07}$ & $0.96$ & $1.56$ & $0.01$ & $0.08$ & $0.95$  \\
& $\theta_9$ & $0.00$ & $0.08$ & $0.00$ & $0.06_{0.07}$ & $0.97$ & $1.60$ & $0.00$ & $0.08$ & $0.96$  \\
& $\theta_{10}$ & $0.00$ & $0.08$ & $0.00$ & $0.06_{0.06}$ & $0.93$ & $1.50$ & $0.00$ & $0.08$ & $0.96$  \\
  \hline  
\end{tabular}
\end{table}

\newpage

\section{Simulation Results: Accuracy Parameters}\label{sec:add-sim}

We present additional simulation results to complement our numerical studies in Section \ref{sec-sim} in the main text.  Figures \ref{figure:bias:RE-add} and \ref{figure:RE-add} show the percent bias and relative efficiency of the SL, SSL, and DR estimators of the accuracy measures with and without CV for bias correction with $K=3$ and all combinations of $S$ and $n_s$. 


\begin{figure}[htpb!]
\centering
\caption{Percent biases of the apparent (AP), CV, and ensemble cross validation (eCV) estimators of the Brier score and overall misclassification rate (OMR) for the SL, SSL and DR estimation methods under (i) ($\Msc_{\mbox{\tiny correct}}$, $\Isc_{\mbox{\tiny correct}}$), (ii)  ($\Msc_{\mbox{\tiny incorrect}}$, $\Isc_{\mbox{\tiny correct}}$), and (iii)  ($\Msc_{\mbox{\tiny incorrect}}$, $\Isc_{\mbox{\tiny incorrect}}$) with $K=3$ folds for CV.}
\label{figure:bias:RE-add}
\caption*{(i) ($\Msc_{\mbox{\tiny correct}}$, $\Isc_{\mbox{\tiny correct}}$).}
\begin{minipage}{1\textwidth}
  \centering
\mbox{
  \includegraphics[width=0.38\textwidth]{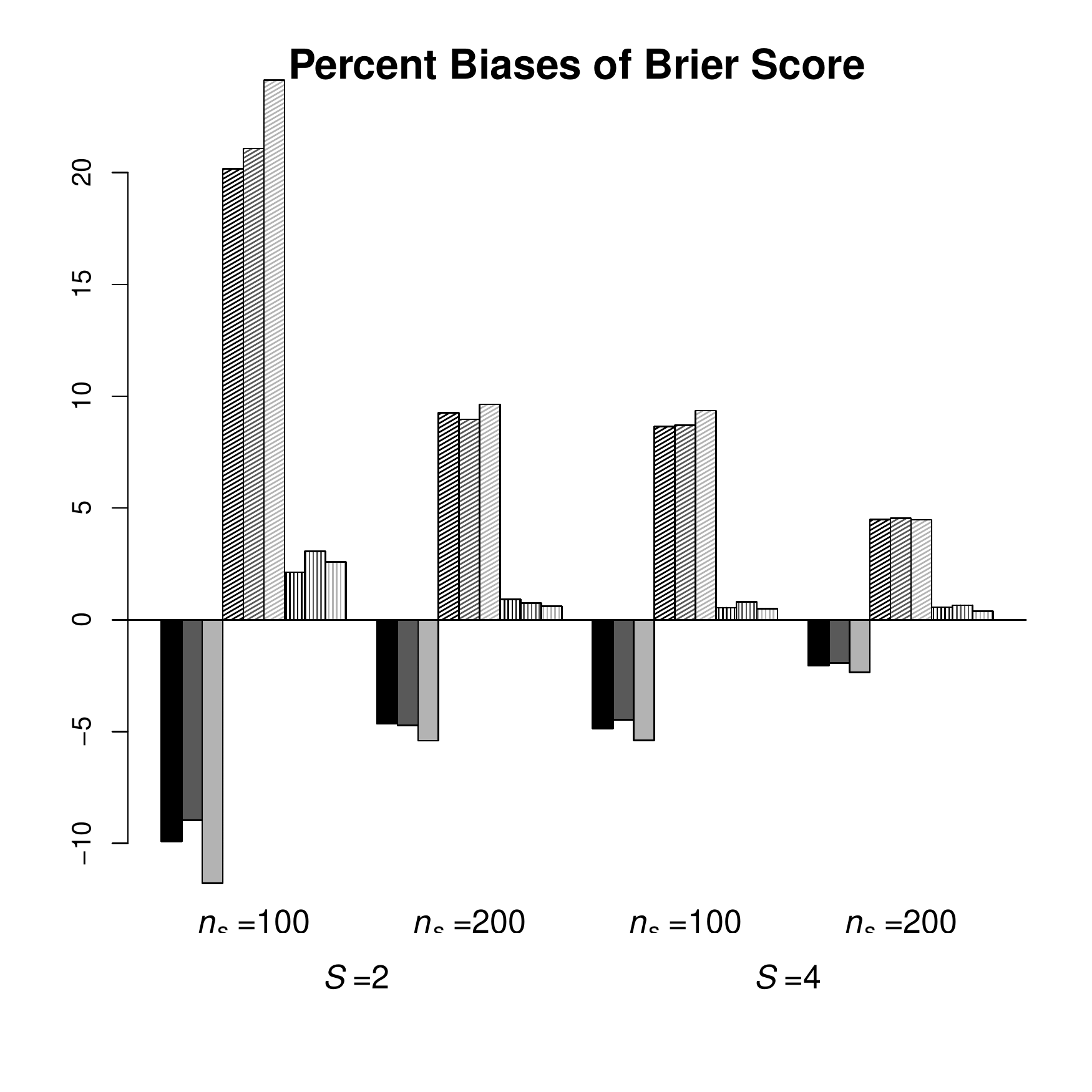}
  \includegraphics[width=0.38\textwidth]{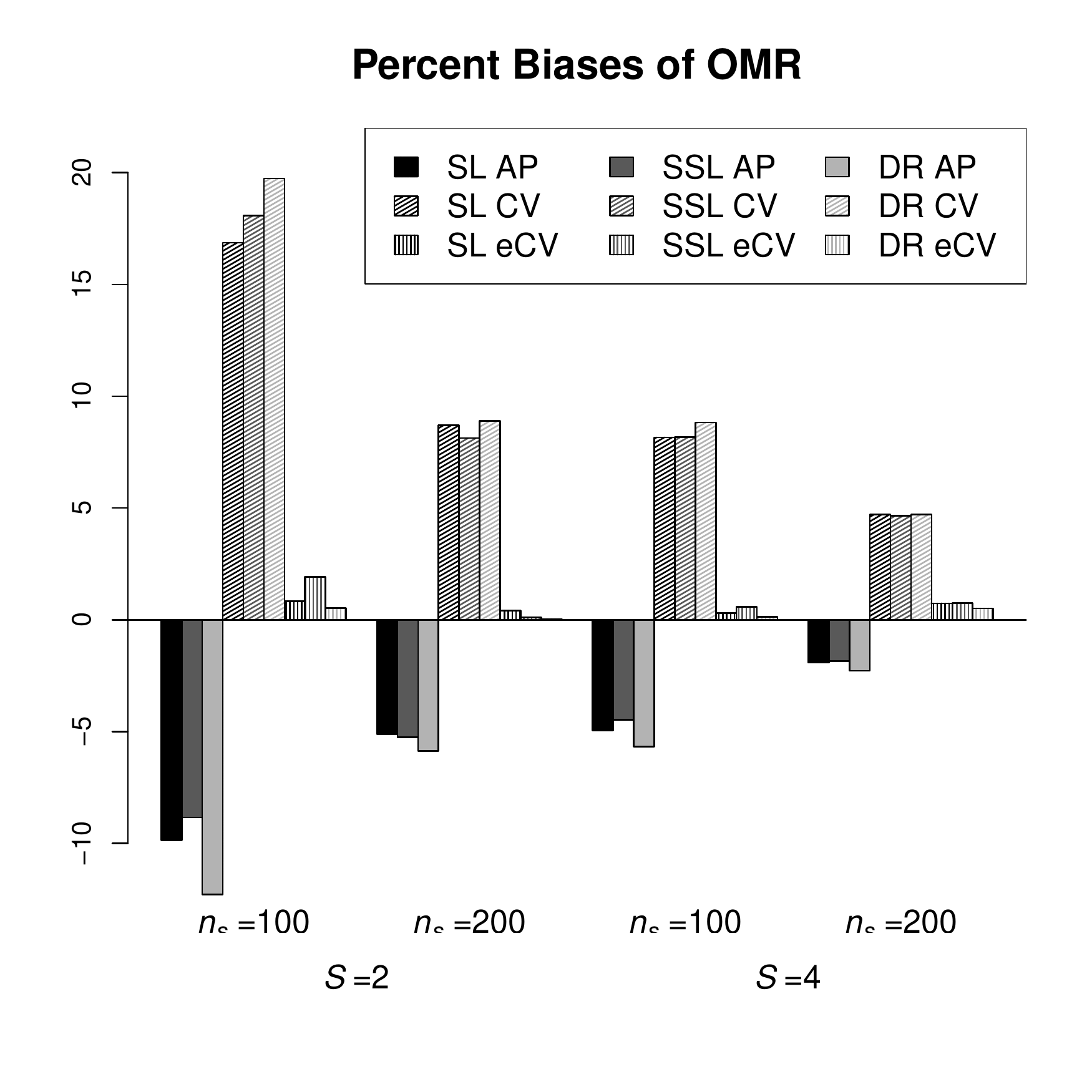}}
\end{minipage}
\caption*{(ii) ($\Msc_{\mbox{\tiny incorrect}}$, $\Isc_{\mbox{\tiny correct}}$).}
\begin{minipage}{1\textwidth}\centering
\mbox{
  \includegraphics[width=0.38\textwidth]{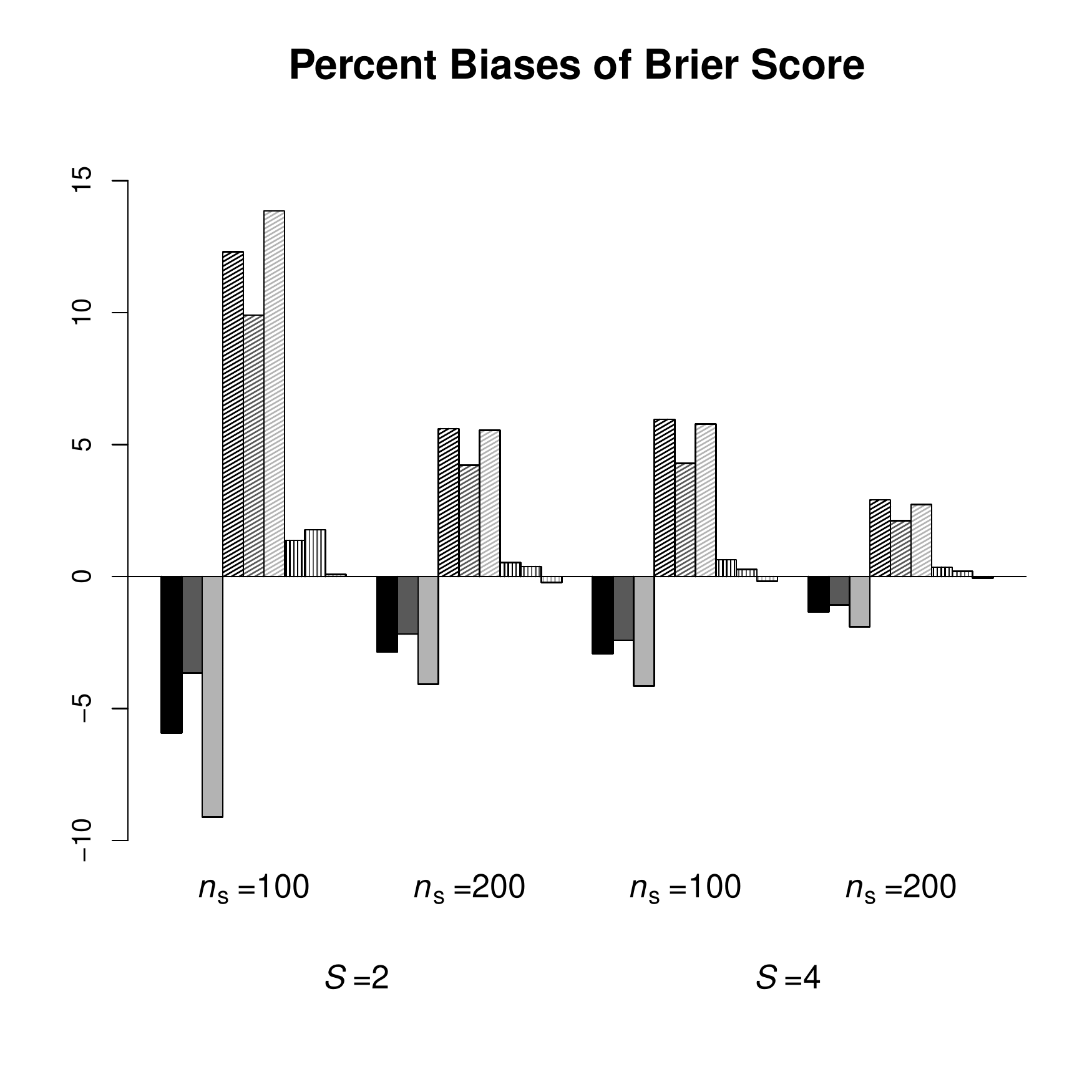}
  \includegraphics[width=0.38\textwidth]{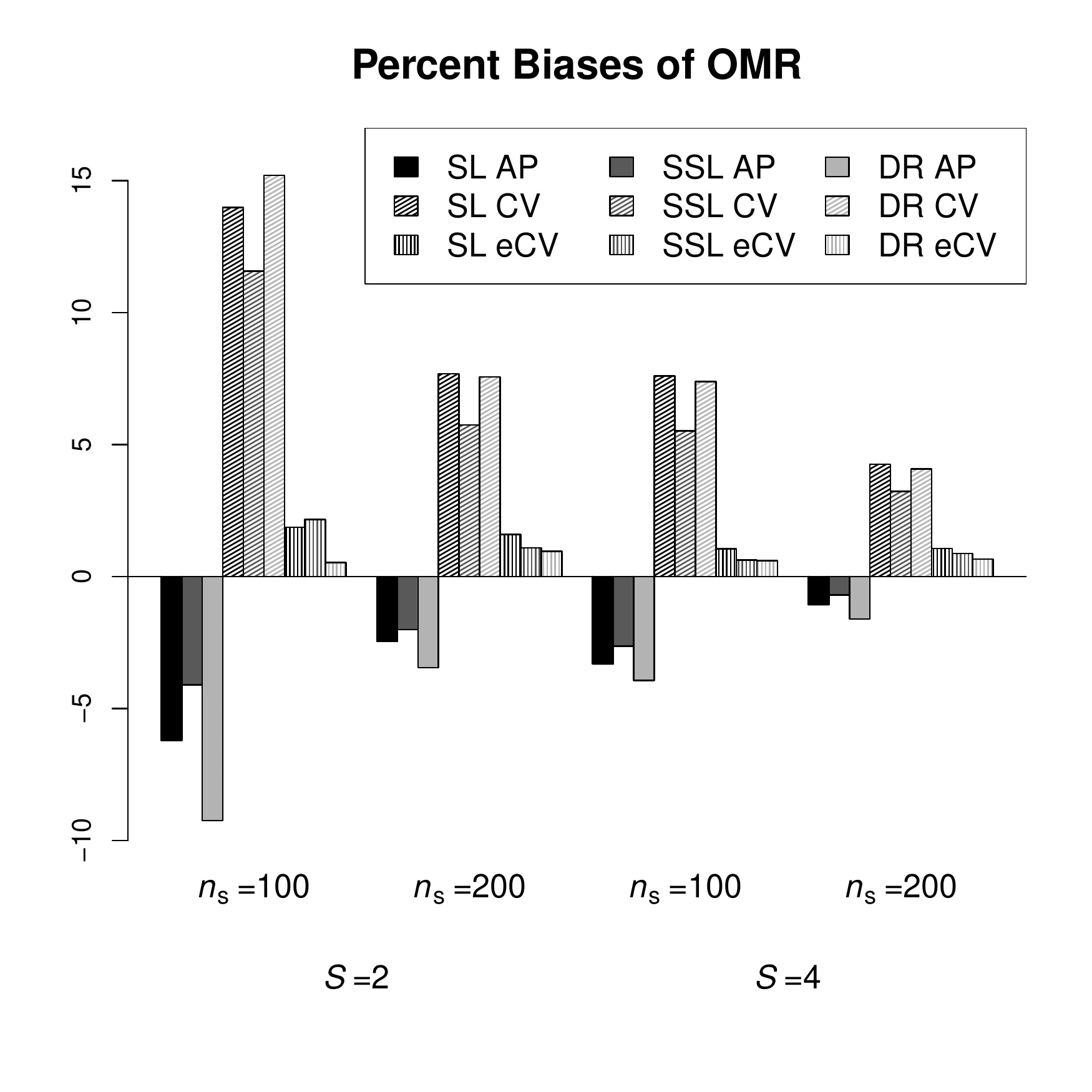}}
\end{minipage}
\caption*{(iii) ($\Msc_{\mbox{\tiny incorrect}}$, $\Isc_{\mbox{\tiny incorrect}}$).}
\begin{minipage}{1\textwidth}\centering
\mbox{
  \includegraphics[width=0.38\textwidth]{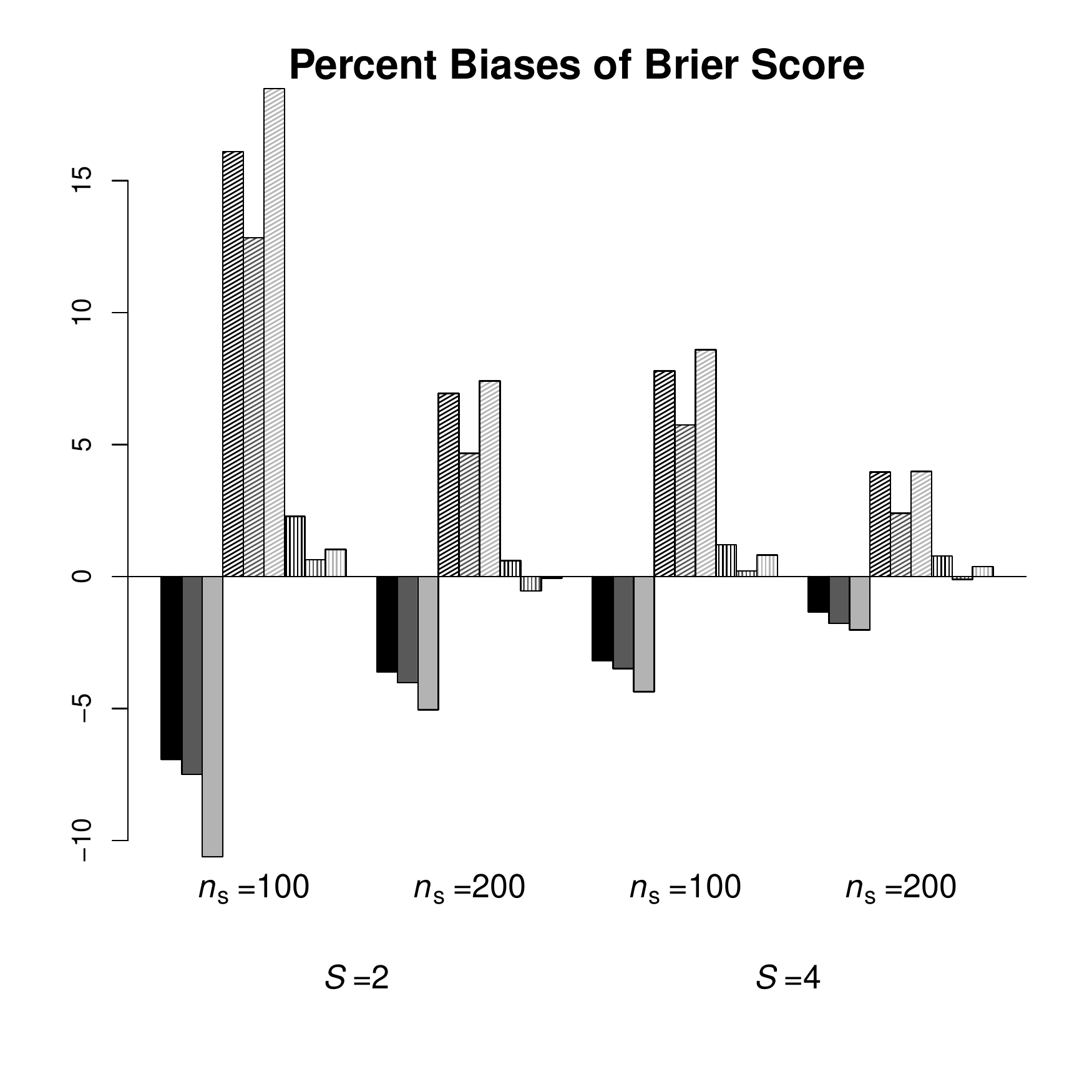}
  \includegraphics[width=0.38\textwidth]{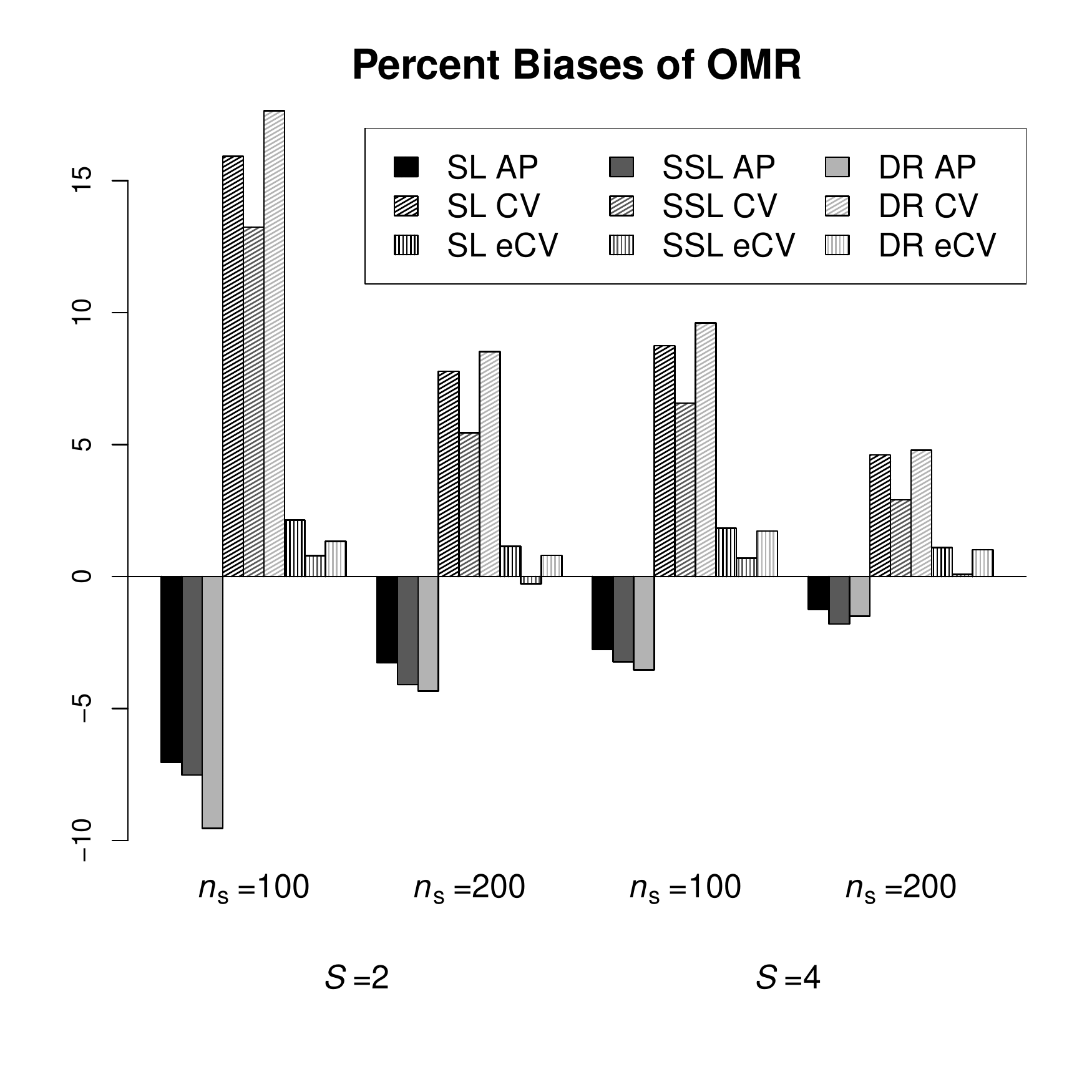}}
\end{minipage}
\end{figure}

\newpage

\begin{figure}[htpb!]
\centering
\caption{Relative efficiency (RE) of  $\Dhat\subSSL^{\omega}$ (SSL) and $\Dhat\subDR^{\omega}$ (DR) compared to $\Dhat\subSL^{\omega}$ for the Brier Score (BS) and OMR under (i) ($\Msc_{\mbox{\tiny correct}}$, $\Isc_{\mbox{\tiny correct}}$), (ii) ($\Msc_{\mbox{\tiny incorrect}}$, $\Isc_{\mbox{\tiny correct}}$); and (iii) ($\Msc_{\mbox{\tiny incorrect}}$, $\Isc_{\mbox{\tiny incorrect}}$) with $K=3$ folds for CV.}
\label{figure:RE-add}
  \centering
  \includegraphics[scale = 0.46]  {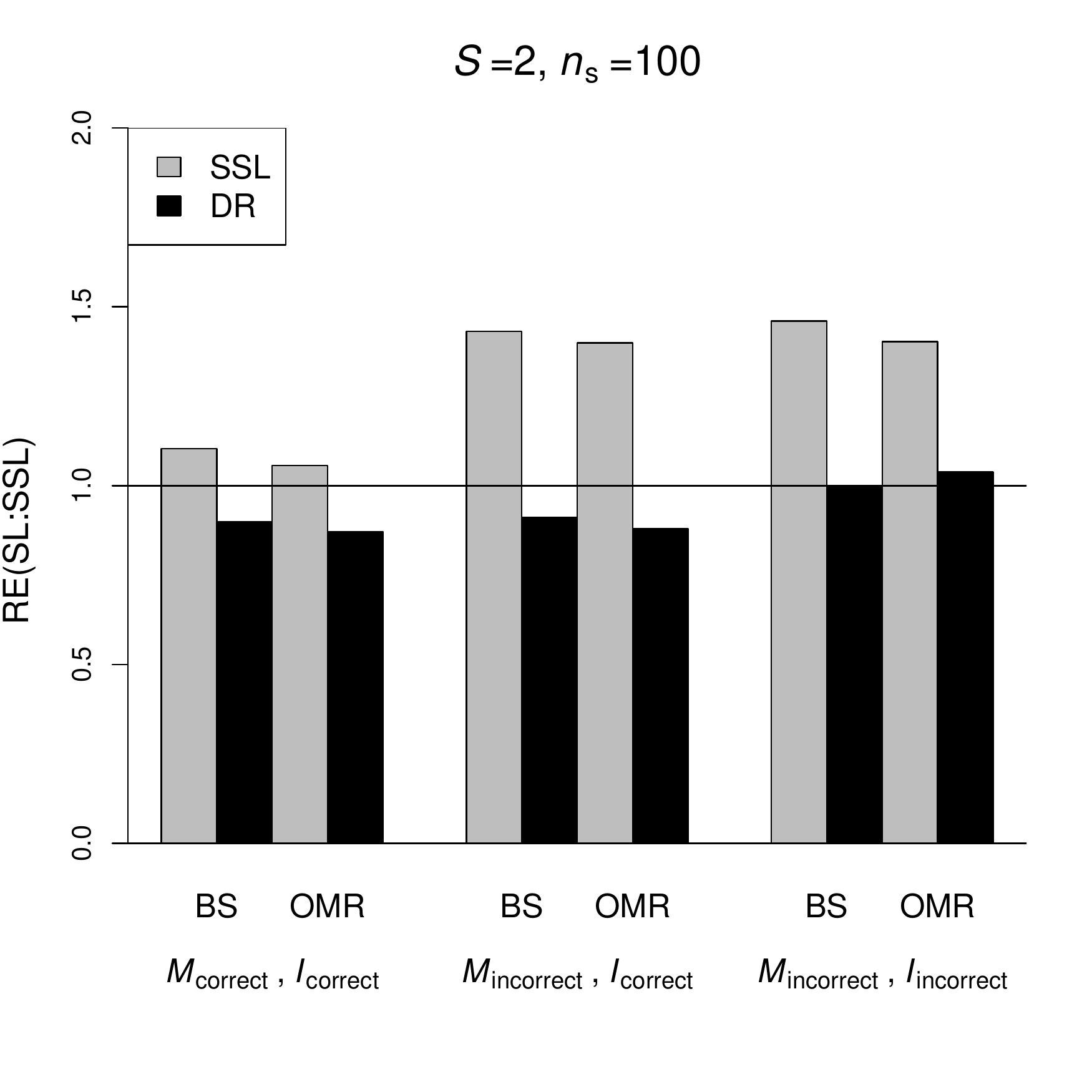}
  \includegraphics[scale = 0.46]  {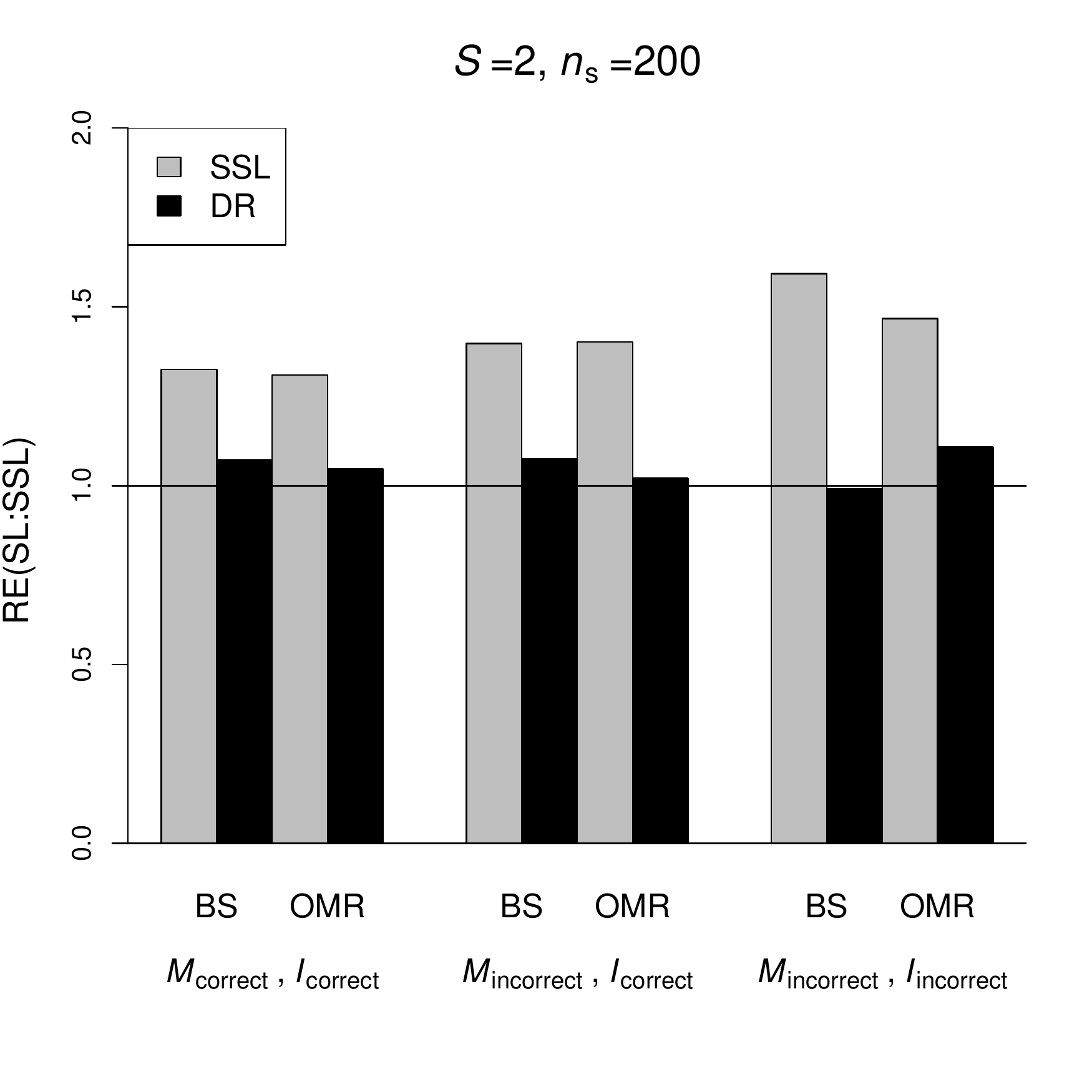}
    \includegraphics[scale = 0.46]  {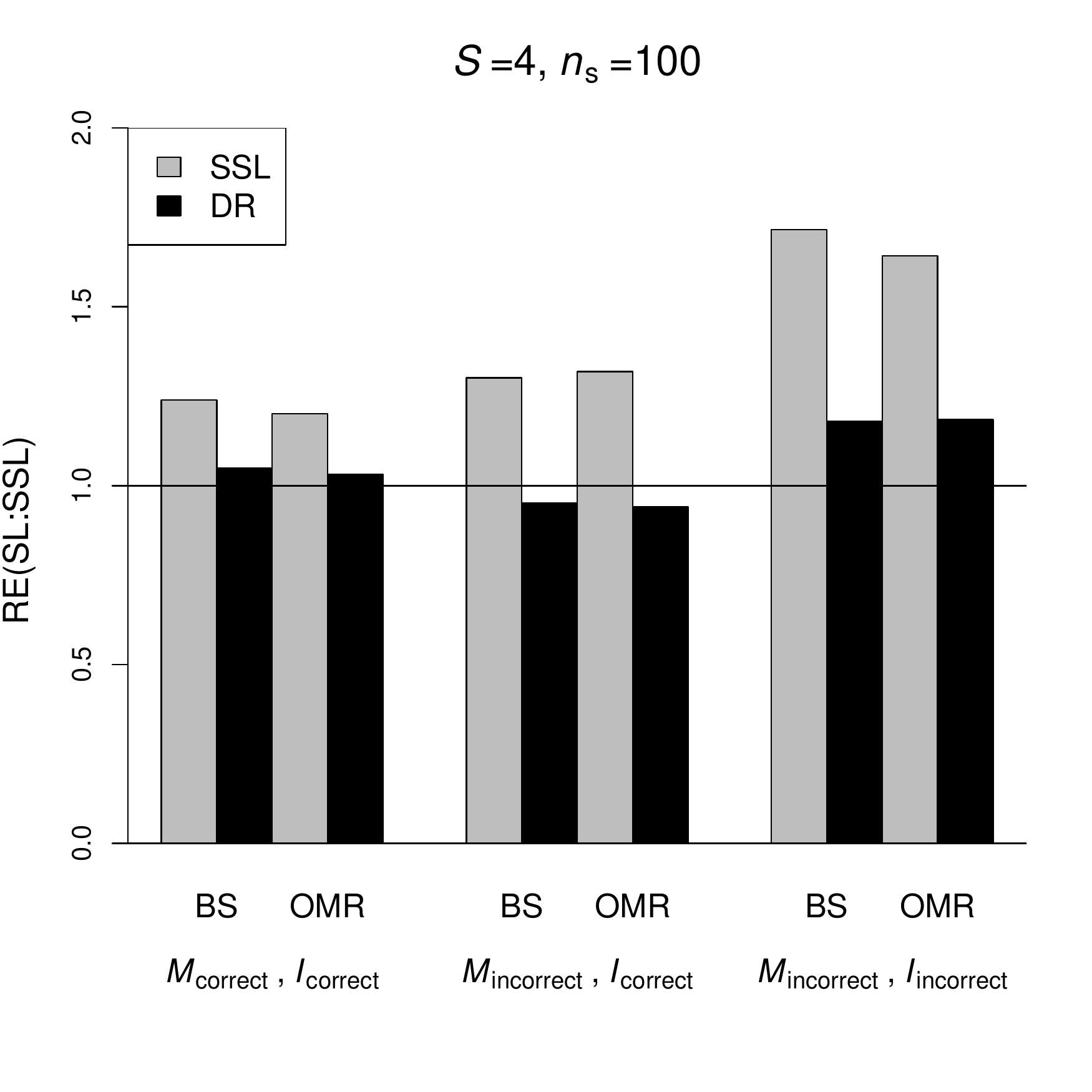}
  \includegraphics[scale = 0.46]  {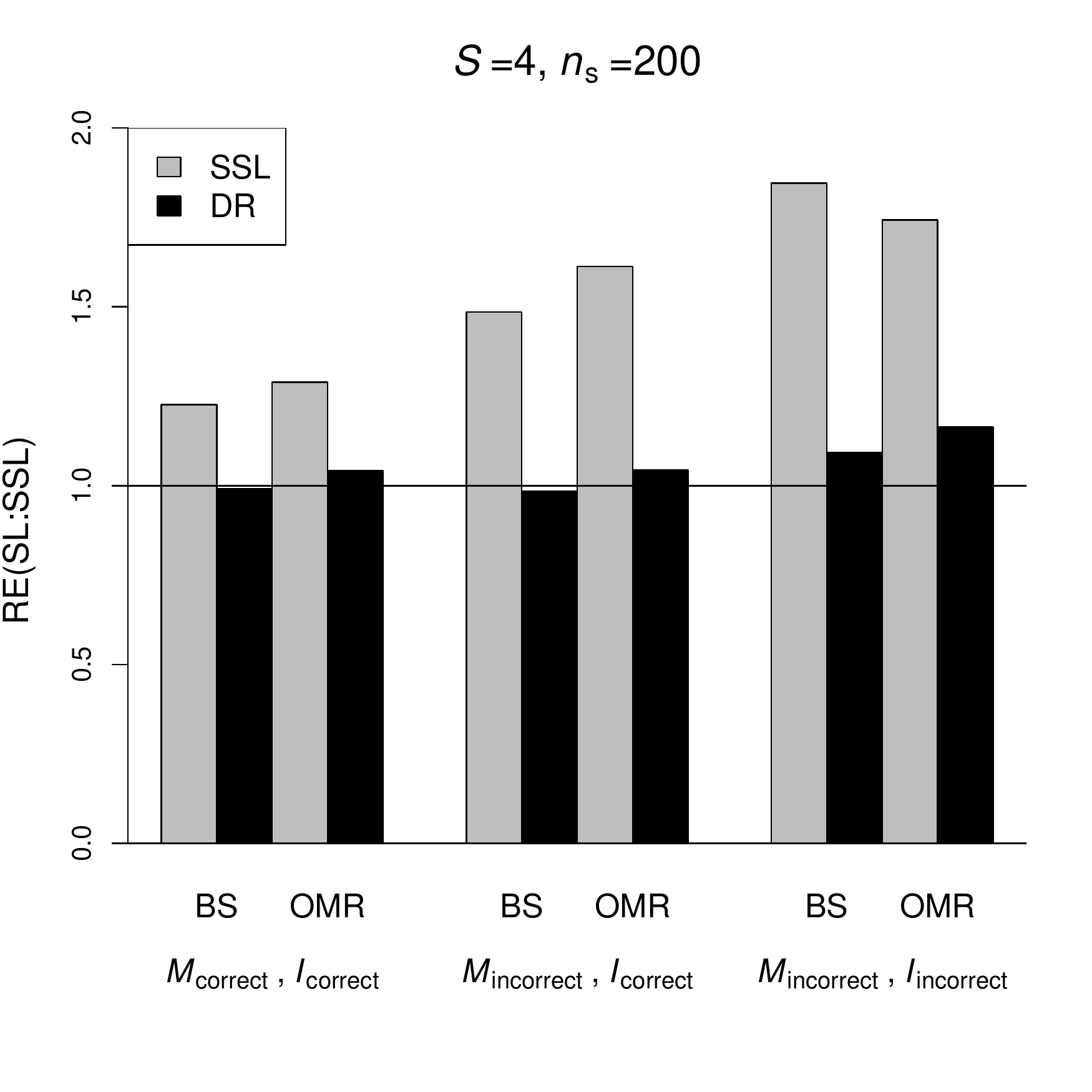}
\end{figure}

\newpage

\section{Simulation Results: Intrinsic Efficient Estimators}\label{sec:sup:intri}
{\red 
The following numerical studies compare the intrinsic efficient SS estimators from \ref{sec:asym-intrinsic} and Appendix \ref{sec:app:intri:all} of the main paper with our primary SS proposals. Here, we consider two settings for data generation. In both settings, we let $S=2$, $p=3$, $n_s=200$, $(x_1,x_2,z)\trans\sim N(\bzero,\Cbb')$ where $\Cbb_{kl}' = (0.4)^{|k-l|}$, and $\Ssc \in \{1,2\}$ with $\Ssc = 1 + I(z\geq 1)$. We generate $y$ as 
\begin{enumerate}[(a)]
\item  $y=I\left(2x_1-2x_2+5\{I(\Ssc=1)-I(\Ssc=2)\}x_1x_2 +\epsilon\sublogit > 0\right).$

\item $y=I\left(\Ssc\{x_1-x_2\}+1.5x_1x_2 +\epsilon\sublogit > 0\right).$

\end{enumerate}
We choose the imputation basis $\bPhi$ as $(1,x_1,x_2,x_1x_2)\trans$ under both settings and set the tuning parameters for the ridge penalty to be $0$. For simplicity, we study and compare the apparent estimator of the accuracy measure from the two approaches. For the intrinsic efficient approach, we set the direction $\boldsymbol e$ as $(1,0,0)\trans$, $(0,1,0)\trans$ and $(0,0,1)\trans$ to estimate each element of $\btheta$ separately. Note that in the intrinsic efficient approach, the objective function for estimating $\bgamma$ is non-convex. We thus carefully design the optimization procedure for estimating $\bgamma$ to ensure training stability. Specifically, we use Newton's method initialized with the fitted imputation coefficients of our proposed SSL approach. In each iteration, our method monitors the descent of the squared loss and stops updating $\bgamma$ upon convergence up to a specified tolerance level. This approach performed well in the two settings considered as we found no evidence of instability due to the non-convexity across simulations.

Table \ref{tab:app:intri} presents the bias, ESE and RE (to the SL estimators) of the SL, proposed SSL and intrinsic efficient SSL approaches under Settings (a) and (b). 
Under Setting (a), the intrinsic efficient estimators achieve improved efficiency when compared to our proposed estimators, with the average relative efficiency of the latter to the former being $1.35$ for $\btheta$, and $1.17$ for $D$. We observed that the imputation coefficients $\bgamma$ have different limiting values in the two approaches. As an example, $\bar\bgamma$ of the proposed SSL approach is $(0.1,0.6,-1.2,1.0)\trans$, while the limit of $\widehat\theta_{2,\intri}$ is $(-0.3,1.0,-2.7,4.3)\trans$. Consistent with the results of Section \ref{sec:asym-intrinsic}, these differences are expected to lead to improved efficiency of the intrinsic efficient approach. Under Setting (b), imputation parameters of the two approaches are similar, and $\bthetahat\subintri,D\subintri$ have very similar performance to the proposed SSL estimators. Here, the proposed estimators have slightly smaller ESE since the variances of its sample weights for estimating $\bgamma$ are smaller than those of the intrinsic efficient approach.

}

\begin{table}[!htbp]
\centering
 \caption{{\red Bias and empirical SE (ESE) of the SL, proposed SSL, and the intrinsic efficient (IE) SSL approaches in estimating $\btheta$ and $D$ under Settings (a) and (b) introduced in Section \ref{sec:sup:intri}. Shown also is the relative efficiency (RE) of the proposed and the intrinsic efficient SSL estimators to the SL estimator.}}
    \label{tab:app:intri}
      \centering
      \centerline{Setting (a)}
      \vspace{0.1in}
       \begin{tabular}{cc|rr|rrr|rrr}
       
       \hline
      & \multicolumn{1}{r}{} & \multicolumn{2}{|c|}{SL}& \multicolumn{3}{c|}{proposed SSL} & \multicolumn{3}{c}{IE SSL} \\ 
  \hline
 && Bias & ESE & Bias & ESE & RE & Bias & ESE & RE\\ 
  \hline
  \hline
& $\theta_0$ & 0.006 & 0.138 & 0.008 & 0.122 & 1.27 & 0.003 & 0.109 & 1.54 \\
& $\theta_1$ & 0.025 & 0.183 & 0.022 & 0.152 & 1.52 & 0.010 & 0.122 & 2.21 \\
& $\theta_2$ & 0.016 & 0.197 & 0.015 & 0.170 & 1.37 & 0.011 & 0.145 & 1.90 \\
\hline
  \hline
& Brier Score & 0.002 & 0.013 & 0.001 & 0.012 & 1.21 & 0.001 & 0.011  & 1.38 \\
& OMR & 0.002 & 0.028 & 0.001 & 0.025 & 1.29 & 0.001 & 0.022  & 1.55\\ 
  \hline
  \hline
\end{tabular}\vspace{0.3in}

      \centerline{Setting (b)}
      \vspace{0.1in}
       \begin{tabular}{cc|rr|rrr|rrr}
       
       \hline
      & \multicolumn{1}{r}{} & \multicolumn{2}{|c|}{SL}& \multicolumn{3}{c|}{proposed SSL} & \multicolumn{3}{c}{IE SSL} \\ 
  \hline
 && Bias & ESE & Bias & ESE & RE & Bias & ESE & RE\\ 
  \hline
  \hline
(iii) & $\theta_0$ & 0.000 & 0.148 & 0.001 & 0.134 & 1.20 & 0.001 & 0.134 & 1.20 \\
& $\theta_1$ & 0.023 & 0.187 & 0.018 & 0.161 & 1.34 & 0.015 & 0.164 & 1.31 \\
& $\theta_2$ & 0.022 & 0.187 & 0.018 & 0.167 & 1.28 & 0.015 & 0.169 & 1.24 \\
\hline
  \hline
& Brier Score & 0.001 & 0.013 & 0.002 & 0.012 & 1.08 & 0.001 & 0.012  & 1.07 \\
& OMR & 0.002 & 0.030 & 0.002 & 0.027 & 1.21 & 0.002 & 0.027  & 1.21\\ 
  \hline
  \hline
\end{tabular}

\end{table}

\newpage

\section{Simulation Results: Stratified Sampling}\label{sec:strata}

We conducted numerical studies to illustrate the advantage of stratified sampling relative to uniform random sampling as well as the importance of accounting for the sampling mechanism via weighted estimation. Mimicking our real example in Section \ref{sec-data}, we consider settings of outcome model misspecification where the risk of $y$ differs across the sampling groups. We again let $p=10$ and generate $\btheta$ and $\bx$ according to the mechanism described in Section \ref{sec-sim}. We let $S=2$ and simulate $\Ssc = 1 + I(x_1+x_2+\delta\geq 1.5)$ where $\delta \sim N(0,1)$. We consider two settings of outcome and imputation models:
\begin{enumerate}[(i)]
\item[] \hspace{-.35in} (I) ($\Msctilde_{\mbox{\tiny incorrect}}$, $\Isctilde_{\mbox{\tiny correct}}$):  $y = I[  \{0.8\mu_1(\bx)-5\}^{I(\Ssc=1)}\mu_1(\bx)^{I(\Ssc=2)} + \epsilon\sublogit > 1]$ with 
\[
\mu_1(\bx)=\btheta \trans  \bx+0.5(x_1x_2+x_1x_5-x_2x_6) \quad \mbox{and}\quad \bPhi=(1,\Sbb,\bx,\bv_2, \Sbb\text{:}\bx, \Sbb\text{:}\bv_2)\trans; 
\]
\item[] \hspace{-.35in} (II) ($\Msctilde_{\mbox{\tiny incorrect}}$, $\Isctilde_{\mbox{\tiny incorrect}}$): 
$y = I [ \{0.8\mu_2(\bx)-5\}^{I(\Ssc=1)}\mu_2(\bx)^{I(\Ssc=2)}+   e^{-2-3x_4-3x_6}\epsilon\subextreme > 1]$ with  
$$\mu_2(\bx)=\btheta \trans  \bx  + x_1^2 + x_3^2 \quad \mbox{and}\quad \bPhi=(1,\Sbb,\bx,\bv_1)\trans.
$$
\end{enumerate}
{\red 
In both settings, we considered $n = 200$ or $400$, and $N = 20,000$. The stratification variable is generated so that $P(y = 1)$ is much lower in the stratum with $\Ssc=1$ compared to the stratum with $\Ssc=2$. To evaluate the efficiency gain offered by stratified sampling relative to uniform random sampling, we simulate two sampling strategies for the labeled data: (1) uniform random sampling of $n$ subjects, and (2) stratified sampling of $n/2$ subjects from each stratum.  Under (1), the labeled samples are from same population as the unlabeled samples, while under (2), the labeled samples consist of $n/2$ observations from each stratum and have a different distribution from the unlabeled samples, provided the proportion of $\Ssc=1$ is not $0.5$ in the original population, which is the case in both settings. 

With the data under stratified sampling, we obtain the SL and SSL estimators following the same procedures used in Section \ref{sec-sim}, which involves weighting the labeled samples by $\what_i$ according to the inverse of their sampling probability. We also include a naive unweighted SL estimator with $\what_i$ replaced by $1$ in the SL estimation procedures. With the data under the uniform random sampling scheme, we follow the same estimation procedures to obtain the SL and SSL estimators with the weights of labeled samples naturally set as $1$.  We compare (i) SSL and SL under uniform random and stratified random sampling to evaluate the value of the SSL approach and (ii) uniform random and stratified random sampling for both SSL and SL to evaluate the value of the sampling design.


}


{\red 
Table \ref{table:random:stratified} presents relative efficiency (RE) of the SSL versus the SL estimators (SSL vs SL) under either uniform random sampling ($\Ubb$) or stratified sampling ($\Sbb$); RE of the stratified sampling versus the uniform sampling ($\Sbb$ vs $\Ubb$) for either the SL or the SSL estimators. Under both random and stratified sample schemes, the SSL estimators are significantly more efficient than the SL estimators, with the corresponding RE larger than $2$ on nearly all parameters. The stratified sampling strategy is generally more efficient than uniform random sampling for both SL and SSL estimation under models (I) and (II), with the average of RE $>1.3$ under (I) and $>1.6$ under (II). These results coincide with our analysis presented in Section \ref{sec:thm:opt:all} regarding optimal allocation. Taking the supervised estimator of $D(\bthetabar)$ under ($\Msctilde_{\mbox{\tiny incorrect}}$, $\Isctilde_{\mbox{\tiny correct}}$) as an example, we have $S=2$ stratum with $\sigma_1^2 = 0.023$, $\sigma_2^2 = 0.142$, $\rho_1 = 0.69$, and $\rho_2 = 0.31$ (see Section \ref{sec:thm:opt:all} for definition of these quantities). Following the arguments of Section \ref{sec:thm:opt:all}, the optimal allocation of $n_s$ is proportional to $\rho_s\sigma_s$ which yields $n_1 = 0.47n$ and $n_2 =0.53n$.  This coincides with our choice of roughly equal allocation of $n$ across the two stratum. 

In Table \ref{tab:bias:uSL}, we report the bias associated with the unweighted naive estimator under stratified sampling and outcome model misspecification, and compare with the SL and SSL estimators that properly account for the sampling. Our results highlight that weighting is crucial to ensure the validity of the estimators for both the regression parameter and the accuracy parameters.

}

\begin{table}[!htpb]
    \caption{\red Relative efficiency with respect to mean squared error of the semi-supervised (SSL) estimators versus the SL estimators (SSL vs SL) under either uniform random sampling ($\Ubb$) or stratified sampling ($\Sbb$); relative efficiency of the stratified sampling versus the uniform sampling ($\Sbb$ vs $\Ubb$) for either the SL estimators or the SSL estimators under $(\Msctilde_{\mbox{\tiny incorrect}}$, $\Isctilde_{\mbox{\tiny correct}})$ and $(\Msctilde_{\mbox{\tiny incorrect}}$, $\Isctilde_{\mbox{\tiny incorrect}})$, with $n = 200$ or $400$.}
    \label{table:random:stratified}
    \small
      \centering
      \centerline{$n=200$.}
      \vspace{0.1in}
       \begin{tabular}{c|cc|cc|cc|cc}
\hline
 \multicolumn{1}{r}{} & \multicolumn{4}{|c|}{($\Msctilde_{\mbox{\tiny incorrect}}$, $\Isctilde_{\mbox{\tiny correct}}$)} & \multicolumn{4}{c}{($\Msctilde_{\mbox{\tiny incorrect}}$, $\Isctilde_{\mbox{\tiny incorrect}}$)} \\  \cline{2-9}
&  \multicolumn{2}{c|}{SSL vs SL} & \multicolumn{2}{c|}{$\Sbb$ vs $\Ubb$} & \multicolumn{2}{c|}{SSL vs SL} & \multicolumn{2}{c}{$\Sbb$ vs $\Ubb$} \\   \hline
& $\Ubb$ & $\Sbb$ & SL & SSL & $\Ubb$ & $\Sbb$ & SL & SSL\\  \hline  \hline
$\theta_0$ & 5.52 & 3.66 &1.64 &1.09 &9.18  &8.81 &1.71 &1.64 \\
$\theta_1$ & 4.70 & 4.43 &1.31 &1.23 &9.27 &11.05 &1.39 &1.66 \\
$\theta_2$ & 4.13 & 4.04 &1.32 &1.29 &7.45  &7.97 &1.68 &1.80 \\
$\theta_3$ & 5.49 & 4.32 &1.65 &1.30 &6.74  &6.55 &1.58 &1.54  \\
$\theta_4$ & 4.52 & 4.97 &1.20 &1.32 &4.78  &5.48 &1.32 &1.51 \\
$\theta_5$ & 3.94 & 5.07 &1.43 &1.84 &6.76  &9.02 &1.11 &1.48 \\
$\theta_6$ & 4.15 & 4.53 &1.43 &1.56 &6.24  &8.44 &1.08 &1.47 \\
$\theta_7$ & 3.99 & 4.00 &1.58 &1.59 &4.14  &5.65 &1.11 &1.52 \\
$\theta_8$ & 4.60 & 4.92 &1.33 &1.42 &5.01  &6.01 &1.46 &1.75 \\
$\theta_9$ & 3.99 & 4.61 &1.26 &1.45 &5.34  &6.57 &1.46 &1.80 \\
$\theta_{10}$ & 3.71 & 4.69 &1.30 &1.64 &4.44  &5.96 &1.41 &1.89 \\ \hline\hline
Brier score & 1.95 &2.35 &1.29 &1.56 &3.31 &3.51 &1.42 &1.51 \\
OMR & 1.79 &1.91 &1.34 &1.44 &3.24 &3.27 &1.46 &1.48 \\ \hline\hline
\end{tabular}\vspace{0.3in}

      \centerline{$n=400$.}
      \vspace{0.1in}
       \begin{tabular}{c|cc|cc|cc|cc}
\hline
 \multicolumn{1}{r}{} & \multicolumn{4}{|c|}{($\Msctilde_{\mbox{\tiny incorrect}}$, $\Isctilde_{\mbox{\tiny correct}}$)} & \multicolumn{4}{c}{($\Msctilde_{\mbox{\tiny incorrect}}$, $\Isctilde_{\mbox{\tiny incorrect}}$)} \\  \cline{2-9}
&  \multicolumn{2}{c|}{SSL vs SL} & \multicolumn{2}{c|}{$\Sbb$ vs $\Ubb$} & \multicolumn{2}{c|}{SSL vs SL} & \multicolumn{2}{c}{$\Sbb$ vs $\Ubb$} \\   \hline
& $\Ubb$ & $\Sbb$ & SL & SSL & $\Ubb$ & $\Sbb$ & SL & SSL\\  \hline  \hline
$\theta_0$ &2.29 &1.74 &1.55 &1.18 & 4.72 &4.69 &1.60 &1.59 \\
$\theta_1$ &2.20 &1.80 &1.35 &1.10 & 5.00 &5.64 &1.26 &1.42 \\
$\theta_2$ &1.83 &2.05 &1.21 &1.36 & 5.35 &5.91 &1.45 &1.61 \\
$\theta_3$ &2.41 &3.09 &1.16 &1.48 & 4.77 &6.88 &1.18 &1.70  \\
$\theta_4$ &2.07 &3.06 &1.02 &1.51 & 3.46 &5.29 &0.98 &1.50 \\
$\theta_5$ &2.32 &2.58 &1.17 &1.30 & 4.53 &6.29 &0.93 &1.29 \\
$\theta_6$ &2.71 &2.92 &1.37 &1.47 & 4.66 &6.38 &1.07 &1.47 \\
$\theta_7$ &2.44 &2.78 &1.49 &1.69 & 4.08 &6.11 &1.12 &1.67 \\
$\theta_8$ &2.43 &2.86 &1.59 &1.87 & 4.01 &5.82 &1.32 &1.91 \\
$\theta_9$ &2.49 &2.91 &1.22 &1.42 & 4.08 &6.49 &1.10 &1.75 \\
$\theta_{10}$ &2.36 &3.12 &1.22 &1.61 & 3.61 &5.48 &1.41 &2.14 \\ \hline\hline
Brier score &2.00 &2.13 &1.37 &1.45 & 4.05 &4.60 &1.39 &1.58 \\
OMR &1.56 &1.55 &1.39 &1.38 & 3.58 &3.99 &1.38 &1.54 \\ \hline\hline
\end{tabular}



\end{table}

\newpage

\begin{table}[!htpb]
\caption{{\red Bias of the SL, SSL and unweighted SL estimators obtained under stratified sampling ($\Sbb$) for models (I) ($\Msctilde_{\mbox{\tiny incorrect}}$, $\Isctilde_{\mbox{\tiny correct}}$) and (II) ($\Msctilde_{\mbox{\tiny incorrect}}$, $\Isctilde_{\mbox{\tiny incorrect}}$) when $n=400$.}}
\label{tab:bias:uSL}
\centering
\begin{tabular}{c|rrr|rrr}
       \hline
 \multicolumn{1}{r}{} & \multicolumn{3}{|c|}{($\Msctilde_{\mbox{\tiny incorrect}}$, $\Isctilde_{\mbox{\tiny correct}}$)} & \multicolumn{3}{c}{($\Msctilde_{\mbox{\tiny incorrect}}$, $\Isctilde_{\mbox{\tiny incorrect}}$)} \\ 
  \cline{2-7}
& SL & SSL & uSL  & SL & SSL & uSL  \\  \hline   
$\theta_0$ &0.28 &0.23 &0.33 & 0.29 &0.04 &0.23 \\
$\theta_1$ &0.14 &0.12 &0.01 & 0.14 &0.03 &0.13 \\
$\theta_2$ &0.16 &0.12 &0.04 & 0.10 &0.01 &0.03 \\
$\theta_3$ &0.03 &0.00 &0.06 & 0.00 &0.01 &0.01 \\
$\theta_4$ &0.03 &0.01 &0.06 & 0.01 &0.01 &0.02 \\
$\theta_5$ &0.02 &0.02 &0.06 & 0.07 &0.01 &0.20  \\
$\theta_6$ &0.03 &0.02 &0.07 & 0.07 &0.01 &0.21  \\
$\theta_7$ &0.00 &0.00 &0.00 & 0.01 &0.01 &0.01 \\
$\theta_8$ &0.01 &0.00 &0.01 & 0.00 &0.00 &0.00 \\
$\theta_9$ &0.01 &0.00 &0.01 & 0.01 &0.01 &0.01   \\
$\theta_{10}$ &0.00 &0.00 &0.00 & 0.00 &0.00 &0.00\\ \hline\hline
Brier Score &0.002 &0.002 &0.025 & 0.002 &0.001 &0.025  \\
OMR &0.003 &0.002 &0.036 & 0.002 &0.002 &0.036 \\ \hline\hline
\end{tabular}

\end{table}

\newpage
\section{Theoretical Analysis: Density Ratio Estimator}\label{sec:method-DR}
{\red

In this section, we introduce and study density ratio (DR) estimator, an alternative SSL approach proposed by \cite{kawakita2013semi,kawakita2014safe} that conceptually achieves a similar ``safe and efficient" property as the proposed SSL estimator stated in Theorem \ref{thm:1} and Theorem \ref{thm:2}. We compare our method with the DR method from a theoretical perspective to complement our numerical results in Sections \ref{sec-sim} and \ref{sec-data} of the main text.

\subsection{Outline of the DR approach}

Let $\bvarphi=\bvarphi(\bu)$ be some basis functions of $\bu$. To efficiently estimate $\btheta$ using the unlabelled data, one can first solve the DR estimating equation
\[
\frac{1}{N}\sum_{i = 1}^N \what_i \bvarphi_i \exp(\balpha\trans\bvarphi_i)-\frac{1}{N}\sum_{i = 1}^N\bvarphi_i + \lambda_n\supthree \balpha = \bzero,
\]
to obtain $\balphatilde$, where $\lambda_n\supthree=o(n^{-\frac{1}{2}})$ is a tuning parameter for training stability. Next, one can solve the equation for $\btheta$ weighted by the estimated DR:
\[
\frac{1}{N}\sum_{i = 1}^N \exp(\balphatilde\trans\bvarphi_i) \what_i \bx_i \{ y_i -g(\btheta \trans\bx_i) \} = \bzero,
\]
to obtain the DR estimator of the regression paramere, denoted as $\bthetahat\subDR$. Similarly, the DR estimator for $\Dbar$ can be obtained by $\Dhat\subDR=\Dhat\subDR(\bthetahat\subDR)$ where 
$$\Dhat\subDR(\btheta)=\Ninv\sum_{i = 1}^N \exp(\balphatilde\trans\bvarphi_i)\what_i d\{y_i, \yscr(\btheta\trans \bx_i)\}.$$

Note that ${\rm E}[\omega_if(\bu_i)]={\rm E}[f(\bu_i)]$ for any measurable function $f(\cdot)$ so that $\balphatilde\overset{p}{\rightarrow}\bzero$ under mild regularity conditions (see our justification of Proposition \ref{prop:1} for details). At first glance, it seems counter-intuitive that weighting samples with $\exp(\balphatilde\trans\bvarphi_i)$, which converges to $1$, can improve the estimation efficiency. However, this approach can be viewed as a variance reduction procedure via projection \citep{kawakita2013semi,kawakita2014safe}.

Motivated by \cite{kawakita2014safe}, we  next present the asymptotic properties of the DR estimators in Proposition \ref{prop:1} and provide formal justification in Section \ref{sec:asym-DR}. For any random vectors $\bb,\ba$, let $\proj(\bb\mid \ba)$ represent the linear projector of $\bb$ onto the space spanned $\ba$ (on the population of sampling $\bF_i$), i.e. $\proj(\bb\mid \ba)={\rm E}(\bb\ba\trans)({\rm E}\ba^{\otimes 2})^{-1}\ba$, and $\proj\supvert(\bb\mid \ba)=\bb-\proj(\bb\mid \ba)$.
\begin{prop}
Under Conditions \ref{cond:1}, \ref{cond:3}, \ref{cond:2}(C) and the assumption that ${\rm E}(\bvarphi^{\otimes 2})\succ 0$, we have $\widehat \btheta\subDR \overset{p}{\to} \bthetabar$,  $\widehat D\subDR \overset{p}{\to} \Dbar$, and
\begin{align*}
\nhalf(\widehat \btheta\subDR - \bthetabar) =& \nhalf\sum_{s=1}^S \rho_s \left\{ n_s^{-1} \sum_{i=1}^N V_iI(\Ssc_i = s) \proj\supvert(\be_{\SL i}\mid\bvarphi_i) \right\} + o_p(1);\\
\nhalf(\Dhat\subDR-\Dbar)=&\nhalf  \sum_{s= 1}^S \rho_s \left[\ninv_s \sum_{i = 1}^N V_i I(\Ssc_i = s) \proj\supvert\left(d(y_i, \Yscbar_i) - D(\bthetabar) + \dot{\bD}(\bthetabar)\trans \be_{\SL i}  \mid\bvarphi_i\right)\right]+o_p(1).
\end{align*}
In addition, $\nhalf(\widehat \btheta\subDR - \bthetabar)$ converges weakly to $N(\bzero,\bSigma\subDR)$ where
\[
\bSigma\subDR=\sum_{s=1}^S \rho_s^2 \rho_{1s}^{-1} {\rm E}\left[\{\proj\supvert(\be_{\SL i}\mid\bvarphi_i) \}^{\otimes 2} \mid \Ssc_i = s\right],
\]
and $\nhalf(\Dhat\subDR-\Dbar)$ converges weakly to $N(0,\sigma^2\subDR)$ with $\sigma^2\subDR$ defined in Section \ref{sec:asym-DR}.

\label{prop:1}
\end{prop}

Using these results, we comment on the DR estimators and compare them with our imputation-based SSL estimators in the following remarks.
\begin{remark}
Let the basis be formed as $\bvarphi_i=(I(\Ssc_i=1)\bvarphi_{1i}\trans,\ldots,I(\Ssc_i=S)\bvarphi_{Si}\trans)\trans$, where $\bvarphi_{si}$ is some basis function of $\bu$. Then for any measurable function $\bbf(\cdot)$,
\[
{\rm E}\left[\{\proj\supvert(\bbf(\bu_i)\mid\bvarphi_i) \}^{\otimes 2} \mid \Ssc_i = s\right]={\rm E}\left[\{\proj\supvert_s(\bbf(\bu_i)\mid\bvarphi_{si}) \}^{\otimes 2} \mid \Ssc_i = s\right],
\]
where $\proj\supvert_{s}(\ba\mid\bb)=\ba-\proj_{s}(\ba\mid\bb)$ and $\proj_{s}$ represents the linear projection operator on the population from strata $s$. This result, combined with Proposition \ref{prop:1}, implies that $\bSigma\subSL\succeq\bSigma\subDR$ and $\sigma\subSL^2\succeq\sigma\subDR^2$, i.e.\ the DR estimators have smaller or equal asymptotic variance than the SL estimators. A similar efficiency dominance property is achieved by our intrinsic efficient SSL estimators as discussed in Section \ref{sec:asym-intrinsic}, but is not readily achieved by the proposed SSL estimators when the imputation model is misspecified.
\label{rem:5.4.1}
\end{remark}

\begin{remark}
When the imputation model in our method is correctly specified, and $\be_{\SL i}$ and $d(y_i, \Yscbar_i) - D(\bthetabar)$ are in the linear space of $\bvarphi_i$, our SSL estimators and the DR estimators will be asymptotically equivalent, and both semiparametric efficient. Under model misspecification, one could not make a simple comparison on the asymptotic efficiency between our method and DR, and neither of them would always be able to dominate the other.
\label{rem:5.4.2}
\end{remark}

Though the DR estimators essentially achieve similar theoretical efficiency properties as our proposed estimators, we do find that with the same or similar choices on their basis function, our method performs much better than DR in all of our numerical studies (see Sections \ref{sec-sim}, \ref{sec-data} and the Supplement Material for details). In addition, we conducted simulation studies (results not shown here) to evaluate the DR estimator specified as in Remark \ref{rem:5.4.1}, with a larger basis than the DR estimator used in our main simulation studies. Although this specification guarantees variance reduction compared to SL theoretically, we find that its finite sample performance is slightly worse than the specification of the DR estimator used in our numerical studies.

The poor finite sample performance of DR is due to the fact that our imputation procedures involve $y$ and automatically select the desirable subset or direction of the basis $\bPhi$ for characterizing the conditional mean of $y$.  The DR method, on the other hand, projects the influence functions to the whole $\bvarphi$ directly without screening or distilling the basis with $y$, resulting in excessive overfitting in finite sample. This issue is pronounced for both methods when the basis functions $\bPhi$ and $\bvarphi$ are of high dimensionality. However, in this scenario, one could use sparse regression approaches such as lasso or adaptive lasso \citep{tibshirani1996regression,zou2006adaptive} to estimate the imputation model in our method and avoid overfitting. In contrast, one cannot impose sparse regularization when using DR since $y$ is not involved in constructing $\balphatilde$. As a consequence, the high dimensional $\bvarphi$ cannot be accommodated by the DR method in its standard form.

\subsection{Asymptotic Properties of $\bthetahat\subDR$ and $\Dhat\subDR$}\label{sec:asym-DR}
We first justify the consistency and provide the asymptotic expansion for the fitted density ratio coefficients $\balphatilde$. Note that
\[
{\rm E}\left[\frac{1}{N}\sum_{i = 1}^N w_i \bvarphi_i \exp(\bzero\trans\bvarphi_i)-\frac{1}{N}\sum_{i = 1}^N\bvarphi_i\right]= \bzero,
\]
and under the regularity Condition \ref{cond:1} and with ${\rm E}(\bvarphi^{\otimes2})\succ0$, we can follow a similar procedure as in Appendix \ref{sec: asym-thetaSSL} (i.e. using \cite{newey1994large,van2000asymptotic}) to show that $\balphatilde\overset{p}{\to}\bzero$ and 
\begin{equation}
\nhalf\balphatilde=\nhalf({\rm E}\bvarphi^{\otimes2})^{-1}\left(-\frac{1}{N}\sum_{i = 1}^N \what_i \bvarphi_i+\frac{1}{N}\sum_{i = 1}^N\bvarphi_i\right)+o_p(1)=O_p(1).
\label{equ:h:1}
\end{equation}
This result, combined with Condition \ref{cond:1}, implies that $\sup_{i}|\exp(\balphatilde\trans\bvarphi_i)-0|\overset{p}{\to}0$ and thus
\[
\frac{1}{N}\sum_{i = 1}^N \exp(\balphatilde\trans\bvarphi_i) \what_i \bx_i \{ y_i -g(\btheta \trans\bx_i) \}-\frac{1}{N}\sum_{i = 1}^N\what_i \bx_i \{ y_i -g(\btheta \trans\bx_i) \}=o_p(1),
\]
uniformly for $\btheta\in\bTheta$. Similar to Appendix \ref{sec: asym-thetaSL}, we can (again using \cite{newey1994large,van2000asymptotic}) show that $\bthetahat\subDR\overset{p}{\to}\bthetabar$ and then use equation (\ref{equ:h:1}) to derive the following expansion:
\begin{align*}
\nhalf(\widehat \btheta\subDR - \bthetabar) =&  \frac{\nhalf}{N}\sum_{i=1}^N  \what_i(1+\balphatilde\trans\bvarphi_i) \bA^{-1}\bx_i \{ y_i -g(\bthetabar \trans\bx_i) \} + o_p(1)\\
=&\nhalf\left[\frac{1}{N}\sum_{i=1}^N  \what_i\bz_i+\left(\frac{1}{N}\sum_{i = 1}^N\bvarphi_i-\frac{1}{N}\sum_{i = 1}^N \what_i \bvarphi_i\right)\trans({\rm E}\bvarphi^{\otimes2})^{-1}\left(\frac{1}{N}\sum_{i=1}^N  \what_i\bvarphi_i\bz_i\trans\right)\right]+ o_p(1),
\end{align*}
where $\bz_i=\bA^{-1}\bx_i \{ y_i -g(\btheta \trans\bx_i) \}$. Note that under Condition \ref{cond:1}, the term $({\rm E}\bvarphi^{\otimes2})^{-1}N^{-1}\sum_{i=1}^N\what_i\bvarphi_i\bz_i\trans$ converges to $\bar\bxi_1:=({\rm E}\bvarphi^{\otimes2})^{-1}{\rm E}\bvarphi\bz\trans$, i.e. the projection coefficients of $\bz_i$ onto $\bvarphi_i$, in probability. Also, by Condition \ref{cond:1} and the classical Central Limit Theorem, we have that 
\[
\frac{1}{N}\sum_{i = 1}^N\bvarphi_i-\frac{1}{N}\sum_{i = 1}^N \what_i \bvarphi_i=O_p(n^{-\frac{1}{2}}+N^{-\frac{1}{2}}).
\]
It then follows that
\[
\nhalf(\widehat \btheta\subDR - \bthetabar) =\nhalf\left\{\frac{1}{N}\sum_{i=1}^N \what_i(\bz_i-\bar\bxi_1\trans\bvarphi_i)+\frac{1}{N}\sum_{i = 1}^N\bar\bxi_1\trans\bvarphi_i \right\}+ o_p(1).
\]
By the definition of $\bthetabar$, ${\rm E}\bz_i=0$ and $\bvarphi_i$ contains $1$.  Then,
\[
{\rm E}\bar\bxi_1\trans\bvarphi_i={\rm E}[\proj(\bz_i\mid\bvarphi_i)]={\rm E}\bz_i-{\rm E}[\proj\supvert(\bz_i\mid\bvarphi_i)]=\bzero,
\]
where the last equality follows from the fact that the expectation of the residual from the linear regression model is zero if the regression basis contains $1$. By the classical Central Limit Theorem, we have $\Ninv\sum_{i = 1}^N\bar\bxi_1\trans\bvarphi_i=O_p(N^{-\frac{1}{2}})=o_p(n^{-\frac{1}{2}})$. Consequently, 
\[
\nhalf(\widehat \btheta\subDR - \bthetabar) =\frac{\nhalf}{N}\sum_{i=1}^N \what_i(\bz_i-\bar\bxi_1\trans\bvarphi_i)+ o_p(1)= \nhalf\sum_{s=1}^S \rho_s \left\{ n_s^{-1} \sum_{i=1}^N V_iI(\Ssc_i = s) \proj\supvert(\be_{\SL i}\mid\bvarphi_i) \right\} + o_p(1),
\]
which, again by the classical Central Limit Theorem, weakly converges to $N(\bzero,\bSigma\subDR)$ where
\[
\bSigma\subDR=\sum_{s=1}^S \rho_s^2 \rho_{1s}^{-1} {\rm E}\left[\{\proj\supvert(\be_{\SL i}\mid\bvarphi_i) \}^{\otimes 2} \mid \Ssc_i = s\right].
\]

For the accuracy measure estimator, $\Dhat\subDR$, we can again use (\ref{equ:h:1}) and follow a similar approach as in Appendix \ref{sec: asym-DSL} to show that $\sup_{\btheta\in\mathbf{\Theta}} | \Dhat\subDR( \btheta) - D( \btheta) | = o_p(1)$. Then we can derive the expansion (following the arguments of Appendix \ref{sec: asym-DSL}) to obtain
\[
\nhalf \{ \Dhat \subDR( \btheta)  - D(\btheta)\} =  \frac{n^{\frac{1}{2}}}{N}\sum_{i = 1}^N\what_i(1+\balphatilde\trans\bvarphi_i) [d\{y_i, \yscrl(\btheta \trans \bx_i)\} - D(\btheta)]  + o_p(1),
\]
for an arbitrary $\btheta\in\bTheta$. Following the arguments for $\nhalf(\widehat \btheta\subDR - \bthetabar)$, we have
\[
\nhalf \{ \Dhat \subDR( \btheta)  - D(\btheta)\}=\nhalf\sum_{s=1}^S \rho_s \left\{ n_s^{-1} \sum_{i=1}^N V_iI(\Ssc_i = s) \proj\supvert(d\{y_i, \yscrl(\btheta \trans \bx_i)\} - D(\btheta)\mid\bvarphi_i) \right\} + o_p(1),
\]
for an arbitrary $\btheta\in\bTheta$. As shown in Appendix \ref{sec: asym-DSL}, $\nhalf \{ \Dhat \subSL ( \btheta)  - D(\btheta)\}$ is stochastically equicontinuous at $\bthetabar$ and $D(\btheta)$ is continuously differentiable at $\bthetabar$. So, we have
\begin{align*}
&\nhalf \{ \Dhat \subDR(\bthetahat\subDR)  -\Dbar\}=\nhalf \{ \Dhat \subDR( \bthetahat \subDR)  - D(\bthetahat \subDR)\} +\nhalf \{ D(\bthetahat \subDR)  - D(\bthetabar)\} \\
=&\nhalf  \sum_{s= 1}^S \rho_s \left[\ninv_s \sum_{i = 1}^N V_i I(\Ssc_i = s) \proj\supvert\left(d(y_i, \Yscbar_i) - D(\bthetabar) + \dot{\bD}(\bthetabar)\trans \be_{\SL i}  \mid\bvarphi_i\right)\right]+o_p(1),
\end{align*}
where for the second equality follows from the delta method, the asymptotic expansion of $\bthetahat\subDR-\bthetabar$, and the fact that $\proj\supvert(\ba_1+\ba_2\mid\bb)=\proj\supvert(\ba_1\mid\bb)+\proj\supvert(\ba_2\mid\bb)$ holds for any linear projector. By Condition \ref{cond:1} and the classical Central Limit Theorem, $\nhalf \{ \Dhat \subDR(\bthetahat\subDR)  -\Dbar\}$ converges weakly to $N(0,\sigma^2\subDR)$ where
$$\sigma^2\subDR =  \sum_{s= 1}^S \rho_s^2\rho_{1s}^{-1}  {\rm E}\left[ \left\{ \proj\supvert\left(d(y_i, \Yscbar_i)   - D(\bthetabar) + \dot{\bD}(\bthetabar)\trans \be_{\SL i} \mid\bvarphi_i\right)\right\}^2  \mid \Ssc_i = s \right].$$
These arguments justify Proposition \ref{prop:1}.

}

\newpage

\newpage

\section{Simulation Results: Nonlinear Effects}\label{sec:sup:pc}



Here we analyze nonlinear transformations of the original covariates, denoted by $\bZ$, in the outcome regression model to achieve improved prediction performance relative to the simple linear model.  We consider the leading principal components (PCs) of $\bZ$ and $\bPsi(\bZ)$ where $\bPsi(\cdot)$ is a vector of nonlinear transformation functions.  We evaluate this approach under the three settings considered in the main text (i) ($\Msc_{\mbox{\tiny correct}}$, $\Isc_{\mbox{\tiny correct}}$), (ii) ($\Msc_{\mbox{\tiny incorrect}}$, $\Isc_{\mbox{\tiny correct}}$), (iii) ($\Msc_{\mbox{\tiny incorrect}}$, $\Isc_{\mbox{\tiny incorrect}}$), and (iv) an additional Gaussian mixture (GM) setting where the effect of  $\bZ$ is highly nonlinear in $y$.  Specifically, in the GM setting, 
$$y \sim \mbox{Bernoulli}(0.5), \quad  x\suborik = w_k + 0.12w_k^{3I(k \in \{3,4,7,8\})}, \quad \bw=(w_1,\ldots,w_{10})\trans \mid y  \sim N(y\bmu,\bSigma_y),$$
$$\bSigma_0 = [0.2^{|k-k'|}]_{k=1,...,10}^{k'=1,...,10}, \quad \bSigma_1 = \bSigma_0 + [0.3^{|k-k'|} -0.4I(k=k')+0.2I(k\ne k')]_{k=1,...,10}^{k'=1,...,10},$$
$\bmu=(0.2,-0.2,0.2,-0.2,0.2,-0.2,0.1,-0.1,0,0)\trans$, $\Sbb=I(w_3+\delta_1<0.5)$ where $\delta_1 \sim N(0,1)$. 

 For each setting, we compare the out of sample prediction performance of (1) the proposed GLM with $\bx=\bZ$ estimated via SL or SSL, (2) the proposed GLM based on the first 5 PCs of $\bZ$ and the first 5 PCs of $\bPsi(\bZ)$ where $\bPsi(\bZ)$ contains all quadratic and two-way interaction terms of $\bZ$, (3) the imputation model used in the SSL approach (IMP), and (4) a random forest (RF) model. The SSL approach uses the same sets of basis functions as those in main text under settings (i)--(iii) to augment the predictors for construction of the imputation model. For setting (iv), the basis used for augmentation is the same as that for (iii). For simplicity, we use the apparent estimators for $D$ and $\widecheck\btheta\subSSL$ for SSL estimation of $\btheta$. Table \ref{tab-PC} presents the out of sample performance of the prediction model estimated from the aforementioned methods when $n=200$ and $N=20000$. Table \ref{tab:D:perform:secs2} provides a comparison of the standard errors of various estimators (including the SL, SSL, and DR) of the accuracy parameters for the outcome model constructed with the leading PCs of  $\bZ$ and $\bPsi(\bZ)$.

\newpage

\begin{table}[h]\centering
\caption{{\label{tab:pred} Out of sample Brier score and OMR of the proposed SL and SSL prediction models with $\bX = \bZ$ and  $\bX = \bPsi(\bZ)$ compared with the imputation model used in the SSL approach (IMP) and the random forest (RF) model under settings (i) ($\Msc_{\mbox{\tiny correct}}$, $\Isc_{\mbox{\tiny correct}}$), (ii) ($\Msc_{\mbox{\tiny incorrect}}$, $\Isc_{\mbox{\tiny correct}}$), (iii) ($\Msc_{\mbox{\tiny incorrect}}$, $\Isc_{\mbox{\tiny incorrect}}$) and (iv) the GM setting. 
}}
\begin{tabular}{c|c|cccc|cc}
\hline
Setting & & \multicolumn{4}{c|}{$\bX =\bZ$} & \multicolumn{2}{c}{$\bX = \bPsi(\bZ)$} \\
\cline{3-8}
& & SL & SSL & IMP & RF & SL & SSL \\
\hline
(i) & Brier score & 0.100 &0.100 &0.120 &0.115 &0.101 &0.099 \\
& OMR &  0.139 &0.140 &0.164 &0.156 &0.141 &0.140 \\ \hline 
(ii) &Brier score &0.162 &0.159 &0.139 &0.155 &0.160 &0.157 \\
& OMR & 0.230 &0.227 &0.190 &0.218 &0.226 &0.223 \\ \hline
(iii) & Brier score & 0.120 &0.117 &0.131 &0.111 &0.115 &0.111 \\
& OMR & 0.167 &0.163 &0.180 &0.140 &0.163 &0.160 \\ \hline
(iv) &Brier score & 0.241 &0.236 &0.459 &0.181 &0.173 &0.171 \\
& OMR  & 0.401 &0.394 &0.484 &0.265 &0.254 &0.254 \\ \hline
\end{tabular}\label{tab-PC}
\end{table}

\begin{table}[!htbp]
    \caption{{\label{tab:D:perform:secs2} The 100$\times$ESE of $\Dhat\subSL^w$, $\Dhat\subSSL^w$ and $\Dhat\subDR^w$ of the PC basis outcome model under (i) ($\Msc_{\mbox{\tiny correct}}$, $\Isc_{\mbox{\tiny correct}}$), (ii)  ($\Msc_{\mbox{\tiny incorrect}}$, $\Isc_{\mbox{\tiny correct}}$), (iii) ($\Msc_{\mbox{\tiny incorrect}}$, $\Isc_{\mbox{\tiny incorrect}}$) and (iv) the GM setting with the outcome model constructed with  $\bX = \bPsi(\bZ)$. Also included are the relative efficiency (denoted by RE) of the SSL and DR approaches to SL.}}
      \centering
      \begin{tabular}{c |c|cc|cc|| c|cc|cc|}
  \hline
 Setting & \multicolumn{5}{c||}{OMR} & \multicolumn{5}{c|}{Brier score} \\ 
 \cline{2-11}
 & \multicolumn{1}{c|}{$\Dhat\subSL^w$}& \multicolumn{1}{c}{$\Dhat\subSSL^w$}& \multicolumn{1}{c|}{RE}& \multicolumn{1}{c}{$\Dhat\subDR^w$} & \multicolumn{1}{c||}{RE} 
 & \multicolumn{1}{c|}{$\Dhat\subSL^w$}& \multicolumn{1}{c}{$\Dhat\subSSL^w$} & \multicolumn{1}{c|}{RE} & \multicolumn{1}{c}{$\Dhat\subDR^w$} & \multicolumn{1}{c|}{RE} \\
  \hline
(i)   &2.58 &2.31 &1.25 &2.60 &0.99 &1.51 &1.41 &1.14 &1.57 &0.92 \\
(ii)  &3.08 &2.61 &1.38 &3.23 &0.91 &1.74 &1.54 &1.28 &1.90 &0.84 \\
(iii) &3.02 &2.78 &1.19 &3.27 &0.86 &1.74 &1.62 &1.16 &1.82 &0.92 \\
(iv)  &2.99 &2.69 &1.24 &3.51 &0.73 &1.26 &1.18 &1.15 &1.68 &0.57 \\
\hline
\end{tabular}
\end{table}



\end{document}